\documentclass[10pt]{article}

\usepackage{amsmath,amsthm,verbatim,amssymb,amsfonts,amscd, graphicx, enumitem, array, longtable, booktabs, bm, wrapfig, natbib}
\usepackage{mathrsfs}
\usepackage{graphics}
\usepackage{centernot}
\usepackage{authblk}
\usepackage[T1]{fontenc}
\usepackage[hidelinks]{hyperref}
\newcommand{\R}{\mathbb R}
\newcommand{\E}{\mathbb E}
\newcommand{\Normal}{\mathcal N}
\newcommand{\Tr}{\operatorname{Tr}}
\newcommand{\Unif}{\operatorname{Unif}}
\newcommand{\Var}{\operatorname{Var}}

\newcommand{\dd}{\,\mathrm d}
\newcommand{\1}{\mathbf 1}

\newtheorem{theorem}{Theorem}
\newtheorem{proposition}{Proposition}
\newtheorem{lemma}{Lemma}
\newtheorem{Result}{Result}

\newtheorem{assumption}{Assumption}
\theoremstyle{definition}

\theoremstyle{remark}
\newtheorem*{remark}{Remark}

\title{In-context Learning of Single-index Targets:\\ Comparing Kernel and Feature Learners}
\author[1,3]{Haotian Gu\thanks{These authors contributed equally to this work.}}
\author[1,2]{Yizhou Xu$^\star$}
\author[1]{Lenka Zdeborová}
\affil[1]{\small Statistical Physics of Computation Laboratory, \'Ecole Polytechnique F\'ed\'erale de Lausanne (EPFL)}
\affil[2]{\small Information, Learning and Physics Laboratory, \'Ecole Polytechnique F\'ed\'erale de Lausanne (EPFL)}
\affil[3]{\small University of Chinese Academy of Sciences}
\date{}

\begin{document}
\maketitle
\begin{abstract}
In-context learning (ICL) enables a pretrained model to infer a task from demonstrations without updating its parameters. While much of the existing theory focuses on linear target functions, in this paper we study nonlinear cases by comparing two one-layer attention architectures on the same family of single-index tasks. A kernel learner first maps inputs through a fixed nonlinear feature map and then applies linear attention, whereas a feature learner applies attention to the original input, followed by a learned nonlinear readout. We derive predictions for their memorization and generalization errors using the replica method, retaining the effects of pretraining size, task-pool diversity, and training and inference context lengths. The resulting predictions closely match numerical experiments across a broad range of regimes.  Our analysis yields phase diagrams that characterize when each architecture is advantageous as the amount of pretraining data, task diversity, and context lengths vary. We further identify qualitatively different context-length scalings for the two learners. Together, these results clarify how architectural choices interact with the dataset and govern nonlinear in-context learning.
\end{abstract}

\section{Introduction}
In-context learning (ICL) is the ability of a pretrained model to infer a task from examples provided in its input context and make predictions without updating its parameters. This behavior is one of the most distinctive capabilities of modern Transformers \citep{brown2020language,dong2024survey}, and has motivated a growing theoretical literature that studies ICL in controlled problems. Early work showed that Transformers trained on families of synthetic tasks can learn new functions from demonstrations \citep{garg2022can}, and connected it to familiar learning algorithms such as gradient descent, least squares, and ridge regression \citep{akyurek2022learning,von2023transformers,bai2023transformers,mahankali2024one}. These results established that attention-based models can implement nontrivial learning procedures within their forward pass.

A more recent line of work asks how such in-context algorithms are acquired during pretraining. For linear attention, whether a model memorizes the pretraining tasks or generalizes to unseen tasks depends jointly on the number of pretraining prompts, task diversity, and the number of demonstrations available at inference time \citep{raventos2023pretraining,nguyen2025differential,goddard2025can}. Recent high-dimensional analyses make these dependencies quantitative and characterize how the shared task structure affects in-context generalization \citep{lu2025asymptotic,wu2024many,takanami2025learning,kwon2026effect,letey2026pretrain}. Other works also study the learning dynamics of ICL \citep{singh2023transient,zhang2025training} and the influence of various factors including data structures, distribution shift, hyperparameters and so on \citep{dragutinovic2025softmax,takanami2025learning,kwon2026effect,letey2026pretrain,bordelon2026theory}.

Most of this theory, however, focuses on linear targets. Nonlinear tasks introduce an additional question: beyond aggregating demonstrations in context, the learner must determine how the nonlinearity on the input should be learned. Single-index models provide a natural setting for this question, and create a classical separation between fixed-kernel methods and feature-learning methods \citep{ghorbani2020neural,bietti2022learning,ba2022high,dandi2024two}.
ICL introduces a further distinction: the relevant direction changes between tasks and must be extracted from the demonstrations without updating the parameters. Recent work has shown that pretrained Transformers can exploit low-dimensional structure to learn single-index targets efficiently in context \citep{oko2024pretrained}, and that nonlinear Transformers can perform inference-time feature learning on single-index tasks \citep{nishikawa2025nonlinear}. 

Our goal is thus to compare fixed-kernel methods and feature-learning methods across various amounts of pretraining data, task diversities, and context lengths.
We consider two learners that use the same one-layer attention but place the nonlinear transformation on opposite sides of attention. The \emph{kernel learner} first maps every input through a fixed nonlinear feature map and then performs linear attention in the resulting feature space. In contrast, the \emph{feature learner} performs attention directly on the original inputs and applies a learned nonlinear readout after attention. 

We study the two architectures on the same learning problem. Pretraining prompts are generated from a finite pool of tasks, allowing us to separately measure performance on tasks that appeared during pretraining and on fresh tasks drawn from the underlying task distribution. We refer to these quantities as the \emph{memorization error} (on seen tasks) and the \emph{generalization error} (on unseen tasks), respectively. Varying the task-pool size therefore lets us study how architectural choices interact with the memorization--generalization tradeoff emphasized in previous work \citep{raventos2023pretraining,lu2025asymptotic,goddard2025can,takanami2025learning,farne2026rules}.

For analytical tractability, we instantiate the kernels using spherical harmonics and analyze both learners using the non-rigorous yet well-established replica method \citep{mezard1987spin,zdeborova2016statistical}. The comparison is defined for common values of the input dimension \(d\), feature dimension \(D\), pretraining size \(N\), task-pool size \(N_0\), and context lengths \(L,\widetilde L\). We validate these predictions numerically over a broad range of parameters.

Our main contributions are threefold.

\begin{itemize}[leftmargin=1.5em]
\vspace{-2mm}
    \item \textbf{Comparison of kernel and feature learners on the same dataset.}
    We formulate the two learners on a common learning problem, with the same task pool, pretraining prompts, and inference prompts, and characterize their memorization and generalization performance as functions of \(d,D,N_0,N,L,\widetilde L\). For each learner, we derive replica-based predictions and the resulting formulas accurately predict the performance of both learners over a broad range of regimes.

    \item \textbf{Phase diagrams for when each architecture is advantageous.}
    We map out how the relative performance of the two learners depends jointly on problem parameters. The feature learner benefits from its lower-dimensional representation and typically reaches low error with fewer pretraining prompts, while the kernel learner can achieve lower memorization error when the task pool is small and pretraining is sufficient. We further identify regimes in which context-length mismatch changes the relative advantage of the two architectures.

    \item \textbf{A theoretical characterization of context-length scaling.}
    In the large-pretraining and large-task-diversity limit, we prove explicit context-length scaling. For a kernel learner containing degree-\(k\) spherical harmonics, recovering that component requires a context length of order \(d^k\), which produces a stage-wise learning curve. By contrast, the feature learner exhibits a single transition at \(L=\Theta(d)\) and can exploit all Hermite components of the teacher at the same time. We also rigorously characterize the behavior when the training context length and the test contest length are different.
\end{itemize}

\section{Settings}
\label{sec:settings}
\subsection{Task}
We consider an ICL problem in which the training dataset consists of $N$ prompts\\ $\mathcal D=\left\{(x_0^\mu,y_0^\mu,x_1^\mu,y_1^\mu,\ldots,x_L^\mu,y_L^\mu)\right\}_{\mu=1}^N$. Each prompt contains $L$ demonstrations $\{(x_\ell^\mu,y_\ell^\mu)\}_{\ell=1}^L$ together with a query-label pair $(x_0^\mu,y_0^\mu)$. At inference time, the learner observes the demonstrations and the query $x_0$ and aims to predict the corresponding label $y_0$.

The inputs are sampled independently and uniformly on the sphere of radius $\sqrt d$: $x_\ell^\mu\overset{\mathrm{i.i.d.}}{\sim}\Unif\!\left(\mathbb S^{d-1}(\sqrt d)\right)$ for $\ell=0,1,\ldots,L$ and $\mu=1,2,\ldots,N$. Each prompt is associated with a latent task vector drawn from a finite task pool. Specifically, we first generate $\mathcal A_0=\{a^k\}_{k=1}^{N_0}$ with $a^k\overset{\mathrm{i.i.d.}}{\sim}
\Unif\!\left(\mathbb S^{d-1}(\sqrt d)\right)$ and then independently sample a task vector $a^\mu:=a^{\tau^\mu}$ with $\tau^\mu\overset{\mathrm{i.i.d.}}{\sim}\Unif(\{1,\ldots,N_0\})$ for each training prompt $\mu$. Conditional on $a^\mu$, the labels are generated by the noisy single-index model
\begin{equation}
y_\ell^\mu=\sigma_\star\!\left(\frac{(a^\mu)^\top x_\ell^\mu}{\sqrt d}\right)+\epsilon_\ell^\mu,
\qquad
\epsilon_\ell^\mu\overset{\mathrm{i.i.d.}}{\sim}\mathcal N(0,\Delta),
\label{eq:teacher-model}
\end{equation}
for $\ell=0,1,\ldots,L$.
Throughout the paper, $\sigma_\star:\mathbb R\to\mathbb R$ is a fixed nonlinear teacher function satisfying $\mathbb E[\sigma_\star(G)]=0$, $\mathbb E[\sigma_\star(G)^2]=1$ and $\mu_1:=\mathbb E[\sigma_\star(G)G]
\neq 0$, where the expectation is over $G\sim\mathcal N(0,1)$.
While the first two assumptions are imposed for simplicity, $\mu_1\neq0$ is essential. It guarantees a nonvanishing linear component of the teacher and places the problem in the $\mathcal O(d^2)$ learning regime \citep{nishikawa2025nonlinear}.

We evaluate the trained learner on an inference prompt containing $\widetilde L$ demonstrations: $\{(x_\ell,y_\ell)\}_{\ell=1}^{\widetilde L}$ and a query $x_0$ associated with a task vector $a^\star$. To separate the learner performance from irreducible label noise, we use noiseless labels at inference time: $y_\ell=\sigma_\star\!\left(\frac{(a^\star)^\top x_\ell}{\sqrt d}\right)$ for $\ell=0,1,\ldots,\widetilde L$, where $x_\ell\overset{\mathrm{i.i.d.}}{\sim}\Unif\!\left(\mathbb S^{d-1}(\sqrt d)\right)$.
Given the demonstrations and the query, let $\widehat y$ denote the learner prediction for $y_0$.

We distinguish two notions of test performance according to whether the inference task has appeared in the training task pool. We define the memorization and the generalization error by
\begin{equation}
\mathcal E_M:=\mathbb E_{a^\star\sim\Unif(\mathcal A_0)}
\mathbb E_{x_{0:\widetilde L}}
\left[(y_0-\widehat y)^2\right],\ \mathcal E_G:=\mathbb E_{a^\star\sim\Unif(\mathbb S^{d-1}(\sqrt d))}\mathbb E_{x_{0:\widetilde L}}
\left[(y_0-\widehat y)^2\right].
\label{eq:memorization-generalization-error}
\end{equation}
For memorization, the test task is sampled uniformly from the same pool $\mathcal A_0$,
whereas for generalization, a fresh task is sampled uniformly from the sphere.
Here $x_{0:\widetilde L}:=(x_0,x_1,\ldots,x_{\widetilde L})$, and the expectation is taken over a fresh inference prompt, conditional on the trained learner and, for $\mathcal E_M$, on the task pool $\mathcal A_0$.
Related works have also studied in-distribution and out-of-distribution task generalization by restricting the task vectors to lower-dimensional subspaces \citep{takanami2025learning,lu2025asymptotic}. We do not consider the distribution shift here because this is not our main focus.

The above task model is closely related to existing models of ICL \citep{takanami2025learning,lu2025asymptotic,oko2024pretrained,nishikawa2025nonlinear}. In particular, \cite{oko2024pretrained,nishikawa2025nonlinear} characterize the sample complexity required to learn such task families. Our goal is instead to characterize and compare the memorization and generalization errors of the two classes of learners introduced below. A schematic comparison is shown in Figure~\ref{fig:kernel-vs-feature}.

\begin{figure}[t]
\vspace{-10mm}
\centering
\includegraphics[width=\textwidth]{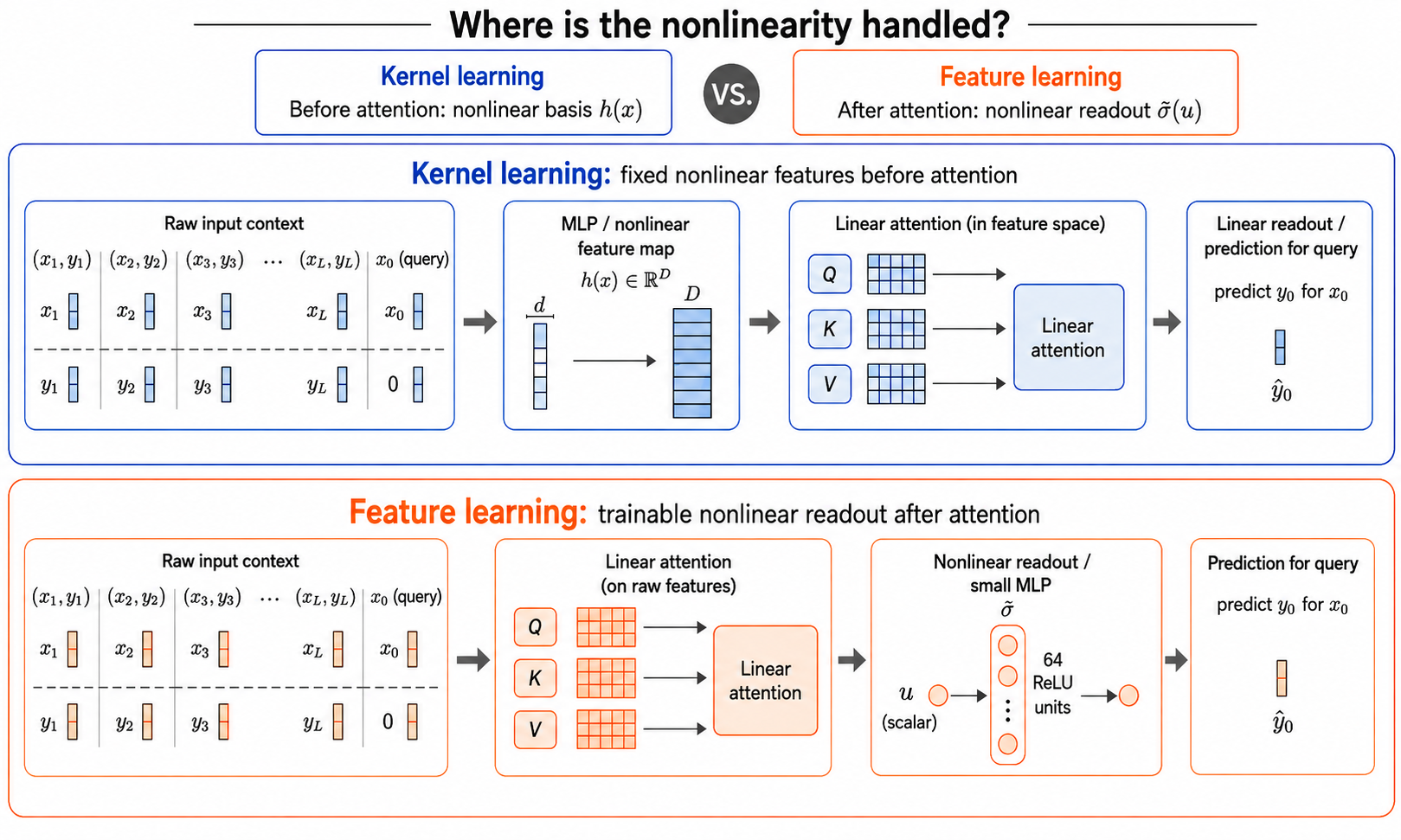}
\vspace{-6mm}
\caption{The kernel and feature learners. The kernel learner applies a fixed map $h:\mathbb R^d\to\mathbb R^D$ before self-attention, whereas the feature learner applies a trainable readout $\widehat\sigma$ after self-attention.}
\label{fig:kernel-vs-feature}
\vspace{-3mm}
\end{figure}

\subsection{Kernel learner}
We first consider the \emph{kernel learner}, which applies a fixed nonlinear feature map to each input before a single self-attention layer. Let $h:\mathbb R^d\to\mathbb R^D$
be a fixed feature map. For the $\mu$-th training prompt, we form the embedded prompt
\begin{equation}
E^\mu=
\begin{bmatrix}
h(x_1^\mu) & \cdots & h(x_L^\mu) & h(x_0^\mu)\\
y_1^\mu & \cdots & y_L^\mu & 0
\end{bmatrix}
\in\mathbb R^{(D+1)\times(L+1)},
\label{eq:kernel-embedded-prompt}
\end{equation}
where the final zero corresponds to the masked query label.
We apply a single linear self-attention layer,
\begin{equation}
\mathrm{Attn}(E^\mu)=E^\mu+W^V E^\mu
\frac{(W^K E^\mu)^\top W^Q E^\mu}{L},
\label{eq:linear-self-attention}
\end{equation}
and take the bottom-right entry as the prediction: $\widehat y^\mu=\mathrm{Attn}(E^\mu)_{D+1,L+1}$.
Following the standard parameterization used in theoretical analyses of self-attention \citep{takanami2025learning,lu2025asymptotic}, we restrict the value and key-query matrices to satisfy
\begin{equation}
W^V=\begin{bmatrix}
* & *\\
0 & 1
\end{bmatrix},
\qquad
(W^K)^\top W^Q=\begin{bmatrix}
W^\top & *\\
0 & *
\end{bmatrix},
\label{eq:choice-attention}
\end{equation}
where $W\in\mathbb R^{D\times D}$ is the trainable attention matrix and the blocks denoted by $*$ do not affect the predicted query label. Under this parameterization, the learner output reduces to
\begin{equation}
\widehat y^\mu=\Tr\!\left(W(H^{\mathrm{K}})^\top\right),\ \text{where}\ H^{\mathrm{K}}:=\frac{1}{L}h(x_0^\mu)\left(\sum_{\ell=1}^Ly_\ell^\mu h(x_\ell^\mu)\right)^\top\in\mathbb R^{D\times D}
\label{eq:H^K}
\end{equation}
Thus, after the fixed embedding $h$, training the attention layer is equivalent to a linear regression problem in the $D^2$ entries of $H_\mu^{\mathrm{K}}$.
We train $W$ by empirical risk minimization,
\begin{equation}
\widehat W_{\mathrm K}=\arg\min_{W\in\mathbb R^{D\times D}}
\left\{\sum_{\mu=1}^N\left[\Tr\!\left(
W(H_\mu^{\mathrm K})^\top\right)-y_0^\mu\right]^2+\lambda N_0\|W\|_F^2\right\},
\label{eq:kernel-erm}
\end{equation}
where $\lambda\ge0$ denotes the weight decay strength. We denote by $\mathcal E_M^{\mathrm K}$ and $\mathcal E_G^{\mathrm K}$ the memorization and generalization errors defined in
\eqref{eq:memorization-generalization-error}, respectively, when the prediction is produced by $\widehat W_{\mathrm K}$.

For analytical tractability, we specialize the fixed feature map $h$ to a finite collection of spherical harmonics. For each degree $k\ge1$, let $h_k(x)\in\mathbb R^{D_{k,d}}$ denote an orthonormal basis of the degree-$k$ spherical harmonics on $\mathbb S^{d-1}(\sqrt d)$, where $D_{k,d}=\frac{2k+d-2}{k}\binom{k+d-3}{k-1}$. Let $\mathcal K\subset\mathbb Z_{\ge1}$ be a fixed finite set of harmonic degrees. We choose $h(x)=\bigoplus_{k\in\mathcal K} h_k(x)\in\mathbb R^D$, where $D=D(d):=\sum_{k\in\mathcal K}D_{k,d}$ and $\bigoplus$ denotes vertical concatenation. 
This spherical-harmonic construction should be viewed as an analytically convenient realization of a fixed nonlinear feature map. More practical embeddings, such as random features \citep{hu2024asymptotics,aguirre2025random,defilippis2024dimension}, lead to closely related architectures. We leave the analysis of such random-feature learners to future work.

\subsection{Feature learner}
We next consider the \emph{feature learner}, which reverses the order of the nonlinear transformation and the attention module. 
For the $\mu$-th training prompt, we form
\begin{equation}
E^\mu=\begin{bmatrix}
x_1^\mu & \cdots & x_L^\mu & x_0^\mu\\
y_1^\mu & \cdots & y_L^\mu & 0
\end{bmatrix}
\in\mathbb R^{(d+1)\times(L+1)}.
\label{eq:feature-prompt}
\end{equation}
We apply the same single linear self-attention layer as in \eqref{eq:linear-self-attention} and impose the parameterization
\eqref{eq:choice-attention}, now with $W\in\mathbb R^{d\times d}$.
The bottom-right entry of the attention output is then
\begin{equation}
u_W^\mu:=\mathrm{Attn}(E^\mu)_{d+1,L+1}=\frac{1}{L}
\sum_{\ell=1}^Ly_\ell^\mu\,(x_0^\mu)^\top W x_\ell^\mu.
\label{eq:feature-attention-output-expanded}
\end{equation}
We equivalently have
\begin{equation}
u_W^\mu=\Tr\!\left(W(H_\mu^{\mathrm F})^\top\right),\ \text{where}\ H^{\mathrm F}:=\frac{1}{L}x_0^\mu\left(\sum_{\ell=1}^Ly_\ell^\mu x_\ell^\mu\right)^\top\in\mathbb R^{d\times d}.
\label{eq:feature-attention-output}
\end{equation}
The final prediction is obtained by applying a scalar nonlinear readout $\widehat y^\mu=\sigma(u_W^\mu)$.

We train the attention module and the nonlinear readout in two stages. In the first stage, we fix a prescribed scalar function $\sigma_0:\mathbb R\to\mathbb R$ and optimize only the attention matrix:
\begin{equation}
\widehat W_{\mathrm F}=\arg\min_{W\in\mathbb R^{d\times d}}\left\{\sum_{\mu=1}^N\left[\sigma_0\!\left(\Tr\!\left(W(H_\mu^{\mathrm F})^\top\right)\right)-y_0^\mu\right]^2+\lambda N_0\|W\|_F^2\right\}.
\label{eq:feature-first-stage}
\end{equation}
Here $\sigma_0$ is fixed throughout the first stage and is not itself trained.

In the second stage, we freeze $\widehat W_{\mathrm F}$ and optimize only the scalar readout. Let $U_{\widehat W_{\mathrm F}}=\Tr\!\left(\widehat W_{\mathrm F}(H^{\mathrm F})^\top\right)$
be the scalar representation produced by $\widehat W_{\mathrm F}$ on an independent fresh prompt of length $L$ drawn from the same pretraining distribution, with its task vector sampled uniformly from the same task pool $\mathcal A_0$. We define
\begin{equation}
\widehat\sigma\in\arg\min_{\sigma}
\mathbb E\left[\left(\sigma(U_{\widehat W_{\mathrm F}})-y_0\right)^2\,\middle|\,\widehat W_{\mathrm F},\mathcal A_0\right],
\label{eq:feature-second-stage-population}
\end{equation}
where the expectation is over the fresh prompt, including its task index, inputs, and label noise. The solution of \eqref{eq:feature-second-stage-population} is $\widehat\sigma(u)=\mathbb E\left[y_0\,\middle|\,U_{\widehat W_{\mathrm F}}=u,\,\widehat W_{\mathrm F},\mathcal A_0\right]$.

The population optimization in \eqref{eq:feature-second-stage-population} is an idealization of training a one-dimensional multi-layer perceptron (MLP) on an independent set of fresh prompts after the attention matrix has been learned. More explicitly, in experiments we use $N'$ additional prompts and fit
\begin{equation}
\widehat\sigma_{N'}=\arg\min_{\sigma\in\mathcal S}
\frac{1}{N'}\sum_{\mu=1}^{N'}
\left[\sigma\!\left(\Tr\!\left(\widehat W_{\mathrm F}(H_\mu^{\mathrm F})^\top\right)\right)-y_0^\mu\right]^2,
\label{eq:feature-second-stage-empirical}
\end{equation}
where $\mathcal S$ is chosen to be a one-dimensional MLP. Since this stage involves only scalar regression, we idealize its estimation error away by taking $N'\to\infty$ and $\frac{N'}{N}\to0$.
Consequently, $N+N'=N(1+o(1))$, so the additional samples used to learn the scalar readout form a vanishing fraction of the total number of training prompts and do not change the leading-order sample complexity.

We denote by $\mathcal E_M^{\mathrm F}$ and $\mathcal E_G^{\mathrm F}$ the memorization and generalization errors defined in \eqref{eq:memorization-generalization-error}, respectively, for the learner specified by $(\widehat W_{\mathrm F},\widehat\sigma)$.

\section{Main results}
In this section we present and compare the memorization and generalization errors of the kernel learner and the feature learner. 
The results are derived using the replica method together with a Gaussian-equivalence ansatz, in which the local fields are replaced by Gaussian variables with matching leading-order covariances. Such Gaussian equivalence has been established rigorously for several related models \citep{goldt2022gaussian,hu2022universality,montanari2022universality}. In our ICL setting, however, we use it as an ansatz and validate the resulting predictions numerically.

Let $c_k$ denote the Hermite coefficients of $\sigma_\star$, $\theta(a):=\bigoplus_{k\in\mathcal K} c_k\sqrt{\frac{D}{D_{k,d}}}\,h_k(a)\in\R^D$ and $\nu_{\rm in}:=\sum_{k\in\mathcal{K}}c_k^2$, $\nu_{\rm out}:=\sum_{k\notin\mathcal{K}}c_k^2$. 
Result \ref{res:kernel} gives the memorization error and the generalization error in the limit \(D\to\infty\) with \(\kappa_K:=N_0/D\), \(\alpha_K:=L/D\), \(\widetilde\alpha_K:=\widetilde L/D\), and \(\gamma_K:=N/(DN_0)\) fixed. 
\begin{Result}[kernel learner]
\label{res:kernel}
For $\nu_{\rm in}>0$, let
$(m_K,q_K,\overline q_K,\chi_K,\overline\chi_K,
\widehat m_K,\widehat q_K,\widehat\chi_K)$ solve
Eqs.~\eqref{eq:SE-kernel}.
The distribution of the learned attention matrix is represented by
\begin{equation}
\widehat W_K\overset{\mathrm d}=
\frac1{\widehat\chi_K}
\left[\widehat m_KS_K+
\frac{\sqrt{\nu_{\rm in}\widehat q_K}}{\sqrt{N_0}}Z_K
\left(S_K+\frac{1+\Delta}{\alpha_K}I_D\right)^{1/2}\right]
\left[S_K+\left(\frac{1+\Delta}{\alpha_K}+\frac{\lambda}{\widehat\chi_K}\right)I_D\right]^{-1},
\label{eq:kernel-learned-matrix-main}
\end{equation}
where $S_K:=\frac1{N_0}\sum_{k=1}^{N_0}\theta(a^k)\theta(a^k)^\top$ and $(Z_K)_{ij}\overset{\mathrm{iid}}\sim\Normal(0,1)$.
For $\mathcal P\in\{\mathrm M,\mathrm G\}$, the memorization error and the generalization error can be written as
\begin{equation}
\mathcal E_{\mathcal P}^{\mathrm{kernel}}\approx\nu_{\rm out}+\nu_{\rm in}(1-2m_{K,\mathcal P}+q_{K,\mathcal P})+\frac{\overline q_K}{\widetilde\alpha_K},
\label{eq:kernel-errors-main}
\end{equation}
where $m_{K,\mathrm M}:=m_K,q_{K,\mathrm M}:=q_K$ and
$m_{K,\mathrm G}:=\frac1{D\nu_{\rm in}}\E\Tr(\widehat W_KS_{K,\mathrm{pop}}),q_{K,\mathrm G}:=\frac1{D\nu_{\rm in}}\E\Tr(\widehat W_KS_{K,\mathrm{pop}}\widehat W_K^\top)$ with $S_{K,\mathrm{pop}}:=\E_a[\theta(a)\theta(a)^\top]$.

Moreover, for large $\gamma_K$ we have $\mathcal E_{\mathrm M}^{\mathrm{kernel}}
=f_{\mathrm M}(\alpha_K,\widetilde\alpha_K,S_K)+O\!\left(\frac{D^2}{N}\right)$ and $\mathcal E_{\mathrm G}^{\mathrm{kernel}}
=f_{\mathrm G}
(\alpha_K,\widetilde\alpha_K,S_K,S_{K,\mathrm{pop}})+O\!\left(\frac{D^2}{N}\right)$,
where the functions $f_{\mathcal P}$ are given in Eqs.~\eqref{eq:app-kernel-large-N-memorization}--\eqref{eq:app-kernel-large-N-generalization}.
\end{Result}

Result~\ref{res:kernel} is derived in Appendix~\ref{app:kernel-details}. It is based on the approximation
\begin{equation}
\sigma_\star\left(\frac{a^\top x}{\sqrt d}\right)
\approx\sum_{k\in\mathcal K}
\frac{c_k}{\sqrt{D_{k,d}}}h_k(a)^\top h_k(x)+\xi,
\qquad
\xi\sim\mathcal N(0,\nu_{\rm out}),
\label{eq:kernel-expansion-main}
\end{equation}
where the residual component is modeled as independent Gaussian noise under a Gaussian-equivalence ansatz, as detailed in Appendix~\ref{app:expansion}. 

Result \ref{res:feature} gives the memorization and the generalization error of the feature learner with two-stage training in the limit \(d\to\infty\) with \(\kappa_F:=N_0/d\), \(\alpha_F:=L/d\), \(\widetilde\alpha_F:=\widetilde L/d\), and \(\gamma_F:=N/(dN_0)\) fixed, which generalizes previous work on linear settings \citep{takanami2025learning,lu2025asymptotic}. 
\begin{Result}[feature learner]
\label{res:feature}
Let $m,q,\bar q,\chi_Q,\chi_{\bar Q}$ and their conjugates $\widehat m,\widehat\chi_Q,\widehat\chi_{\bar Q},\widehat q_Q,\widehat q_{\bar Q}$ be the order parameters that solve the state equations in Eq.~\eqref{eq:app-feature-state-equations}.  The distribution of the learned attention matrix is
\begin{equation}
\widehat W_F\overset{\mathrm d}=\left(\widehat mS_F+\frac1{\sqrt{N_0}}Z(\widehat q_QS_F+\widehat q_{\bar Q}I_d)^{1/2}\right)\big((\lambda+\widehat\chi_{\bar Q})I_d+\widehat\chi_QS_F\big)^{-1},
\label{eq:effective-W}
\end{equation}
where $S_F:=\frac1{N_0}\sum_{k=1}^{N_0}a^k(a^k)^\top$ and $Z_{ij}\overset{\mathrm{iid}}\sim\Normal(0,1)$.
The learned stage-II readout is
\begin{equation}
\hat{\sigma}(v)=\E_\psi\!\left[\sigma_\star\!\left(
\frac{\mu_1m}{\mu_1^2q+(1+\Delta)\bar q/\alpha_F}v
+\sqrt{1-\frac{\mu_1^2m^2}{\mu_1^2q+(1+\Delta)\bar q/\alpha_F}}\,\psi\right)\right],
\label{eq:optimal-readout}
\end{equation}
where the expectation is over $\psi\sim\Normal(0,1)$. 
For $\mathcal P\in\{\mathrm M,\mathrm G\}$, the memorization error and the generalization error can be written as
\begin{equation}
\mathcal E_{\mathcal P}^{\mathrm{feature}}\approx\E_{g,z}\left[\left(\sigma_\star(g)-\widehat\sigma\left(C_{\mathcal P}g+\sqrt{V_{\mathcal P}-C_{\mathcal P}^2}\,z\right)\right)^2\right],\qquad
g,z\overset{\rm iid}{\sim}\Normal(0,1),
\label{eq:rs-test-error}
\end{equation}
where we define $C_{\rm M}=\mu_1m,V_{\rm M}=\mu_1^2q+\frac{\bar q}{\widetilde\alpha_F}$ and $C_{\rm G}=\mu_1d^{-1}\E_Z\Tr \widehat W_F,
V_{\rm G}=\left(\mu_1^2+\frac1{\widetilde\alpha_F}\right)\bar q$.

Moreover, for large $\gamma_F$, $\sigma_0(v)=v$ and $\mathcal P\in\{\mathrm M,\mathrm G\}$ we have $\mathcal E_{\mathcal P}^{\mathrm{feature}}=g_{\mathcal P}(\alpha_F,\widetilde\alpha_F,S_F)+O\!\left(\frac{d^2}{N}\right)$, where the function $g_{\mathcal P}$ is given in Eq.~\eqref{eq:app-feature-large-N-errors}.

\end{Result}

The following theorem studies the context scalings. For simplicity we choose $\sigma_0(v)=v$. While it can be derived from Results \ref{res:kernel} and \ref{res:feature}, we also prove it rigorously. See Appendix \ref{app:rigorous}.
\begin{theorem}[Informal]
\label{theo:comparison}
Take $\kappa_F,\kappa_K\to\infty$ after
$\gamma_F,\gamma_K\to\infty$, and set
$\Delta=\lambda=0$ and $\sigma_0(v)=v$.
Let $\alpha=L/d$ and
$\widetilde\alpha=\widetilde L/d$.
For the kernel learner, define
$\eta_{k,L}:=\frac{c_k^2L}{D_{k,d}+c_k^2L}$.
Then, for $\mathcal P\in\{\mathrm M,\mathrm G\}$,
\begin{equation}
\mathcal E_{\mathcal P}^{\mathrm{kernel}}=\nu_{\rm out}+\sum_{k\in\mathcal K}\left[c_k^2(1-\eta_{k,L})^2+\frac{D_{k,d}}{\widetilde L}\eta_{k,L}^2\right]+o(1).
\label{eq:kernel-population}
\end{equation}
For the feature learner, define
$\rho_\alpha:=\frac{c_1}{\sqrt{c_1^2+1/\alpha}}$,
$r_\alpha:=\rho_\alpha^2$ and
$s_{\alpha,\widetilde\alpha}=\frac{\rho_\alpha}{\rho_{\widetilde\alpha}}$.
Then
\begin{equation}
\mathcal E_{\mathcal P}^{\mathrm{feature}}=1-\sum_{k\ge1}c_k^2r_{\widetilde\alpha}^k+
\sum_{\ell\ge0}
\left[
\sum_{m\ge0}
c_{\ell+2m}\rho_\alpha^{\ell+2m}
\sqrt{\frac{(\ell+2m)!}{\ell!}}
\frac{s_{\alpha,\widetilde\alpha}^{\ell}}{m!}
\left(
\frac{s_{\alpha,\widetilde\alpha}^2-1}{2}
\right)^m
-
c_\ell\rho_{\widetilde\alpha}^{\ell}
\right]^2
+o(1),
\label{eq:feature-population}
\end{equation}
\end{theorem}
One immediate corollary is that if we choose $\widetilde L=L=\alpha d^s$ with fixed $s>0$ and $d\to\infty$, we have
\begin{equation}
\mathcal E_{\mathcal P}^{\mathrm{feature}}
\longrightarrow
\begin{cases}
1, & 0<s<1,\\
1-\sum_{k=1}^{\infty}c_k^2r_\alpha^k, & s=1,\\
0, & s>1,
\end{cases}
\quad \mathcal P\in\{\mathrm M,\mathrm G\},
\label{eq:comparison-feature-context-scaling}
\end{equation}
while for the kernel learner:
\begin{equation}
\mathcal E_{\mathcal P}^{\mathrm{kernel}}
\longrightarrow
\nu_{\rm out}+\sum_{\substack{k\in\mathcal K,\ k>s}}c_k^2
+\mathbf 1_{\{s\in\mathcal K\}}
\frac{c_s^2}{1+\alpha s!c_s^2},
\qquad
\mathcal P\in\{\mathrm M,\mathrm G\}.
\label{eq:comparison-kernel-context-scaling}
\end{equation}
This suggests that for the kernel learner, a degree-$k$ harmonic sector is asymptotically fully recovered when $L\gg d^k$, is only partially recovered when $L=\Theta(d^k)$, and is asymptotically unrecovered when $L\ll d^k$. 
Moreover, assuming $s=1$ and $1\in\mathcal K$, we obtain $
\mathcal E_{\mathcal P}^{\mathrm{kernel}}-\mathcal E_{\mathcal P}^{\mathrm{feature}}\approx\sum_{k\ge2}c_k^2r_\alpha^k$.
For a nonlinear teacher this gap is strictly positive for every $\alpha>0$.

Thus there are two advantages of the feature learner. First, Results \ref{res:kernel} and \ref{res:feature} suggest that the two learners operate at different pretraining sample scales. Denote \(p:=\max\mathcal K\). The feature and kernel attention matrices contain $d^2$ and $D^2=\Theta(d^{2p})$
parameters, respectively, and hence converge on the scales
$N\gg d^2$ and $N\gg D^2$. Thus, when $p>1$, there is a broad regime $d^2\ll N\ll D^2$ in which the feature learner is already close to its infinite-pretraining limit while the kernel learner is not.

Second, the two architectures use the demonstrations differently. For the kernel learner,
a degree-$k$ spherical-harmonic sector requires a context length of order $d^k$ to be recovered. Under $L=\Theta(d)$, only the degree-one sector contributes.
By contrast, the feature learner only requires $L=\Theta(d)$ demonstrations and exploits all Hermite components of the teacher with factors $r_\alpha^k$ simultaneously.

On the other hand, the kernel learner has an advantage that a longer inference context always improves its performance by Eq. \eqref{eq:kernel-population}.
In contrast, such monotonicity need not hold for the feature learner. There is a nonnegative mismatch error (the third term of Eq. \eqref{eq:feature-population}) for the feature learner which arises because the Stage-II readout is optimized for the training context length \(L\) and is then kept fixed at inference. 
Theorem \ref{theo:comparison} also implies that when $\tilde L$ is fixed, both the kernel and the feature learner reach the optimal performance at $L=\tilde L$. Increasing the training context beyond $\tilde L$ is not beneficial.

\section{Experiments}
\label{sec:experiments}
\begin{figure}[t]
\vspace{-10mm}
\centering
\includegraphics[width=\textwidth]{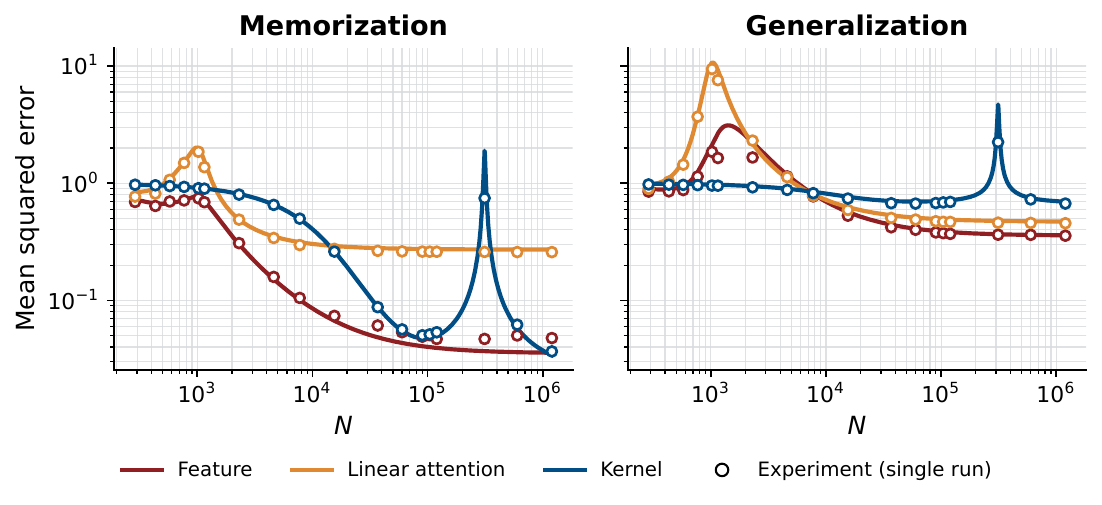}
\vspace{-8mm}
\caption{Memorization (left) and generalization (right) errors as a function of $N$. We fix $N_0=24$ and $L=\widetilde L=1024$. Red, orange, and blue denote the feature learner, linear attention, and the kernel learner, respectively. Solid lines show replica predictions, and circles show numerical results from a single training run. The three learners use the same training and inference prompts.}
\label{fig:gamma-merged}
\vspace{-2mm}
\end{figure}

\begin{figure}[t]
\vspace{-5mm}
\centering
\includegraphics[width=\textwidth]{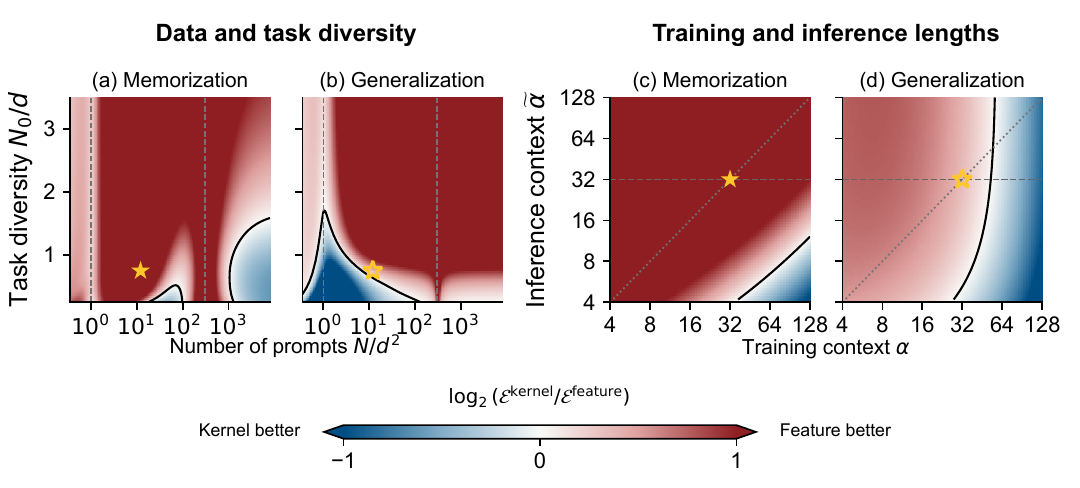}
\vspace{-8mm}
\caption{Replica predictions for the relative errors of the two learners. In all panels, the color shows $\log_2(\mathcal E^{\mathrm{kernel}}/\mathcal E^{\mathrm{feature}})$. (a,b) Memorization and generalization as $N/d^2$ and $N_0/d$ vary, with $\alpha=\widetilde\alpha=32$; vertical dashed lines mark $N=d^2$ and $N=D^2$. (c,d) Memorization and generalization as the training length $\alpha=L/d$ and inference length $\widetilde\alpha=\widetilde L/d$ vary, with $N_0=24$ and $N=12288$; the diagonal dotted line marks equal lengths, and the horizontal dashed line marks $\widetilde\alpha=32$. Gold stars mark the shared parameter setting ($N=12288, N_0=24, \alpha=\widetilde\alpha=32$). Blue favors the kernel learner, red favors the feature learner, and black curves indicate equal errors.}
\vspace{-2mm}
\label{fig:learner_advantage_four_panel}
\end{figure}

We compare the two learners at finite $d,D,N_0,N,L,\widetilde L$, using the same task pool
and prompts, and evaluate the state equations at the corresponding values. 
Experiments use the common teacher:
\begin{equation}
\sigma_\star\!\left(\frac{a^\top x}{\sqrt d}\right)
=\frac{\sqrt3}{2}\frac{a^\top x}{\sqrt d}
+\frac{\sqrt{2}}{4}
\left(\frac{(a^\top x)^2}{d}-1\right).
\label{eq:common-teacher}
\end{equation}
For the kernel learner, we use the complete degree-one and degree-two spherical harmonic basis. Note that because the degree-one and degree-two harmonic basis contains the teacher, the kernel learner has no representation error (\(\nu_{\rm out}=0\)).

Unless stated otherwise, we fix $d=32$, $D=559$, $\Delta=0.01$ and $\lambda=0.001$. We first increase $N$ with a small fixed task pool, then vary task diversity and the training and inference lengths to compare the two learners across these settings. In this section we denote \(\alpha:=L/d\) and \(\widetilde\alpha:=\widetilde L/d\).
\begin{itemize}[leftmargin=1.5em]
\item \textbf{Numerics.} We use the attention block in Eq.~\eqref{eq:choice-attention}. In terms of the feature learner, we use the following two-stage protocol.
  \begin{itemize}
  \item \textit{Stage I.} Initialize the attention block, fix the scalar readout to $\sigma_0(u)=u$ by initializing a width-64 ReLU readout (which remains frozen during Stage I), and do full-batch L-BFGS.
  \item \textit{Stage II.} Freeze the learned attention block, independently draw $1024$ prompts\footnote{This choice is for a generic MLP readout. In general, Stage II requires far fewer samples than Stage I; for example, in our case fitting a quadratic readout is enough.} from the same task pool, and train the readout by AdamW followed by full-batch L-BFGS.
  \end{itemize}
For the kernel learner, we map each input to the fixed kernel basis $h(x)$ and optimize the attention block corresponding to Eq.~\eqref{eq:choice-attention} by full-batch L-BFGS.
We also compare them with the linear attention baseline, which is equivalent to the feature learner after Stage I. All three learners use the same Stage-I and inference prompts.  Each inference error averages $32768$ prompts.

\item \textbf{Replica curves.}
  \begin{itemize}
  \item \textit{Feature learner.} We solve Eq.~\eqref{eq:app-feature-state-equations}, substitute the order parameters into Eq.~\eqref{eq:rs-test-error} with the readout in Eq.~\eqref{eq:optimal-readout}.
  \item \textit{Kernel learner.} We solve Eqs.~\eqref{eq:SE-kernel} and substitute their order parameters into Eq.~\eqref{eq:kernel-errors-main}.
  \end{itemize}
\end{itemize}

\vspace{-3mm}
\subsection{Varying the number of training prompts}
\vspace{-1mm}
\label{sec:experiments-N}

We first compare the numerics and the replica prediction with different numbers of training prompts. We fix $N_0=24$ and $\alpha=\widetilde\alpha=32$, and vary $N$ from $288$ to $1200000$. Increasing $N$ allows us to compare the two learners both before and after their interpolation peaks.
In Figure \ref{fig:gamma-merged}, we can observe a good agreement between the numerics and the replica prediction. In Appendix~\ref{app:consistency} we further check the consistency of corresponding order parameters between experiments and replica prediction.
The feature learner, linear attention and kernel learner all show the familiar double-descent behavior \citep{lu2025asymptotic}, with interpolation peaks near $d^2$ and $D^2$, respectively. Stage II substantially reduces the interpolation peak for the feature learner, although a smaller peak remains. Our main interest, however, is the relative performance of the two learners across the different data regimes. We highlight the main observations below.

\begin{itemize}[leftmargin=1.5em]
\item \textbf{The feature learner reaches low error earlier.}
After its first interpolation peak, the feature learner (red line) reaches low error well before the kernel learner (blue line), since $D\gg d$. At $N=36864$, its memorization/generalization errors are $0.061/0.421$, compared with $0.088/0.675$ for the kernel learner, which later passes through its interpolation peak near $D^2$. The feature learner is therefore better over $d^2\ll N\ll D^2$. 
\item \textbf{The kernel learner wins on memorization at large $N$, but fails on generalization.} At $N=1200000$, the kernel learner achieves memorization error $0.037$, below the feature learner's $0.046$. Generalization still favors the feature learner, with mean error $0.354$ compared with $0.657$ for the kernel learner. In fact, as $N \to \infty$, Results \ref{res:kernel} and \ref{res:feature} give the limiting errors:
\begin{equation*}
\begin{aligned}
\mathcal E_{\mathrm M}^{\mathrm{feature}}&\simeq0.03536,
&\mathcal E_{\mathrm G}^{\mathrm{feature}}&\simeq0.35682,
&\mathcal E_{\mathrm M}^{\mathrm{kernel}}&\simeq0.02264,
&\mathcal E_{\mathrm G}^{\mathrm{kernel}}&\simeq0.67073.
\end{aligned}
\end{equation*}
The kernel learner can thus learn the training tasks more accurately with more trainable parameters once enough prompts are available, but generalizes worse at low task diversity.
\end{itemize}

\vspace{-2mm}
\subsection{Varying task diversity and context lengths}
\vspace{-1mm}
\label{sec:experiments-N0}
The previous comparison uses a fixed value of $N_0,\alpha,\widetilde\alpha$. In this section, we vary these hyperparameters to obtain a more comprehensive phase diagram  (Figure \ref{fig:learner_advantage_four_panel}).

Figure \ref{fig:learner_advantage_four_panel}(a,b) shows that the kernel learner has a memorization advantage when the training set size $N$ is large and the task pool $N_0$ is small. Increasing $N_0$ removes this kernel memorization advantage. At smaller $N$, the relative performance is further affected by the two interpolation peaks near $N\sim d^2$ and $N\sim D^2$. Hence, the kernel learner has a generalization advantage only in the low-\(N\) regime, where the feature learner is affected by its \(N\sim d^2\) interpolation peak. This is consistent with the observation in Figure \ref{fig:gamma-merged}.
We next vary the training and inference lengths in Figure \ref{fig:learner_advantage_four_panel}(c,d).

\vspace{-1mm}
\paragraph{Longer training prompts do not always improve inference performance.} 
Figure \ref{fig:learner_advantage_four_panel}(c,d) suggests that the relative advantage depends strongly on the training and inference context lengths. Two regimes can favor the kernel learner. First, under a large context-length mismatch, the feature learner can deteriorate (see Theorem \ref{theo:comparison}). Second, when both $\alpha$ and $\widetilde{\alpha}$ are large, the kernel learner can reconstruct the teacher with smaller fluctuations (see Theorem \ref{theo:comparison}), since the kernel basis is complete.

Overall, these replica predictions show that the relative advantage of feature and kernel learners depends jointly on the amount of pretraining data $N$, task diversity $N_0$, and the training and inference context lengths $\alpha,\widetilde\alpha$. Additional numerical results are provided in Appendix \ref{app:vary_N0}, where we compare the replica prediction with the experiments with different $N_0$.

\vspace{-1mm}
\subsection{Learning different harmonic degrees}
\vspace{-1mm}
Theorem \ref{theo:comparison} predicts that the two learners have different context scalings. Figure~\ref{fig:multiscale-teachers} illustrates this difference for the common teacher \eqref{eq:common-teacher}. The solid curves are obtained by substituting $d=16$ and $L=16^s$ into Eqs. \eqref{eq:kernel-population} and \eqref{eq:feature-population}, which agree closely with the numerical experiments. The feature learner has lower error throughout the displayed range. The dashed curves evaluate the same formulas at $d=10^6$ and show the separation of sample complexities emerge as $d$ becomes larger. The kernel learner acquires the degree-one and degree-two components at the separate scales $L\asymp d$ and $L\asymp d^2$, whereas the feature learner recovers both components at the same time.

\begin{wrapfigure}{r}{0.48\linewidth}
\vspace{-5mm}
  \centering
  \includegraphics[width=\linewidth]{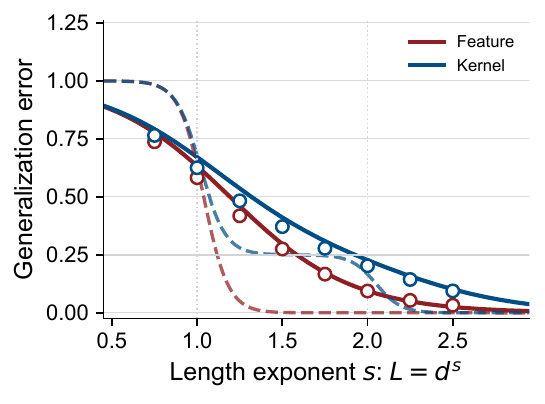}
  \vspace{-7mm}
  \caption{Learning different harmonic degrees. The horizontal coordinate is $s=\log L/\log d$, with equal training and inference lengths $L=\widetilde L=d^s$. The teacher is \eqref{eq:common-teacher}. Circles show experiments at $d=16$, $N_0=8192$, $N=1200000$, $\Delta=0$, and $\lambda=0$. Solid curves are Theorem \ref{theo:comparison} (Eqs.~\eqref{eq:kernel-population} and \eqref{eq:feature-population}) evaluated at $d=16$. Dashed curves are the same formulas evaluated at $d=10^6$ and display the separation between the scales $L\asymp d$ and $L\asymp d^2$.}
  \label{fig:multiscale-teachers}
\vspace{-12mm}
\end{wrapfigure}

\section{Conclusion and Future Work}
In this work, we studied nonlinear ICL by comparing the kernel learner and the feature learner on the same single-index task. Using replica calculations, we characterized how their memorization and generalization errors depend jointly on pretraining size, task diversity, and the training and inference context lengths. In the population limit, we rigorously prove two context-scaling mechanisms. More broadly, our results suggest that the relative advantage of kernel and feature learning in ICL is determined by how the architectures interact with the amount of pretraining data, task diversity, and the information available in the context.

Several directions remain open. Our analysis focuses on one-layer linear attention, spherical-harmonic features, and an idealized two-stage training procedure in which the scalar readout is learned from sufficient additional data. It would be useful to understand whether the same separation persists for random features or learned embeddings, jointly trained readouts, softmax attention, and deeper Transformer architectures. On the theoretical side, important directions include proving the finite-dimensional results and analyzing ICL beyond single-index targets with $\mu_1\neq0$.

\section*{Acknowledgements}
We acknowledge funding from the Swiss National Science Foundation grants SNSF SMArtNet (grant number 212049), the Simons Collaboration on the Physics of Learning and Neural Computation via the Simons Foundation grant (\#1257413 (LZ)), the Swiss State Secretariat for Education, Research and Innovation (SERI), which funded the ERC Advanced Grant SPASM-net, and the EPFL AI Center PhD Fellowship Program 2026 (YX).

\bibliographystyle{plain}
\bibliography{ref}

\appendix
\tableofcontents
\section{Derivation of the kernel learner}
\label{app:kernel-details}
\renewcommand{\theResult}{\thesection.\arabic{Result}}
\setcounter{Result}{0}
\subsection{Teacher expansion}
\label{app:expansion}
We first express the nonlinear teacher in the spherical-harmonic basis used by the kernel learner. Let
\begin{equation}
g_d:=\frac{a^\top x}{\sqrt d},
\qquad
a,x\overset{\mathrm{i.i.d.}}{\sim}
\Unif(\mathbb S^{d-1}(\sqrt d)).
\label{eq:app-spherical-score}
\end{equation}
The random variable $g_d$ has mean zero and variance one and converges in
distribution to $G\sim\mathcal N(0,1)$ as $d\to\infty$.

For each $k\geq0$, let $P_{k,d}$ denote the degree-$k$ Gegenbauer polynomial
orthonormal with respect to the distribution of $g_d$, namely
\begin{equation}
\mathbb E\!\left[P_{j,d}(g_d)P_{k,d}(g_d)\right]
=\delta_{jk}.
\end{equation}
With the spherical harmonics normalized with respect to the uniform probability
measure on $\mathbb S^{d-1}(\sqrt d)$, the spherical addition theorem \citep{frye2012spherical} gives
\begin{equation}
P_{k,d}\!\left(\frac{a^\top x}{\sqrt d}\right)
=
\frac{1}{\sqrt{D_{k,d}}}
h_k(a)^\top h_k(x).
\label{eq:app-addition-theorem}
\end{equation}

Define the Gegenbauer coefficients $c_{k,d}:=\mathbb E\!\left[\sigma_\star(g_d)P_{k,d}(g_d)\right]$. Then, in $L^2$ with respect to the distribution of $g_d$,
\begin{equation}
\begin{aligned}
\sigma_\star(g_d)=\sum_{k\geq0}c_{k,d}P_{k,d}(g_d)=\sum_{k\geq0}
\frac{c_{k,d}}{\sqrt{D_{k,d}}}
h_k(a)^\top h_k(x).
\end{aligned}
\label{eq:app-harmonic-expansion-full}
\end{equation}

Let $H_k$ denote the normalized Hermite polynomial and define the Hermite coefficients of the teacher by
\begin{equation}
c_k:=\mathbb E[\sigma_\star(G)H_k(G)].
\end{equation}
For every fixed $k$, the normalized Gegenbauer polynomial $P_{k,d}$ converges to $H_k$ as $d\to\infty$ \citep{mei2021learning}. We assume throughout that
$\sigma_\star$ is sufficiently regular so that
\begin{equation}
c_{k,d}\longrightarrow c_k
\quad\text{for every fixed }k,
\qquad
\mathbb E[\sigma_\star(g_d)^2]
\longrightarrow
\mathbb E[\sigma_\star(G)^2]=1.
\label{eq:app-gegenbauer-hermite-convergence}
\end{equation}

We now separate the harmonic degrees represented by the learner from the
remaining component. For a fixed finite set
$\mathcal K\subset\mathbb N_{\ge1}$, define
\begin{equation}
\begin{aligned}
f_{\mathcal K,d}(a,x)
&:=
\sum_{k\in\mathcal K}
\frac{c_{k,d}}{\sqrt{D_{k,d}}}
h_k(a)^\top h_k(x),\\
r_{\mathcal K,d}(a,x)
&:=
\sigma_\star\!\left(\frac{a^\top x}{\sqrt d}\right)
-
f_{\mathcal K,d}(a,x).
\end{aligned}
\label{eq:app-teacher-basis-split}
\end{equation}
By orthogonality of the Gegenbauer polynomials, for every fixed $a$,
\begin{equation}
\mathbb E_x\!\left[
r_{\mathcal K,d}(a,x)h_k(x)
\right]
=0,
\qquad k\in\mathcal K,
\label{eq:app-residual-orthogonality}
\end{equation}
and
\begin{equation}
\mathbb E_x[r_{\mathcal K,d}(a,x)^2]
=
\sum_{k\notin\mathcal K}c_{k,d}^2.
\label{eq:app-residual-finite-d-variance}
\end{equation}
Using Parseval's identity together with
\eqref{eq:app-gegenbauer-hermite-convergence},
\begin{equation}
\begin{aligned}
\sum_{k\notin\mathcal K}c_{k,d}^2=
\mathbb E[\sigma_\star(g_d)^2]-\sum_{k\in\mathcal K}c_{k,d}^2\longrightarrow
1-\sum_{k\in\mathcal K}c_k^2=\sum_{k\notin\mathcal K}c_k^2.
\end{aligned}
\label{eq:app-residual-variance-limit}
\end{equation}
Define
\begin{equation}
\theta(a):=\bigoplus_{k\in\mathcal K}
c_k\sqrt{\frac{D}{D_{k,d}}}\,h_k(a).
\label{eq:app-kernel-task-vector}
\end{equation}
Since $\mathcal K$ is finite and $c_{k,d}\to c_k$ for every
$k\in\mathcal K$,
\begin{equation}
\begin{aligned}
\mathbb E_{a,x}\left[
f_{\mathcal K,d}(a,x)
-\frac{1}{\sqrt D}\theta(a)^\top h(x)
\right]^2
=
\sum_{k\in\mathcal K}(c_{k,d}-c_k)^2
\longrightarrow0.
\end{aligned}
\label{eq:app-modeled-teacher-linear-error}
\end{equation}
Hence
\begin{equation}
f_{\mathcal K,d}(a,x)
=
\frac{1}{\sqrt D}\theta(a)^\top h(x)
+o_{L^2}(1).
\label{eq:app-modeled-teacher-linear}
\end{equation}
The residual
$r_{\mathcal K,d}(a,x)$, however, is in general neither Gaussian nor
independent of $h(x)$. To obtain a tractable effective model, we make an
additional Gaussian-equivalence approximation: at the level of the
high-dimensional learning problem, we replace the residual by an independent Gaussian random variable with the same variance. This approximation is motivated by Gaussian-equivalence results for high-dimensional random feature models
\citep{hu2024asymptotics,aguirre2025random,defilippis2024dimension}.

Writing
\begin{equation}
\nu_{\rm in}:=\sum_{k\in\mathcal K}c_k^2,
\qquad
\nu_{\rm out}:=\sum_{k\notin\mathcal K}c_k^2
=1-\nu_{\rm in},
\label{eq:app-kernel-in-out-variance}
\end{equation}
we therefore approximate the original teacher by the effective linear model
\begin{equation}
y
=
\frac{1}{\sqrt D}\theta(a)^\top h(x)
+\widetilde\epsilon,
\qquad
\widetilde\epsilon
\sim
\mathcal N(0,\Delta+\nu_{\rm out}),
\label{eq:app-effective-linear-teacher}
\end{equation}
where $\widetilde\epsilon$ is taken to be independent of $(a,x)$ and of the
feature vector $h(x)$.

\subsection{Fixed point equations}
\begin{Result}[kernel learner: state equations]
\label{res:app-kernel-state-equations}

Given $(\alpha_K,\gamma_K,\kappa_K,\nu_{\rm in},\Delta,\lambda,S_K)$,
the state equations are
\begin{equation}
\begin{aligned}
\widehat m_K=\widehat\chi_K
&=\frac{\gamma_K\nu_{\rm in}}
{\nu_{\rm in}(1+\chi_K)+\frac{1+\Delta}{\alpha_K}\overline\chi_K},
\\[1mm]
\widehat q_K
&=\frac{\gamma_K\nu_{\rm in}
\left[1+\Delta-2\nu_{\rm in}m_K+\nu_{\rm in}q_K
+\frac{1+\Delta}{\alpha_K}\overline q_K\right]}
{\left[\nu_{\rm in}(1+\chi_K)
+\frac{1+\Delta}{\alpha_K}\overline\chi_K\right]^2},
\\[1mm]
\chi_K
&=\frac{1}{\widehat\chi_K}\left[1-\zeta_K r_K(\zeta_K)\right],
\\
\overline\chi_K
&=\frac{\nu_{\rm in}}{\widehat\chi_K}
\left[\frac{1-\kappa_K}{\zeta_K}+\kappa_K r_K(\zeta_K)\right],
\\
m_K
&=\frac{\widehat m_K}{\nu_{\rm in}\widehat\chi_K}
(\nu_{\rm in}-\zeta_K+\zeta_K^2r_K(\zeta_K)),
\\
q_K
&=\frac{\widehat m_K^2}{\nu_{\rm in}\widehat\chi_K^2}
(\nu_{\rm in}-2\zeta_K+3\zeta_K^2r_K(\zeta_K)-\zeta_K^3t_K(\zeta_K))
\\&\quad+\frac{\widehat q_K}{\widehat\chi_K^2}
\left[1+(\beta_K-2\zeta_K)r_K(\zeta_K)+(\zeta_K^2-\beta_K\zeta_K)t_K(\zeta_K)\right],
\\
\overline q_K
&=\frac{\kappa_K\widehat m_K^2}{\widehat\chi_K^2}
\left[1-2\zeta_Kr_K(\zeta_K)+\zeta_K^2t_K(\zeta_K)\right]\\
&\quad+\frac{\nu_{\rm in}\widehat q_K}{\widehat\chi_K^2}\left[(1-\kappa_K)\frac{\beta_K}{\zeta_K^2}+\kappa_K\left\{r_K(\zeta_K)+(\beta_K-\zeta_K)t_K(\zeta_K)\right\}\right],
\end{aligned}
\label{eq:SE-kernel}
\end{equation}
where $r_K(\zeta):=\lim_{N_0\to\infty}\frac1{N_0}\Tr(G_K+\zeta I_{N_0})^{-1}$ and $t_K(\zeta):=\lim_{N_0\to\infty}\frac1{N_0}\Tr(G_K+\zeta I_{N_0})^{-2}$. Here we define $G_K:=\frac1D\Theta_K^\top\Theta_K\in\mathbb{R}^{N_0\times N_0}$ with $\Theta_K:=
\begin{bmatrix}
\theta(a^1)&\cdots&\theta(a^{N_0})
\end{bmatrix}
\in\mathbb R^{D\times N_0}$.
We also define $\beta_K:=(1+\Delta)\frac{\kappa_K}{\alpha_K}$ and $\zeta_K:=\beta_K+\frac{\kappa_K\lambda}{\widehat\chi_K}$.
These equations determine five order parameters
$(m_K,q_K,\overline q_K,\chi_K,\overline\chi_K)$ and three conjugates
$(\widehat m_K,\widehat q_K,\widehat\chi_K)$.
\end{Result}

\subsection{Limiting spectrum}
Result \ref{res:app-kernel-state-equations} does not specialize to a particular spectrum of $G_K$. Meanwhile, we can derive the analytical limiting spectrum of $G_K$ to simplify the equations.
In Section \ref{sec:experiments}, we use a different scaling $\frac{N_0}{d}\longrightarrow \varrho\in(0,\infty)$ for our comparison with the feature learner. Next we derive the limiting resolvents $r_K(\zeta)$ and $t_K(\zeta)$ in this scaling.
Let
\begin{equation}
\Theta_K:=
\begin{bmatrix}
\theta(a^1)&\cdots&\theta(a^{N_0})
\end{bmatrix}
\in\mathbb R^{D\times N_0},
\qquad
S_K=\frac1{N_0}\Theta_K\Theta_K^\top,
\end{equation}
and introduce the \(N_0\times N_0\) task Gram matrix $G_K:=\frac1D\Theta_K^\top\Theta_K$.
We retain the notation $\kappa_K:=\frac{N_0}{D}$.
By the definition of \(\theta(a)\),
\begin{equation}
G_K=\sum_{k\in\mathcal K}
c_k^2\frac1{D_{k,d}}H_k^\top H_k,
\qquad
H_k:=
\begin{bmatrix}
h_k(a^1)&\cdots&h_k(a^{N_0})
\end{bmatrix}.
\label{eq:app-kernel-Gram-harmonic-decomposition}
\end{equation}
In the regime \(N_0/d\to\varrho\), the degree-one block has a nontrivial Marchenko--Pastur spectrum, whereas every degree \(k\ge2\) satisfies
\(N_0/D_{k,d}\to0\) and contributes asymptotically as an identity matrix.
More precisely, the polynomial-regime spectral decomposition
of \cite{misiakiewicz2022spectrum} gives
\begin{equation}
G_K=c_{1,\mathcal K}^2\frac1dH_1^\top H_1+\nu_{>1,\mathcal K} I_{N_0}+o_{\mathrm{op}}(1),
\label{eq:app-kernel-Gram-linearization}
\end{equation}
where
\begin{equation}
c_{1,\mathcal K}^2:=c_1^2\mathbf 1_{\{1\in\mathcal K\}},
\qquad
\nu_{>1,\mathcal K}:=\nu_{\rm in}-c_{1,\mathcal K}^2.
\end{equation}
Here \(H_1=[a^1,\ldots,a^{N_0}]\).
Let $\nu_{\kappa}$ denote the Marchenko--Pastur distribution
\begin{equation}
\nu_{\kappa}(\dd s)=\bigl(1-\min\{1,\kappa\}\bigr)\delta_0(\dd s)+\1_{[s_-,s_+]}(s)\frac{\kappa\sqrt{(s_+-s)(s-s_-)}}{2\pi s}\dd s,\qquad s_\pm=(1\pm\kappa^{-1/2})^2.
\label{eq:MP-law}
\end{equation}
The empirical spectral distribution of \(d^{-1}H_1^\top H_1\) converges to \(\nu_{1/\varrho}\). Hence the empirical spectral distribution of \(G_K\) converges to the law of
\begin{equation}
g_{\varrho}=\nu_{>1,\mathcal K}+c_{1,\mathcal K}^2 s,
\qquad
s\sim\nu_{1/\varrho}.
\label{eq:app-kernel-Gram-limit-law}
\end{equation}
In particular, if \(1\notin\mathcal K\), the limiting distribution reduces to the point mass at \(\nu_{\rm in}\).
The deterministic limits of the resolvents are therefore
\begin{equation}
\mathcal R_{\varrho}(\zeta):=\int\frac{\nu_{1/\varrho}(\mathrm ds)}
{\nu_{>1,\mathcal K}+c_{1,\mathcal K}^2s+\zeta},
\label{eq:app-kernel-Gram-resolvent-limit}
\end{equation}
and
\begin{equation}
\mathcal T_{\varrho}(\zeta):=\int
\frac{\nu_{1/\varrho}(\mathrm ds)}
{\left(\nu_{>1,\mathcal K}+c_{1,\mathcal K}^2s+\zeta\right)^2}.
\label{eq:app-kernel-Gram-second-resolvent-limit}
\end{equation}
For every fixed \(\zeta>0\),
\begin{equation}
r_K(\zeta)\longrightarrow\mathcal R_{\varrho}(\zeta),
\qquad
t_K(\zeta)\longrightarrow
\mathcal T_{\varrho}(\zeta).
\end{equation}
Substituting
\(r_K(\zeta_K)\simeq\mathcal R_{\varrho}(\zeta_K)\) and
\(t_K(\zeta_K)\simeq\mathcal T_{\varrho}(\zeta_K)\)
into \eqref{eq:SE-kernel} gives the state equations used in our numerical experiments.

\subsection{Reduction to Takanami et al.}
\label{app:kernel-reduction}
We now relate the kernel learner to the linear in-context learning model analyzed
by \cite{takanami2025learning}. The correspondence consists of
an exact algebraic equivalence,
followed by a Gaussian-equivalence approximation for the effective
statistics entering the replica calculation.

We first remove the normalization $\nu_{\rm in}\neq1$. Recall from
Eq.~\eqref{eq:app-effective-linear-teacher} that the effective teacher obtained
in the previous subsection is
\begin{equation}
y
=
\frac{1}{\sqrt D}\theta(a)^\top h(x)
+\widetilde\epsilon,
\qquad
\widetilde\epsilon\sim
\mathcal N(0,\Delta+\nu_{\rm out}).
\label{eq:app-kernel-effective-model-before-rescaling}
\end{equation}
For $\nu_{\rm in}>0$, define
\begin{equation}
\theta'(a):=\frac{\theta(a)}{\sqrt{\nu_{\rm in}}},
\qquad
y':=\frac{y}{\sqrt{\nu_{\rm in}}},
\qquad
\Delta':=\frac{\Delta+\nu_{\rm out}}{\nu_{\rm in}},
\qquad
S_K':=\frac{S_K}{\nu_{\rm in}},
\qquad
\lambda':=\frac{\lambda}{\nu_{\rm in}}.
\label{eq:app-kernel-effective-rescaling}
\end{equation}
Since $H^K$ is linear in the labels, the rescaled labels satisfy
$(H^K)'=H^K/\sqrt{\nu_{\rm in}}$, and hence
\begin{equation}
\begin{aligned}
\sum_{\mu=1}^{N}
\left[
\Tr\!\left(W(H_\mu^K)^\top\right)-y_0^\mu
\right]^2
+\lambda N_0\|W\|_F^2
=
\nu_{\rm in}
\left\{
\sum_{\mu=1}^{N}
\left[
\Tr\!\left(W((H_\mu^K)')^\top\right)-(y_0^\mu)'
\right]^2
+\lambda' N_0\|W\|_F^2
\right\}.
\end{aligned}
\label{eq:app-kernel-objective-rescaling}
\end{equation}
Therefore the minimizing matrix $W$ is unchanged. It is thus sufficient to
derive the result in the normalized setting
\begin{equation}
\nu_{\rm in}=1,
\qquad
\nu_{\rm out}=0,
\qquad
\|\theta(a)\|^2=D,
\label{eq:app-kernel-normalized-setting}
\end{equation}
with the understanding that the replacements in
Eq.~\eqref{eq:app-kernel-effective-rescaling} are restored at the end.
The last identity follows from
$\|h_k(a)\|^2=D_{k,d}$ and the definition of $\theta(a)$.

\paragraph{Correspondence of the linearized variables.}
In the notation of \cite{takanami2025learning}, define
\begin{equation}
x^{\rm Tak}:=\frac{h(x)}{\sqrt D},
\qquad
w^{\rm Tak}:=\theta(a),
\qquad
\sigma_{\rm Tak}^2:=\Delta.
\label{eq:app-takanami-basic-map}
\end{equation}
Then
\begin{equation}
\mathbb E[x^{\rm Tak}]=0,
\qquad
\mathbb E[x^{\rm Tak}(x^{\rm Tak})^\top]
=\frac1D I_D,
\label{eq:app-takanami-input-moments}
\end{equation}
and the effective teacher takes the same linear form
\begin{equation}
y=(w^{\rm Tak})^\top x^{\rm Tak}+\epsilon,
\qquad
\epsilon\sim\mathcal N(0,\Delta).
\label{eq:app-takanami-linear-model-map}
\end{equation}
Moreover, the effective feature matrix used by the attention model is exactly
\begin{equation}
\begin{aligned}
H^{\rm Tak}=\frac{D}{L}\,
x_0^{\rm Tak}\sum_{\ell=1}^{L}
y_\ell(x_\ell^{\rm Tak})^\top=\frac1Lh(x_0)\sum_{\ell=1}^{L}y_\ell h(x_\ell)^\top=H^K.
\end{aligned}
\label{eq:app-kernel-prompt-statistic-map}
\end{equation}
Thus the correspondence between the task matrices is also exact:
\begin{equation}
S^{\rm Tak}=\frac1{N_0}\sum_{k=1}^{N_0}
w^k(w^k)^\top=\frac1{N_0}\sum_{k=1}^{N_0}
\theta(a^k)\theta(a^k)^\top=S_K.
\label{eq:app-kernel-task-matrix-map}
\end{equation}

For replicas $W^a,W^b$, the corresponding macroscopic overlaps are
\begin{equation}
m_K^a:=\frac1D\Tr(W^aS_K),
\qquad
Q_K^{ab}:=\frac1D\Tr(W^aS_K(W^b)^\top),
\qquad
\overline Q_K^{ab}:=\frac1D\Tr(W^a(W^b)^\top),
\label{eq:app-kernel-order-parameter-map}
\end{equation}
which coincide with the order parameters used in the linear model of
\cite{takanami2025learning}.

\paragraph{Gaussian equivalence for the effective statistic.}
The correspondence above does not imply that $x^{\rm Tak}=h(x)/\sqrt D$
is Gaussian. We therefore require one additional approximation before
transferring the replica calculation.

For a fixed task $a$, the first moments needed below are:
\begin{equation}
\mathbb E[y_0\mid a]=0,
\qquad
\mathbb E[y_0^2\mid a]=1+\Delta,
\label{eq:app-kernel-label-moments}
\end{equation}
and
\begin{equation}
\mathbb E[y_0H^K_{ij}\mid a]
=
\mathbb E[y_0h_i(x_0)\mid a]\,
\mathbb E[yh_j(x)\mid a]
=
\frac{\theta_i(a)\theta_j(a)}{D}.
\label{eq:app-kernel-label-prompt-moment}
\end{equation}
For the second moment of $H^K$, the calculation gives
\begin{equation}
\mathbb E[H^K_{ij}H^K_{i'j'}\mid a]=\delta_{ii'}
\left[\frac{L-1}{LD}\theta_j(a)\theta_{j'}(a)+\frac1L B_{jj'}(a)\right],
\label{eq:app-kernel-prompt-second-moment-exact}
\end{equation}
where
\begin{equation}
B(a):=\mathbb E_x[y^2h(x)h(x)^\top\mid a].
\label{eq:app-kernel-B-matrix}
\end{equation}
By $\|h(x)\|^2=D$ for every $x$, we have
\begin{equation}
\frac1D\Tr B(a)=\mathbb E_x[y^2\mid a]=1+\Delta.
\label{eq:app-kernel-B-trace}
\end{equation}
Moreover, rotational invariance gives
\begin{equation}
\mathbb E_a[B(a)]=(1+\Delta)I_D.
\label{eq:app-kernel-B-task-average}
\end{equation}
Conditioned on $a$, however, $B(a)$ is not asymptotically isotropic. Nevertheless, for the following we replace $B(a)$ with its expectation (i.e., a self-averaging ansatz) and write 
\begin{equation}
\mathbb E[H^K_{ij}H^K_{i'j'}\mid a]
\simeq\frac{\delta_{ii'}}D\left(\theta_j(a)\theta_{j'}(a)+\frac{D(1+\Delta)}L\delta_{jj'}\right).
\label{eq:app-kernel-prompt-second-moment-GE}
\end{equation}

For the following, we invoke the Gaussian-equivalence step used in \cite{takanami2025learning}: the effective learning problem is characterized through the leading first- and second-order statistics, while higher-order
correlations are neglected in the high-dimensional limit. 
To make the connection explicit, define the task-dependent quadratic teacher matrix
\begin{equation}
W_{\rm teacher}(a):=\frac{\theta(a)\theta(a)^\top}{D}
\label{eq:app-kernel-quadratic-teacher}
\end{equation}
and the surrogate query label
\begin{equation}
\overline y_0:=\Tr\!\left(
W_{\rm teacher}(a)(H^K)^\top
\right)
+\epsilon_0,
\qquad
\epsilon_0\sim\mathcal N(0,\Delta).
\label{eq:app-kernel-surrogate-query-label}
\end{equation}
Under Eq.~\eqref{eq:app-kernel-prompt-second-moment-GE},
$(\overline y_0,H^K)$ and $(y_0,H^K)$ have the same leading first- and
second-order statistics. Hence, following the Gaussian-equivalence argument
of \cite{takanami2025learning}, we replace the original training problem by
the corresponding many-teacher regression problem.

More specifically, for replicas $\{W^a\}_{a=1}^n$, introduce the scalar fields
\begin{equation}
u^a
:=
\Tr\!\left[
\left(
W_{\rm teacher}(a)-W^a
\right)(H^K)^\top
\right]
+\epsilon_0.
\label{eq:app-kernel-replica-fields}
\end{equation}
Under the Gaussian-equivalence approximation, the vector
$(u^1,\ldots,u^n)$ is replaced by a centered Gaussian vector with covariance
determined only by the overlaps in
Eq.~\eqref{eq:app-kernel-order-parameter-map}. After averaging over the task
pool, its covariance is
\begin{equation}
\mathcal M_K^{ab}
=
Q_K^{ab}
-m_K^a-m_K^b
+1+\Delta
+
\frac{1+\Delta}{\alpha_K}\overline Q_K^{ab},
\qquad
\alpha_K:=\frac{L}{D}.
\label{eq:app-kernel-replica-field-covariance}
\end{equation}
This is precisely the covariance entering the data-averaged replicated
partition function of \cite{takanami2025learning}.

Consequently, under the same replica-symmetric ansatz, the remainder of the
replica calculation has the same form as in the linear model, with the
identifications
\begin{equation}
M_0\leftrightarrow N_0,
\qquad
M\leftrightarrow N,
\qquad
\alpha\leftrightarrow\alpha_K,
\qquad
\kappa\leftrightarrow\kappa_K,
\qquad
\gamma\leftrightarrow\gamma_K,
\qquad
S\leftrightarrow S_K,
\qquad
\sigma^2\leftrightarrow\Delta.
\label{eq:app-kernel-takanami-parameter-map}
\end{equation}
Restoring the original normalization with the replacements
\begin{equation}
S_K\mapsto\frac{S_K}{\nu_{\rm in}},
\qquad
\Delta\mapsto
\frac{\Delta+\nu_{\rm out}}{\nu_{\rm in}},
\qquad
\lambda\mapsto\frac{\lambda}{\nu_{\rm in}},
\label{eq:app-kernel-restore-normalization}
\end{equation}
this yields the saddle-point equations 
\begin{equation}
\begin{aligned}
\widehat m_K=\widehat\chi_K
&=\frac{\gamma_K\nu_{\rm in}}
{\nu_{\rm in}(1+\chi_K)+\frac{1+\Delta}{\alpha_K}\overline\chi_K},
\\[1mm]
\widehat q_K
&=\frac{\gamma_K\nu_{\rm in}
\left[1+\Delta-2\nu_{\rm in}m_K+\nu_{\rm in}q_K
+\frac{1+\Delta}{\alpha_K}\overline q_K\right]}
{\left[\nu_{\rm in}(1+\chi_K)
+\frac{1+\Delta}{\alpha_K}\overline\chi_K\right]^2},
\\[1mm]
\chi_K
&=\frac{1}{\kappa_K\widehat\chi_KD}\Tr(S_KA_K^{-1}),
\\
\overline\chi_K
&=\frac{\nu_{\rm in}}{\kappa_K\widehat\chi_KD}\Tr(A_K^{-1}),
\\
m_K
&=\frac{\widehat m_K}{\nu_{\rm in}\widehat\chi_KD}
\Tr(S_K^2A_K^{-1}),
\\
q_K
&=\frac{\widehat m_K^2}{\nu_{\rm in}\widehat\chi_K^2D}
\Tr(S_K^3A_K^{-2})
+\frac{\widehat q_K}{\kappa_K\widehat\chi_K^2D}
\Tr\!\left[S_K\left(S_K+\frac{1+\Delta}{\alpha_K}I_D\right)A_K^{-2}\right],
\\
\overline q_K
&=\frac{\widehat m_K^2}{\widehat\chi_K^2D}
\Tr(S_K^2A_K^{-2})
+\frac{\nu_{\rm in}\widehat q_K}{\kappa_K\widehat\chi_K^2D}
\Tr\!\left[\left(S_K+\frac{1+\Delta}{\alpha_K}I_D\right)A_K^{-2}\right],
\end{aligned}
\end{equation}
where
\begin{equation}
A_K:=S_K+\left(\frac{1+\Delta}{\alpha_K}
+\frac{\lambda}{\widehat\chi_K}\right)I_D.
\end{equation}
We finally obtain \eqref{eq:SE-kernel} using
\begin{equation}
\begin{aligned}
\frac1D\Tr(A_K^{-1})
&=
\kappa_K
\left[
\frac{1-\kappa_K}{\zeta_K}
+
\kappa_K r_K(\zeta_K)
\right],
\\
\frac1D\Tr(S_KA_K^{-1})
&=
\kappa_K
\left[
1-\zeta_K r_K(\zeta_K)
\right],
\\
\frac1D\Tr(S_K^2A_K^{-1})
&=
\nu_{\rm in}
-\zeta_K
+\zeta_K^2r_K(\zeta_K),
\\
\frac1D\Tr(S_K^3A_K^{-2})
&=
\nu_{\rm in}
-2\zeta_K
+3\zeta_K^2r_K(\zeta_K)
-\zeta_K^3t_K(\zeta_K),
\\
\frac1D\Tr(S_K^2A_K^{-2})
&=
\kappa_K
\left[
1-2\zeta_Kr_K(\zeta_K)
+\zeta_K^2t_K(\zeta_K)
\right],
\\
\frac1D\Tr\!\left[
S_K(S_K+b_KI_D)A_K^{-2}
\right]
&=
\kappa_K
\left[
1+(\beta_K-2\zeta_K)r_K(\zeta_K)
+
(\zeta_K^2-\beta_K\zeta_K)t_K(\zeta_K)
\right],
\\
\frac1D\Tr\!\left[
(S_K+b_KI_D)A_K^{-2}
\right]
&=
\kappa_K
\left[
(1-\kappa_K)\frac{\beta_K}{\zeta_K^2}
+
\kappa_K
\left\{
r_K(\zeta_K)
+
(\beta_K-\zeta_K)t_K(\zeta_K)
\right\}
\right],
\end{aligned}
\label{eq:app-kernel-Gram-trace-reduction}
\end{equation}
where $b_K:=\frac{1+\Delta}{\alpha_K}$ and we also used the identity
\begin{equation}
\frac1{N_0}\Tr G_K=\frac1{DN_0}\Tr(\Theta_K^\top\Theta_K)=\nu_{\rm in}.
\end{equation}

We emphasize that the reduction in this subsection concerns the pretraining problem and its saddle-point equations. The inference errors require a separate calculation because the original test labels are noiseless and there is an additional irreducible term;
this is carried out in the following subsection.

\subsection{Derivation of the main results}
We now derive the distribution of the learned attention matrix and the
memorization and generalization errors stated in Result~\ref{res:kernel}.

Recall that
\begin{equation}
\nu_{\rm in}:=\sum_{k\in\mathcal K}c_k^2,
\qquad
\nu_{\rm out}:=\sum_{k\notin\mathcal K}c_k^2,
\qquad
\nu_{\rm in}+\nu_{\rm out}=1,
\end{equation}
and
\begin{equation}
S_K:=\frac1{N_0}
\sum_{k=1}^{N_0}
\theta(a^k)\theta(a^k)^\top.
\end{equation}
The reduction in the previous subsection maps the kernel learner to the normalized linear problem with
\begin{equation}
S_K'=\frac{S_K}{\nu_{\rm in}},
\qquad
\Delta'=\frac{\Delta+\nu_{\rm out}}{\nu_{\rm in}},
\qquad
\lambda'=\frac{\lambda}{\nu_{\rm in}}.
\label{eq:app-kernel-normalized-parameters}
\end{equation}
Applying \cite[Result E.1]{takanami2025learning} to the normalized problem gives
\begin{equation}
\widehat W_K\overset{\mathrm d}=\frac1{\widehat\chi_K}
\left[\widehat m_KS_K'+\sqrt{\frac{\widehat q_K}{N_0}}\,Z_K\left(S_K'+\frac{1+\Delta'}{\alpha_K}I_D
\right)^{1/2}\right]
\left[S_K'+\left(\frac{1+\Delta'}{\alpha_K}+\frac{\lambda'}{\widehat\chi_K}
\right)I_D\right]^{-1},
\label{eq:app-kernel-learned-matrix-normalized}
\end{equation}
where $(Z_K)_{ij}\overset{\mathrm{iid}}\sim\mathcal N(0,1)$. Substituting Eq.~\eqref{eq:app-kernel-normalized-parameters} into
Eq.~\eqref{eq:app-kernel-learned-matrix-normalized} yields
\begin{equation}
\boxed{
\widehat W_K\overset{\mathrm d}=
\frac1{\widehat\chi_K}\left[\widehat m_KS_K+\sqrt{\frac{\nu_{\rm in}\widehat q_K}{N_0}}\,Z_K\left(S_K+\frac{1+\Delta}{\alpha_K}I_D\right)^{1/2}\right]\left[S_K+\left(\frac{1+\Delta}{\alpha_K}+\frac{\lambda}{\widehat\chi_K}\right)I_D\right]^{-1}}.
\label{eq:app-kernel-learned-matrix}
\end{equation}
This gives Eq.~\eqref{eq:kernel-learned-matrix-main}.

We next derive the test errors. Define
\begin{equation}
\begin{aligned}
S_{K,\mathrm{pop}}:=\mathbb E_a[\theta(a)\theta(a)^\top]=\bigoplus_{k\in\mathcal K}
c_k^2\frac{D}{D_{k,d}}I_{D_{k,d}},
\end{aligned}
\label{eq:app-kernel-population-task-matrix}
\end{equation}
where orthogonality of spherical harmonics across different degrees was used.
In particular,
\begin{equation}
\frac1D\Tr S_{K,\mathrm{pop}}
=
\nu_{\rm in}.
\end{equation}

For notational convenience, define
\begin{equation}
S_{K,\mathrm M}:=S_K,
\qquad
S_{K,\mathrm G}:=S_{K,\mathrm{pop}},
\end{equation}
We use the test-error parameters in Eq.~\eqref{eq:kernel-errors-main}, together with
\begin{equation}
\overline q_K
:=
\frac1D\mathbb E\Tr(
\widehat W_K\widehat W_K^\top).
\label{eq:app-kernel-qbar-definition}
\end{equation}
By construction,
\begin{equation}
m_{K,\mathrm M}=m_K,
\qquad
q_{K,\mathrm M}=q_K,
\end{equation}
where $m_K$ and $q_K$ are the order parameters of the normalized problem in
Eq.~\eqref{eq:SE-kernel}.

At inference time the original labels are noiseless. Under the same
Gaussian-equivalence approximation introduced in
Section~\ref{app:expansion}, we write the test labels as
\begin{equation}
y
=
\frac1{\sqrt D}\theta(a)^\top h(x)
+\xi,
\qquad
\xi\sim\mathcal N(0,\nu_{\rm out}),
\label{eq:app-kernel-effective-test-label}
\end{equation}
where $\xi$ is independent of $h(x)$. Thus
\begin{equation}
\mathbb E[y^2\mid a]=1,
\qquad
\mathbb E[yh(x)\mid a]
=
\frac{\theta(a)}{\sqrt D}.
\label{eq:app-kernel-test-label-moments}
\end{equation}

Let
\begin{equation}
H_{\rm te}^K
=
\frac1{\widetilde L}
h(x_0)
\left(
\sum_{\ell=1}^{\widetilde L}
y_\ell h(x_\ell)
\right)^\top,
\qquad
\widehat y_0
=
\Tr\!\left(
\widehat W_K(H_{\rm te}^K)^\top
\right).
\label{eq:app-kernel-test-statistic}
\end{equation}
For a fixed test task $a$,
\begin{equation}
\mathbb E[
y_0(H_{\rm te}^K)_{ij}
\mid a]
=
\frac{\theta_i(a)\theta_j(a)}{D}.
\label{eq:app-kernel-test-cross-moment}
\end{equation}
Therefore,
\begin{equation}
\mathbb E[y_0\widehat y_0]
=
\nu_{\rm in}m_{K,\mathcal P},
\qquad
\mathcal P\in\{\mathrm M,\mathrm G\}.
\label{eq:app-kernel-test-cross-term}
\end{equation}

For the second moment, define
\begin{equation}
B_{\rm te}(a)
:=
\mathbb E_x[
y^2h(x)h(x)^\top\mid a].
\end{equation}
A direct calculation gives
\begin{equation}
\mathbb E[
(H_{\rm te}^K)_{ij}
(H_{\rm te}^K)_{i'j'}
\mid a]
=
\delta_{ii'}
\left[
\frac{\widetilde L-1}{\widetilde L D}
\theta_j(a)\theta_{j'}(a)
+
\frac1{\widetilde L}
(B_{\rm te}(a))_{jj'}
\right].
\label{eq:app-kernel-test-prompt-second-moment}
\end{equation}

As in Appendix \ref{app:kernel-reduction}, we have $\mathbb E_a[B_{\rm te}(a)]=I_D$.
For memorization, we analogously replace the empirical task average by the population quantity. Hence,
using Eq.~\eqref{eq:app-kernel-test-prompt-second-moment}, we obtain
\begin{equation}
\begin{aligned}
\mathbb E[\widehat y_0^2]
&=
\frac{\widetilde L-1}{\widetilde L}
\nu_{\rm in}q_{K,\mathcal P}
+
\frac1{\widetilde L}
\mathbb E\Tr(
\widehat W_K\widehat W_K^\top)
+o(1)\\
&=
\nu_{\rm in}q_{K,\mathcal P}
+
\frac{\overline q_K}{\widetilde\alpha_K}
+o(1),
\end{aligned}
\label{eq:app-kernel-test-prediction-second-moment}
\end{equation}
where we used $\widetilde L=\widetilde\alpha_KD$ in the second equality.

Finally, since the test labels are noiseless and normalized so that
$\mathbb E[y_0^2]=1=\nu_{\rm in}+\nu_{\rm out}$,
\begin{equation}
\begin{aligned}
\mathcal E_{\mathcal P}^{\mathrm{kernel}}
&=
\mathbb E[
(y_0-\widehat y_0)^2]\\
&=
1
-2\nu_{\rm in}m_{K,\mathcal P}
+\nu_{\rm in}q_{K,\mathcal P}
+\frac{\overline q_K}{\widetilde\alpha_K}
+o(1)\\
&=
\boxed{
\nu_{\rm out}
+
\nu_{\rm in}
\left(
1-2m_{K,\mathcal P}+q_{K,\mathcal P}
\right)
+
\frac{\overline q_K}{\widetilde\alpha_K}
}
+o(1),
\qquad
\mathcal P\in\{\mathrm M,\mathrm G\}.
\end{aligned}
\label{eq:app-kernel-test-errors}
\end{equation}
Dropping the asymptotically vanishing $o(1)$ term gives
Eq.~\eqref{eq:kernel-errors-main}.

\paragraph{Large-pretraining limit.}
Now we study the regime in which \(\alpha_K=L/D\), \(\widetilde\alpha_K=\widetilde L/D\), and
\(\kappa_K=N_0/D\) remain fixed, while
\(\gamma_K=N/(DN_0)\to\infty\).
Let
\begin{equation}
b_K:=\frac{1+\Delta}{\alpha_K},
\qquad
\delta_K:=\frac{\lambda}{\widehat\chi_K},
\qquad
A_{K,\gamma}:=S_K+(b_K+\delta_K)I_D.
\label{eq:app-kernel-large-N-definitions}
\end{equation}
Since \(\widehat m_K=\widehat\chi_K\), the first two conjugate equations in
Result~\ref{res:app-kernel-state-equations} give
\begin{equation}
\frac{\nu_{\rm in}\widehat q_K}{N_0\widehat\chi_K^2}=\frac{D}{N}\left(1+\Delta-2\nu_{\rm in}m_K+\nu_{\rm in}q_K+\frac{1+\Delta}{\alpha_K}\overline q_K\right).
\label{eq:app-kernel-finite-N-noise-coefficient}
\end{equation}
Define
\begin{equation}
\Xi_K:=1+\Delta-2\nu_{\rm in}m_K+\nu_{\rm in}q_K+\frac{1+\Delta}{\alpha_K}\overline q_K.
\label{eq:app-kernel-Xi}
\end{equation}
Then Eq.~\eqref{eq:app-kernel-learned-matrix} can be written as
\begin{equation}
\widehat W_K\overset{\mathrm d}=S_KA_{K,\gamma}^{-1}+\sqrt{\frac{D}{N}\Xi_K}\,Z_K(S_K+b_KI_D)^{1/2}A_{K,\gamma}^{-1}.
\label{eq:app-kernel-finite-N-matrix-decomposition}
\end{equation}
We first control the ridge correction \(\delta_K\).
Using the equations for \(\chi_K\) and \(\overline\chi_K\) in
Eq.~\eqref{eq:SE-kernel}, the equation for \(\widehat\chi_K\) becomes
\begin{equation}
\widehat\chi_K+\frac1{\kappa_KD}
\Tr\!\left[(S_K+b_KI_D)A_{K,\gamma}^{-1}\right]=\gamma_K.
\label{eq:app-kernel-hatchi-large-gamma}
\end{equation}
For \(\lambda\ge0\),
\begin{equation}
0\le(S_K+b_KI_D)A_{K,\gamma}^{-1}\preceq I_D,
\end{equation}
and hence
\begin{equation}
\gamma_K-\frac1{\kappa_K}\le\widehat\chi_K\le\gamma_K.
\end{equation}
Therefore
\begin{equation}
\widehat\chi_K=\gamma_K+O(1),\qquad\delta_K=O(\gamma_K^{-1})=O\!\left(\frac{D^2}{N}\right),
\label{eq:app-kernel-ridge-correction-scale}
\end{equation}
where the last relation uses
\(N=\gamma_K\kappa_KD^2\) with fixed \(\kappa_K\).

It follows that the deterministic part of
Eq.~\eqref{eq:app-kernel-finite-N-matrix-decomposition} converges to
\begin{equation}
W_{K,\infty}:=S_K(S_K+b_KI_D)^{-1}.
\label{eq:app-kernel-large-N-solution}
\end{equation}
Indeed, since \(b_K>0\),
\begin{equation}
\left\|S_KA_{K,\gamma}^{-1}-W_{K,\infty}
\right\|_{\rm op}=O(\delta_K)=O\!\left(\frac{D^2}{N}\right).
\label{eq:app-kernel-deterministic-correction}
\end{equation}
The second term in
Eq.~\eqref{eq:app-kernel-finite-N-matrix-decomposition} is a zero-mean fluctuation conditioned on \(S_K\). Its normalized squared Frobenius norm is
\begin{equation}
\begin{aligned}
\frac1D\mathbb E_{Z_K}\left\|\widehat W_K-S_KA_{K,\gamma}^{-1}\right\|_F^2=\frac{D}{N}\Xi_K\Tr\!\left[(S_K+b_KI_D)A_{K,\gamma}^{-2}\right].
\end{aligned}
\label{eq:app-kernel-fluctuation-norm}
\end{equation}
For fixed \(\alpha_K>0\),$
\frac1D\Tr\!\left[(S_K+b_KI_D)A_{K,\gamma}^{-2}\right]=O(1),
$, and \(\Xi_K=O(1)\) in the large-pretraining regime. Consequently,
\begin{equation}
\frac1D\mathbb E_{Z_K}\left\|\widehat W_K-S_KA_{K,\gamma}^{-1}\right\|_F^2=O\!\left(\frac{D^2}{N}\right).
\label{eq:app-kernel-finite-N-fluctuation}
\end{equation}
We next evaluate the limiting test errors by substituting
\(W_{K,\infty}\) into Eq.~\eqref{eq:app-kernel-test-errors}. Define
\begin{equation}
A_{K,\infty}:=S_K+b_KI_D.
\label{eq:app-kernel-large-N-A}
\end{equation}
For memorization,
\begin{equation}
\begin{aligned}
\mathcal E_{\mathrm M}^{\mathrm{kernel}}(\infty)=\nu_{\rm out}+\nu_{\rm in}-\frac2D\Tr(S_K^2A_{K,\infty}^{-1})
+\frac1D\Tr(S_K^3A_{K,\infty}^{-2})
+\frac1{\widetilde\alpha_KD}\Tr(S_K^2A_{K,\infty}^{-2}).
\end{aligned}
\label{eq:app-kernel-large-N-memorization-expanded}
\end{equation}
Using \(D^{-1}\Tr S_K=\nu_{\rm in}\) and
\(S_K=A_{K,\infty}-b_KI_D\), the first three terms reduce to
\begin{equation}
\nu_{\rm in}-\frac2D\Tr(S_K^2A_{K,\infty}^{-1})
+\frac1D\Tr(S_K^3A_{K,\infty}^{-2})=\frac{b_K^2}{D}
\Tr(S_KA_{K,\infty}^{-2}).
\end{equation}
Therefore
\begin{equation}
\boxed{\mathcal E_{\mathrm M}^{\mathrm{kernel}}(\infty)=\nu_{\rm out}+\frac{(1+\Delta)^2}{\alpha_K^2D}
\Tr(S_KA_{K,\infty}^{-2})+\frac1{\widetilde\alpha_KD}\Tr(S_K^2A_{K,\infty}^{-2}).}
\label{eq:app-kernel-large-N-memorization}
\end{equation}
Similarly, using
\(D^{-1}\Tr S_{K,\mathrm{pop}}=\nu_{\rm in}\), the generalization error is
\begin{equation}
\boxed{\mathcal E_{\mathrm G}^{\mathrm{kernel}}(\infty)=\nu_{\rm out}+\frac{(1+\Delta)^2}{\alpha_K^2D}
\Tr(S_{K,\mathrm{pop}}A_{K,\infty}^{-2})+\frac1{\widetilde\alpha_KD}\Tr(S_K^2A_{K,\infty}^{-2}).}
\label{eq:app-kernel-large-N-generalization}
\end{equation}
The deterministic correction caused by
\(\delta_K=\lambda/\widehat\chi_K\) is
\(O(D^2/N)\), while terms linear in the Gaussian fluctuation vanish after averaging over \(Z_K\) and its quadratic contribution is controlled by
Eq.~\eqref{eq:app-kernel-finite-N-fluctuation}. Hence
\begin{equation}
\boxed{\mathcal E_{\mathcal P}^{\mathrm{kernel}}=\mathcal E_{\mathcal P}^{\mathrm{kernel}}(\infty)+O\!\left(\frac{D^2}{N}\right),\qquad
\mathcal P\in\{\mathrm M,\mathrm G\}.}
\label{eq:app-kernel-large-N-rate}
\end{equation}

\paragraph{Infinite task-pool limit and context-length scalings.}
We next set $\Delta=\lambda=0$ and take first $\gamma_K\to\infty$ and then $\kappa_K\to\infty$, while allowing
$L$ and $\widetilde L$ to be different. In this ordered limit,
\begin{equation}
S_K\longrightarrow S_{K,\mathrm{pop}}=\bigoplus_{k\in\mathcal K}c_k^2\frac{D}{D_{k,d}}I_{D_{k,d}}.
\end{equation}
Consequently, the memorization and generalization errors coincide.
Writing
$S:=S_{K,\mathrm{pop}}$ and $A:=S+\frac1{\alpha_K}I_D$,
Eqs.~\eqref{eq:app-kernel-large-N-memorization} and
\eqref{eq:app-kernel-large-N-generalization} give
\begin{equation}
\begin{aligned}
\mathcal E_{\mathcal P}^{\mathrm{kernel}}
=
\nu_{\rm out}
+
\frac1{\alpha_K^2D}\Tr(SA^{-2})
+
\frac1{\widetilde\alpha_KD}\Tr(S^2A^{-2})
+o(1),
\qquad
\mathcal P\in\{\mathrm M,\mathrm G\}.
\end{aligned}
\label{eq:app-kernel-population-error-mismatch-trace}
\end{equation}
Since the degree-$k$ block of $S$ has eigenvalue
$c_k^2\frac{D}{D_{k,d}}$ with multiplicity $D_{k,d}$, we obtain
\begin{equation}
\boxed{
\mathcal E_{\mathcal P}^{\mathrm{kernel}}
=
\nu_{\rm out}
+
\sum_{k\in\mathcal K}
\left[
c_k^2
\left(
\frac{D_{k,d}}{D_{k,d}+c_k^2L}
\right)^2
+
\frac{D_{k,d}}{\widetilde L}
\left(
\frac{c_k^2L}{D_{k,d}+c_k^2L}
\right)^2
\right]
+o(1).
}
\label{eq:app-kernel-population-error-mismatch}
\end{equation}

Although the replica calculation above was formulated with a fixed positive
$\alpha_K=L/D$, we can examine
$d$-dependent context scalings by substituting $\alpha_K(d)=L(d)/D(d)$ before taking $d\to\infty$.
For every fixed $k$,
\begin{equation}
D_{k,d}
=
\frac{d^k}{k!}\left(1+O(d^{-1})\right).
\label{eq:app-kernel-harmonic-dimension-asymptotic}
\end{equation}
Consider a polynomial context-length scaling
\begin{equation}
L=\widetilde L=\alpha d^s,
\qquad
\alpha>0,\quad s>0.
\label{eq:app-kernel-general-context-scaling}
\end{equation}
For each $k\in\mathcal K$,
\begin{equation}
c_k^2
\frac{D_{k,d}}
{c_k^2L+D_{k,d}}
\longrightarrow
\begin{cases}
0,
& k<s,\\[5pt]
\displaystyle
\frac{c_k^2}{1+\alpha k!c_k^2},
& k=s,\\[9pt]
c_k^2,
& k>s,
\end{cases}
\label{eq:app-kernel-degree-scaling}
\end{equation}
where the middle case occurs only when $s$ is an integer contained in
$\mathcal K$. Therefore,
\begin{equation}
\boxed{
\mathcal E_{\mathcal P}^{\mathrm{kernel}}(\infty)
\longrightarrow
\nu_{\rm out}
+
\sum_{\substack{k\in\mathcal K\\k>s}}c_k^2
+
\mathbf 1_{\{s\in\mathbb N,\;s\in\mathcal K\}}
\frac{c_s^2}{1+\alpha s!c_s^2},
}
\qquad
\mathcal P\in\{\mathrm M,\mathrm G\}.
\label{eq:app-kernel-general-context-limit}
\end{equation}
Hence a degree-$k$ harmonic component is asymptotically fully recovered when
$L\gg d^k$, partially recovered when $L=\Theta(d^k)$, and unrecovered when
$L\ll d^k$.

\newpage
\section{Derivation of the feature learner}
\label{app:feature-details}
\setcounter{definition}{0}
\renewcommand{\theResult}{\thesection.\arabic{Result}}
\setcounter{Result}{0}

\subsection{Fixed point equations}
\label{app:feature-fixed-point}
\begin{Result}[feature learner: state equations]
\label{res:app-feature-state-equations}
Given $(\alpha_F,\gamma_F,\kappa_F,\Delta,\lambda,\sigma_\star,\sigma_0,S_F)$, the state equations are
\begin{equation}
\begin{aligned}
V
&:=\mu_1^2(q-m^2)+\frac{1+\Delta}{\alpha_F}\bar q,
\qquad
d_\chi:=\mu_1^2\chi_Q+\frac{1+\Delta}{\alpha_F}\chi_{\bar Q},
\qquad
u_0:=\mu_1mG+\sqrt V\,Z,
\\[2mm]
v^\star
&:=\operatorname*{arg\,min}_{v\in\R}
\left\{
\frac12\left[\sigma_\star(G)+\sqrt\Delta\,\varepsilon-\sigma_0(v)\right]^2
+\frac{(v-u_0)^2}{2d_\chi}
\right\},
\\[2mm]
\widehat m
&=\gamma_F\E\!\left[
\frac{v^\star-u_0}{d_\chi}
\left(
\mu_1G-\frac{\mu_1^2m}{\sqrt V}Z
\right)
\right],
\\
\widehat\chi_Q
&=
-\gamma_F\mu_1^2
\E\!\left[
\frac{(v^\star-u_0)Z}
{d_\chi\sqrt V}
\right],
\\
\widehat\chi_{\bar Q}
&=
-\gamma_F\frac{1+\Delta}{\alpha_F}
\E\!\left[
\frac{(v^\star-u_0)Z}
{d_\chi\sqrt V}
\right],
\\
\widehat q_Q
&=
\gamma_F\mu_1^2
\E\!\left[
\frac{(v^\star-u_0)^2}
{d_\chi^2}
\right],
\\
\widehat q_{\bar Q}
&=
\gamma_F\frac{1+\Delta}{\alpha_F}
\E\!\left[
\frac{(v^\star-u_0)^2}
{d_\chi^2}
\right],
\\[2mm]
m
&=\widehat m\int
\frac{s^2}{\lambda+\widehat\chi_{\bar Q}+\widehat\chi_Qs}
\,\nu_{\kappa_F}(\dd s),
\\
q
&=\int
\frac{
\widehat m^2s^3+\kappa_F^{-1}s(\widehat q_Qs+\widehat q_{\bar Q})
}{
(\lambda+\widehat\chi_{\bar Q}+\widehat\chi_Qs)^2
}
\,\nu_{\kappa_F}(\dd s),
\\
\bar q
&=\int
\frac{
\widehat m^2s^2+\kappa_F^{-1}(\widehat q_Qs+\widehat q_{\bar Q})
}{
(\lambda+\widehat\chi_{\bar Q}+\widehat\chi_Qs)^2
}
\,\nu_{\kappa_F}(\dd s),
\\
\chi_Q
&=\frac1{\kappa_F}\int
\frac{s}{\lambda+\widehat\chi_{\bar Q}+\widehat\chi_Qs}
\,\nu_{\kappa_F}(\dd s),
\\
\chi_{\bar Q}
&=\frac1{\kappa_F}\int
\frac{1}{\lambda+\widehat\chi_{\bar Q}+\widehat\chi_Qs}
\,\nu_{\kappa_F}(\dd s).
\end{aligned}
\label{eq:app-feature-state-equations}
\end{equation}
where $G,Z,\varepsilon\overset{\mathrm{iid}}{\sim}\Normal(0,1)$ and $\nu_{\kappa}$ denotes the spectrum in \eqref{eq:MP-law}.
The fixed point gives the five order parameters
$(m,q,\bar q,\chi_Q,\chi_{\bar Q})$
and their conjugates
$(\widehat m,\widehat q_Q,\widehat q_{\bar Q},
\widehat\chi_Q,\widehat\chi_{\bar Q})$.

\end{Result}

\subsection{Derivation of the state equations}
\label{app:feature-state-equation-derivation}

We condition on the task pool throughout this calculation, and evaluate the resulting normalized traces using its limiting spectrum at the end. For notational simplicity, we will drop the subscript F.

\subsubsection{The replica trick}

Write $\ell(y,v):=\frac12[y-\sigma_0(v)]^2$. The Stage-I objective in Eq.~\eqref{eq:feature-first-stage} is, 
\begin{equation}
\mathcal L_F(W)
=
\frac12\sum_{t=1}^{N}
\left[
\sigma_\star\!\left(
\frac{(a^t)^\top x_0^t}{\sqrt d}
\right)
+\epsilon_0^t
-
\sigma_0\!\left(
\Tr(W(H^t)^\top)
\right)
\right]^2
+
\frac{N_0\lambda}{2}\|W\|_F^2.
\label{eq:app-feature-loss}
\end{equation}

The dataset $\mathcal D$ and the prompt task indices $\tau^t$ are defined in Section~\ref{sec:settings}. Below, $\E_{\mathcal D\mid\mathcal A_0}$ averages over the task indices, inputs, and label noise, conditional on the task pool.
We represent loss minimization by the integral~\citep{edwards1975theory,loureiro2024four}
\begin{equation}
Z_\beta(\mathcal D):=\int\dd W\,e^{-\beta\mathcal L_F(W)},
\qquad
-\lim_{\beta\to\infty}\frac{\log Z_\beta(\mathcal D)}{\beta dN_0}
=\frac{\min_W\mathcal L_F(W)}{dN_0}.
\label{eq:app-feature-loss-integral}
\end{equation}
Here $\beta>0$ is the inverse temperature. 
To evaluate the data average of the logarithm, we use its relation to moments near $n=0$:
\begin{equation}
\begin{aligned}
\log\E[Z_\beta^n\mid\mathcal A_0]
&=\log\E_{\mathcal D\mid\mathcal A_0}e^{n\log Z_\beta}
=n\,\E_{\mathcal D\mid\mathcal A_0}\log Z_\beta+O(n^2),\\
\E_{\mathcal D\mid\mathcal A_0}\log Z_\beta
&=\lim_{n\to0}\frac1n\log\E[Z_\beta^n\mid\mathcal A_0].
\end{aligned}
\label{eq:app-feature-replica-identity}
\end{equation}
We compute the moments for positive integers $n$ and continue the expression to $n$ near zero. The calculation first takes $d\to\infty$ at fixed $n,\beta$ and fixed $(\alpha,\gamma,\kappa)$, then extracts the term linear in $n$, and finally takes $\beta\to\infty$. This continuation and the exchanges of limits needed to relate it to Eq.~\eqref{eq:app-feature-loss-integral} are formal steps in the replica method~\citep{mezard1987spin,loureiro2024four}. 

For integer $n$, $Z_\beta^n$ introduces $n$ copies $W^1,\ldots,W^n$ of the attention matrix. Conditional on $\mathcal A_0$, the prompts are independent, so
\begin{equation}
\begin{aligned}
\E[Z_\beta^n\mid\mathcal A_0]
&=\int\prod_{r=1}^n\dd W^r\exp\!\left[-\frac{\beta N_0\lambda}{2}\sum_r\Tr(W^r(W^r)^\top)\right]\left[\frac1{N_0}\sum_{k=1}^{N_0}\phi_{\beta,k}(W^1,\ldots,W^n)\right]^N,\\
\phi_{\beta,k}
&:=\E_{x_{0:L},\epsilon_{0:L}\mid a^k}\exp\!\left[-\beta\sum_{r=1}^n\ell\!\left(\sigma_\star((a^k)^\top x_0/\sqrt d)+\epsilon_0,\Tr(W^rH^\top)\right)\right].
\end{aligned}
\label{eq:app-feature-replicated-partition}
\end{equation}
The factor $N_0^{-1}\sum_k\phi_{\beta,k}$ averages one prompt over its task index; the power $N$ accounts for all training prompts.

\subsubsection{Teacher score and attention outputs}

For a prompt generated by task $a^\mu$, we first compute the moments of the teacher score $g:=(a^\mu)^\top x_0^\mu/\sqrt d$ and the attention output $u^r:=\Tr(W^r(H^\mu)^\top)$,$r=1,\ldots,n$. We use $\alpha=L/d$, $\widetilde\alpha=\widetilde L/d$, $\gamma=N/(dN_0)$, and $\kappa=N_0/d$, as in Result~\ref{res:app-feature-state-equations}.

Eq.~\eqref{eq:feature-attention-output} gives
\begin{equation}
\begin{aligned}
H^\mu
&:=
\frac{1}{L}
x_0^\mu
\left(
\sum_{\ell=1}^{L}y_\ell^\mu x_\ell^\mu
\right)^\top ,\\
u^r
&=
\Tr\!\left(W^r(H^\mu)^\top\right)
=
\frac{1}{L}\sum_{\ell=1}^{L}
y_\ell^\mu (x_0^\mu)^\top W^r x_\ell^\mu .
\end{aligned}
\label{eq:app-feature-H}
\end{equation}

We first compute the first- and second-order moments of these local fields.
\begin{proposition}
\label{lem:app-feature-local-moments}
For fixed matrices $W^r,W^s\in\R^{d\times d}$ with bounded operator norm and $d^{-1}\Tr(W^r(W^r)^\top)=O(1)$,
\begin{equation}
\begin{aligned}
\E[g\mid a^\mu]
&=0,\\
\E[u^r\mid a^\mu]
&=0,
\\[1mm]
\E[g^2\mid a^\mu]
&=1,
\\[1mm]
\E[g\,u^r\mid a^\mu]
&=
\frac{\mu_1}{d}(a^\mu)^\top W^r a^\mu+o(1),
\\[1mm]
\E[u^r u^s\mid a^\mu]
&=
\frac{\mu_1^2}{d}
(W^r a^\mu)^\top(W^s a^\mu)
+
\frac{1+\Delta}{\alpha d}
\Tr\!\left(W^r(W^s)^\top\right)
+o(1).
\end{aligned}
\label{eq:app-feature-local-moments}
\end{equation}
\end{proposition}

\begin{proof}

Using $\mu_1=\E[\sigma_\star(G)G]$, we decompose
$
x=g_d\,\frac{a^\mu}{\sqrt d}+x_\perp .
$
Then
\begin{equation}
\begin{aligned}
\E[yx\mid a^\mu]
&=
\E\!\left[
y\left(
g_d\frac{a^\mu}{\sqrt d}+x_\perp
\right)
\middle|a^\mu
\right]
=
\E[yg_d\mid a^\mu]\frac{a^\mu}{\sqrt d}
+
\E\!\left[
y\,\E[x_\perp\mid g_d,a^\mu]
\middle|a^\mu
\right]
=
\mu_{1}\frac{a^\mu}{\sqrt d}+o(1),\\
\E[y^2xx^\top\mid a^\mu]
&=
(1+\Delta)I_d
+
O(1)\frac{a^\mu(a^\mu)^\top}{d}
+
o(1).
\end{aligned}
\label{eq:app-feature-single-token-moments}
\end{equation}

The mixed moment of $(g,u^r)$ is therefore
\begin{align}
\E[g\,u^r\mid a^\mu]
&=
\frac{1}{L}
\sum_{\ell=1}^{L}
\E\!\left[
g\,
y_\ell^\mu
(x_0^\mu)^\top W^r x_\ell^\mu
\middle|a^\mu
\right]
\nonumber\\
&=
\frac{1}{L}
\sum_{\ell=1}^{L}
\frac{1}{\sqrt d}
(a^\mu)^\top
W^r
\E[y_\ell^\mu x_\ell^\mu\mid a^\mu]
\nonumber\\
&=
\frac{\mu_1}{d}
(a^\mu)^\top W^r a^\mu
+o(1).
\label{eq:app-feature-mixed-moment}
\end{align}

For the second-order student statistic $u^r$,
\begin{align}
&\E\!\left[u^\mu(W^r)u^\mu(W^s)\middle|a^\mu\right]
\\&=
\frac{1}{L^2}\sum_{\ell\neq\ell'}
\E\!\left[y_\ell^\mu y_{\ell'}^\mu(x_0^\mu)^\top W^r x_\ell^\mu(x_0^\mu)^\top W^s x_{\ell'}^\mu\middle|a^\mu\right]+\frac{1}{L^2}
\sum_{\ell=1}^{L}
\E\!\left[
(y_\ell^\mu)^2
(x_0^\mu)^\top W^r x_\ell^\mu
(x_0^\mu)^\top W^s x_\ell^\mu
\middle|a^\mu
\right]
\nonumber\\
&=
\frac{L-1}{L}
\frac{\mu_1^2}{d}
(W^r a^\mu)^\top(W^s a^\mu)
+
\frac{1}{L}
\Tr\!\left[
W^r
\E[y^2xx^\top\mid a^\mu]
(W^s)^\top
\right]
\nonumber\\
&=
\frac{\mu_1^2}{d}
(W^r a^\mu)^\top(W^s a^\mu)
+
\frac{1+\Delta}{\alpha d}
\Tr\!\left(W^r(W^s)^\top\right)
+o(1).
\label{eq:app-feature-second-moment}
\end{align}

These calculations prove Eq.~\eqref{eq:app-feature-local-moments}.
\end{proof}

These definitions are imposed inside $\phi_{\beta,k}$ by the Fourier representations of Dirac delta functions
\begin{equation}
1=\int\dd g\,\delta(g-(a^k)^\top x_0/\sqrt d)=\int\frac{\dd g\,\dd\widehat g}{2\pi}e^{i\widehat g(g-(a^k)^\top x_0/\sqrt d)},
\label{eq:app-feature-local-delta-g}
\end{equation}
\begin{equation}
1=\int\dd u^r\,\delta(u^r-\Tr(W^rH^\top))=\int\frac{\dd u^r\,\dd\widehat u^r}{2\pi}e^{i\widehat u^r(u^r-\Tr(W^rH^\top))},\qquad r=1,\ldots,n.
\label{eq:app-feature-local-deltas}
\end{equation}
The variables $\widehat g,\widehat u^1,\ldots,\widehat u^n$ are temporary Fourier variables and are unrelated to the conjugate order parameters introduced later.  If $p_\Delta$ denotes the density of $\Normal(0,\Delta)$, substituting Eqs.~\eqref{eq:app-feature-local-delta-g} and \eqref{eq:app-feature-local-deltas} gives
\begin{equation}
\begin{aligned}
\phi_{\beta,k}
&=\int\dd\epsilon_0p_\Delta(\epsilon_0)\int\frac{\dd g\,\dd\widehat g}{2\pi}\prod_{r=1}^n\frac{\dd u^r\,\dd\widehat u^r}{2\pi}
e^{-\beta\sum_r\ell(\sigma_\star(g)+\epsilon_0,u^r)}e^{i\widehat gg+i\sum_r\widehat u^ru^r}\\
&\quad\times\E_{x_{0:L},\epsilon_{1:L}\mid a^k}\exp\!\left[-i\widehat g(a^k)^\top x_0/\sqrt d-i\sum_r\widehat u^r\Tr(W^rH^\top)\right].
\end{aligned}
\label{eq:app-feature-single-prompt-fourier}
\end{equation}

\begin{assumption}[Gaussian equivalence of the local field]
\label{ass:app-feature-local-gaussian}
For fixed $n$ and $L/d\to\alpha\in(0,\infty)$, condition on $a^k,W^1,\ldots,W^n$, independently of the current prompt. We assume that the teacher score and the attention outputs can be replaced, in an averaged prompt, by jointly Gaussian variables with the same leading first two moments:
\begin{equation}
(g,u^1,\ldots,u^n)\Longrightarrow\Normal(0,C_k),
\label{eq:app-feature-local-gaussianity}
\end{equation}
where
\begin{equation}
\begin{aligned}
(C_k)_{00}&=1,\\
(C_k)_{0r}&=\frac{\mu_1}{d}(a^k)^\top W^ra^k+o(1),\\
(C_k)_{rs}&=\frac{\mu_1^2}{d}(W^ra^k)^\top(W^sa^k)+\frac{1+\Delta}{\alpha d}\Tr(W^r(W^s)^\top)+o(1).
\end{aligned}
\label{eq:app-feature-task-local-covariance}
\end{equation}
Here we assume the first two moments are matched, and higher-order correlations are negligible in the high-dimensional limit. 
\end{assumption}

Let $\bm z:=(g,u^1,\ldots,u^n)^\top$ and $\widehat{\bm z}:=(\widehat g,\widehat u^1,\ldots,\widehat u^n)^\top$.  Assumption~\ref{ass:app-feature-local-gaussian} gives
\begin{equation}
\E_{x_{0:L},\epsilon_{1:L}\mid a^k}e^{-i\widehat{\bm z}^\top\bm z}=\exp\!\left(-\frac12\widehat{\bm z}^\top C_k\widehat{\bm z}\right)+o(1).
\label{eq:app-feature-characteristic-function}
\end{equation}
Substituting this function into Eq.~\eqref{eq:app-feature-single-prompt-fourier} and integrating the Fourier variables gives
\begin{equation}
\begin{aligned}
\phi_{\beta,k}
&=\int\dd\epsilon_0p_\Delta(\epsilon_0)\int\dd g\prod_r\dd u^r\int\frac{\dd^{n+1}\widehat{\bm z}}{(2\pi)^{n+1}}
e^{-\beta\sum_r\ell(\sigma_\star(g)+\epsilon_0,u^r)}e^{i\widehat{\bm z}^\top\bm z-\widehat{\bm z}^\top C_k\widehat{\bm z}/2}+o(1)\\
&=\int\dd\epsilon_0p_\Delta(\epsilon_0)\int\dd g\prod_r\dd u^r\,\Normal((g,u^1,\ldots,u^n)\mid0,C_k)e^{-\beta\sum_r\ell(\sigma_\star(g)+\epsilon_0,u^r)}+o(1).
\end{aligned}
\label{eq:app-feature-fourier-to-gaussian}
\end{equation}
The Fourier integrals thus replace the average over the prompt by a Gaussian integral over $(g,u^1,\ldots,u^n)$.

\subsubsection{Order parameters and replica symmetry}

As in \cite{takanami2025learning}, we assume self-averaging with respect to the tasks in the pool. The task-dependent quadratic forms can then be replaced by their averages over the task pool. We therefore introduce
\begin{equation}
\begin{aligned}
m^r
&:=
\frac{1}{dN_0}
\sum_{k=1}^{N_0}
(a^k)^\top W^r a^k
=
\frac1d\Tr(W^rS),
\\
Q^{rs}
&:=
\frac{1}{dN_0}
\sum_{k=1}^{N_0}
(W^ra^k)^\top(W^sa^k)
=
\frac1d\Tr(W^rS(W^s)^\top),
\\
\bar Q^{rs}
&:=
\frac1d\Tr(W^r(W^s)^\top).
\end{aligned}
\label{eq:app-feature-order-parameters}
\end{equation}

Equations~\eqref{eq:app-feature-task-local-covariance} give
\begin{equation}
\begin{aligned}
C_{00}&=1,\\
C_{0r}&=\mu_1m^r,\\
C_{rs}&=\mu_1^2Q^{rs}+\frac{1+\Delta}{\alpha}\bar Q^{rs}.
\end{aligned}
\label{eq:app-feature-local-covariance}
\end{equation}

Hence
\begin{equation}
\phi_\beta(m,Q,\bar Q)=\int\dd\epsilon_0p_\Delta(\epsilon_0)\int\dd g\prod_r\dd u^r\,\Normal((g,u^1,\ldots,u^n)\mid0,C)e^{-\beta\sum_r\ell(\sigma_\star(g)+\epsilon_0,u^r)},
\label{eq:app-feature-common-prompt-factor}
\end{equation}
and
\begin{equation}
\E[Z_\beta^n\mid\mathcal A_0]\simeq\int\prod_r\dd W^r\exp\!\left[-\frac{\beta N_0\lambda}{2}\sum_r\Tr(W^r(W^r)^\top)\right][\phi_\beta(m,Q,\bar Q)]^N.
\label{eq:app-feature-replicated-after-prompt}
\end{equation}

We insert Eq.~\eqref{eq:app-feature-order-parameters} into the integral over the matrices using Dirac delta functions. Their Fourier representations introduce the conjugate parameters $\widehat m^r,\widehat Q^{rs},\widehat{\bar Q}^{rs}$:
\begin{equation}
\begin{aligned}
1&=\int\dd m^r\,\delta\!\left(m^r-\frac1d\Tr(W^rS)\right)\propto\int\dd m^r\dd\widehat m^r\exp\!\left[-dN_0\widehat m^rm^r+N_0\widehat m^r\Tr(W^rS)\right],\\[1mm]
1&=\int\dd Q^{rs}\,\delta\!\left(Q^{rs}-\frac1d\Tr(W^rS(W^s)^\top)\right)\propto\int\dd Q^{rs}\dd\widehat Q^{rs}\exp\!\left[-\frac{dN_0}{2}\widehat Q^{rs}Q^{rs}+\frac{N_0}{2}\widehat Q^{rs}\Tr(W^rS(W^s)^\top)\right],\\[1mm]
1&=\int\dd\bar Q^{rs}\,\delta\!\left(\bar Q^{rs}-\frac1d\Tr(W^r(W^s)^\top)\right)\propto\int\dd\bar Q^{rs}\dd\widehat{\bar Q}^{rs}\exp\!\left[-\frac{dN_0}{2}\widehat{\bar Q}^{rs}\bar Q^{rs}+\frac{N_0}{2}\widehat{\bar Q}^{rs}\Tr(W^r(W^s)^\top)\right].
\end{aligned}
\label{eq:app-feature-overlap-deltas}
\end{equation}
The conjugate integrations are along the Fourier contours, with the factors of $i$ absorbed into the conjugate variables. Substituting these identities separates the dependence on the matrices from that on the order parameters:
\begin{equation}
\E[Z_\beta^n\mid\mathcal A_0]
\simeq\int\dd\Omega\,e^{-n\beta dN_0 f_{n,\beta}(\Omega)},
\label{eq:app-feature-saddle}
\end{equation}
where $\Omega=\{m^r,Q^{rs},\bar Q^{rs},\widehat m^r,\widehat Q^{rs},\widehat{\bar Q}^{rs}\}$ and
\begin{equation}
\begin{aligned}
f_{n,\beta}(\Omega)
&=-\frac{\gamma}{n\beta}\log\phi_\beta
+\frac1{n\beta}\sum_r\widehat m^rm^r
+\frac1{2n\beta}\sum_{r,s}
(\widehat Q^{rs}Q^{rs}+\widehat{\bar Q}^{rs}\bar Q^{rs})
-\frac1{n\beta dN_0}\log\int\prod_r\dd W^r\,e^{G_n},\\
G_n
&=\frac{N_0}{2}\sum_{r,s}\widehat Q^{rs}\Tr(W^rS(W^s)^\top)+\frac{N_0}{2}\sum_{r,s}\widehat{\bar Q}^{rs}\Tr(W^r(W^s)^\top)\\
&\quad+N_0\sum_r\widehat m^r\Tr(W^rS)-\frac{\beta N_0\lambda}{2}\sum_r\Tr(W^r(W^r)^\top).
\end{aligned}
\label{eq:app-feature-replica-action}
\end{equation}
In the high-dimensional limit, the integral in Eq.~\eqref{eq:app-feature-saddle} is evaluated at a stationary point of $f_{n,\beta}$. We next substitute the replica-symmetric form and take the limits described above.

\begin{assumption}[Replica symmetry]
\label{ass:app-feature-replica-symmetry}
We assume that the stationary solution is invariant under permutations of the replica indices. In the zero-temperature limit, write
\begin{equation}
\begin{aligned}
m^r&=m,\\
Q^{rs}&=q+\frac{\chi_Q}{\beta}\delta_{rs},\\
\bar Q^{rs}&=\bar q+\frac{\chi_{\bar Q}}{\beta}\delta_{rs}.
\end{aligned}
\label{eq:app-feature-rs-ansatz}
\end{equation}
Here $q,\bar q$ are the off-diagonal values, and $\chi_Q,\chi_{\bar Q}$ give the leading difference between diagonal and off-diagonal values. We use the same symmetry for the conjugate parameters~\cite{takanami2025learning,loureiro2024four}.
\end{assumption}

The $1/\beta$ differences retain the variance around a stable minimum as the loss-weighted distribution concentrates. This scaling holds for a strictly convex quadratic objective, as in the identity-readout case with $\lambda>0$; for a general $\sigma_0$, finite limits of these scaled differences are part of the ansatz~\cite{loureiro2024four}.
The corresponding scaling of the conjugate parameters is
\begin{equation}
\begin{aligned}
\widehat Q^{rs}&=\beta^2\widehat q_Q-\beta\widehat\chi_Q\delta_{rs},\\
\widehat{\bar Q}^{rs}&=\beta^2\widehat q_{\bar Q}-\beta\widehat\chi_{\bar Q}\delta_{rs},\\
\widehat m^r&=\beta\widehat m.
\end{aligned}
\label{eq:app-feature-conjugate-ansatz}
\end{equation}

After substituting Eqs.~\eqref{eq:app-feature-rs-ansatz} and \eqref{eq:app-feature-conjugate-ansatz}, the remaining variables are
\begin{equation}
\boldsymbol x=(m,q,\bar q,\chi_Q,\chi_{\bar Q},
\widehat m,\widehat q_Q,\widehat q_{\bar Q},
\widehat\chi_Q,\widehat\chi_{\bar Q}).
\label{eq:scalar-variables-loss}
\end{equation}
Grouping the terms in Eq.~\eqref{eq:app-feature-replica-action} gives
\begin{equation}
f(\boldsymbol x)
=\lim_{\beta\to\infty}\lim_{n\to0}
\left.f_{n,\beta}(\Omega)\right|_{\mathrm{RS}}
=f_1(\boldsymbol x)+f_2(\boldsymbol x)+f_3(\boldsymbol x),
\label{eq:app-feature-scalar-limit}
\end{equation}
up to additive constants independent of $\boldsymbol x$. With $u_0,V,d_\chi$ as defined in Eqs.~\eqref{eq:app-feature-state-equations} and $A_F:=(\lambda+\widehat\chi_{\bar Q})I_d+\widehat\chi_QS_F,\ T_F:=\widehat q_QS_F+\widehat q_{\bar Q}I_d.$, we would obtain
\begin{equation}
\begin{aligned}
f_1(\boldsymbol x)
&=\gamma\E_{G,Z,\varepsilon}\min_{v\in\R}
\left\{\ell\!\left(\sigma_\star(G)+\sqrt\Delta\,\varepsilon,v\right)
+\frac{(v-u_0)^2}{2d_\chi}\right\},\\
f_2(\boldsymbol x)
&=\widehat m m
-\frac12\widehat\chi_Qq-\frac12\widehat\chi_{\bar Q}\bar q
+\frac12\widehat q_Q\chi_Q+\frac12\widehat q_{\bar Q}\chi_{\bar Q},\\
f_3(\boldsymbol x)
&=-\frac1{2d}\left[
\widehat m^2\Tr(S^2A_F^{-1})
+\frac1\kappa\Tr(T_FA_F^{-1})\right].
\end{aligned}
\label{eq:app-feature-rs-potential}
\end{equation}
Here $G,Z,\varepsilon$ are independent standard Gaussian variables. The three functions correspond, respectively, to the term containing $\log\phi_\beta$, the sums over the order parameters and their conjugates, and the logarithm of the matrix integral in Eq.~\eqref{eq:app-feature-replica-action}. The conjugate parameters act as auxiliary multipliers enforcing the definitions of the order parameters.

The state equations follow from
\begin{equation}
\frac{\partial f}{\partial x_i}=0,\qquad i=1,\ldots,10.
\label{eq:app-feature-scalar-stationarity}
\end{equation}
We seek a stationary point of $f$. The function determines the training statistics. Below we calculate $f_1,f_2,f_3$ in turn and then take these derivatives.

\subsubsection{Calculation of \texorpdfstring{$f_1$}{f1}}

To calculate $f_1$, we evaluate the factor $\phi_\beta$ obtained above. 
The teacher score $g$ and the replicated attention outputs $(u^1,\ldots,u^n)$ can be replaced by jointly Gaussian variables with the covariance in Eq.~\eqref{eq:app-feature-local-covariance}. 
After imposing the replica-symmetric form in Eq.~\eqref{eq:app-feature-rs-ansatz}, this Gaussian vector can be represented as
\begin{equation}
g=G,\qquad
u^r
=
\mu_1mG+\sqrt V\,Z+\sqrt{\frac{d_\chi}{\beta}}\,\eta_r,
\qquad
G,Z,\eta_1,\ldots,\eta_n
\overset{\rm iid}{\sim}\Normal(0,1),
\label{eq:app-feature-rs-fields}
\end{equation}
where
$V=\mu_1^2(q-m^2)+(1+\Delta)\bar q/\alpha$
and
$d_\chi=\mu_1^2\chi_Q+(1+\Delta)\chi_{\bar Q}/\alpha$,
as defined in Result~\ref{res:app-feature-state-equations}.

The first two terms in $u^r$ are shared by all replicas in Eq.~\eqref{eq:app-feature-local-covariance}. 
The last term is replica-dependent.
Indeed,
$\E[Gu^r]=\mu_1m$ and
$\E[u^ru^s]=\mu_1^2q+(1+\Delta)\bar q/\alpha+(d_\chi/\beta)\delta_{rs}$,
which agrees with the covariance obtained from the order parameters.

Using this Gaussian representation and the definition of $u_0$ in Eq.~\eqref{eq:app-feature-state-equations}, we have
\begin{equation}
\phi_\beta
=\E_{G,Z,\varepsilon}\!\left\{
\left[
\int\frac{e^{-\eta^2/2}\dd\eta}{\sqrt{2\pi}}\,
e^{-\beta\ell(\sigma_\star(G)+\sqrt\Delta\,\varepsilon,\,
u_0+\sqrt{d_\chi/\beta}\,\eta)}
\right]^n\right\}.
\label{eq:app-feature-prompt-power}
\end{equation}
Extracting the coefficient linear in $n$ gives
\begin{equation}
\lim_{n\to0}\frac1n\log\phi_\beta
=
\E_{G,Z,\varepsilon}
\log
\int
\frac{e^{-\eta^2/2}\dd\eta}{\sqrt{2\pi}}\,
\exp\!\left[
-\frac{\beta}{2}
\left(
\sigma_\star(G)+\sqrt\Delta\,\varepsilon
-\sigma_0\!\left(u_0+\sqrt{\frac{d_\chi}{\beta}}\eta\right)
\right)^2
\right].
\label{eq:app-feature-one-prompt-gaussian}
\end{equation}
Thus the original average over all inputs and demonstrations in a prompt has been reduced to an average over a few scalar Gaussian variables.

To take $\beta\to\infty$, set $v=u_0+\sqrt{d_\chi/\beta}\,\eta$. Writing $y=\sigma_\star(G)+\sqrt\Delta\,\varepsilon$ within this calculation,
\begin{equation}
\int\frac{\dd\eta}{\sqrt{2\pi}}\,
e^{-\eta^2/2-\beta\ell(y,u_0+\sqrt{d_\chi/\beta}\,\eta)}
=
\sqrt{\frac{\beta}{2\pi d_\chi}}\int\dd v\,
e^{-\beta\{\ell(y,v)+(v-u_0)^2/(2d_\chi)\}}.
\label{eq:app-feature-prompt-change-variable}
\end{equation}
The quadratic term comes from the Gaussian density. Taking the negative logarithm and dividing by $\beta$ makes the logarithm of the prefactor vanish, while the integral gives the minimum of the scalar objective. Hence
\begin{equation}
f_1(\boldsymbol x)=-\lim_{\beta\to\infty}\lim_{n\to0}
\frac{\gamma}{n\beta}\log\phi_\beta
=
\gamma\E_{G,Z,\varepsilon}\min_{v\in\R}
\left\{
\frac12[\sigma_\star(G)+\sqrt\Delta\,\varepsilon-\sigma_0(v)]^2
+\frac{(v-u_0)^2}{2d_\chi}
\right\}.
\label{eq:app-feature-output-potential}
\end{equation}

We denote the minimizer of the scalar problem in Eq.~\eqref{eq:app-feature-output-potential} by $v^\star$. 
It is this scalar minimization that enters the fixed-point equations. 
For the identity Stage-I readout $\sigma_0(v)=v$, the objective is strictly convex and $v^\star$ is unique.

\subsubsection{Calculation of \texorpdfstring{$f_2$}{f2}}

For $f_2$, we evaluate the sums in Eq.~\eqref{eq:app-feature-replica-action}. For the $Q$ terms, substituting the two ansatzes gives
\begin{equation}
\begin{aligned}
\frac1{2n\beta}\sum_{r,s}\widehat Q^{rs}Q^{rs}
&=\frac{n\beta}{2}\widehat q_Qq
+\frac12\widehat q_Q\chi_Q
-\frac12\widehat\chi_Qq
-\frac1{2\beta}\widehat\chi_Q\chi_Q\xrightarrow{\,n\to0\ \mathrm{then}\ \beta\to\infty\,}
\frac12\widehat q_Q\chi_Q-\frac12\widehat\chi_Qq\\
\qquad
\frac1{n\beta}\sum_r\widehat m^rm^r
&=\widehat m m.
\end{aligned}
\label{eq:app-feature-constraint-contribution}
\end{equation}
The $\bar Q$ terms give the same expression with bars. Combining these terms gives $f_2$ in Eq.~\eqref{eq:app-feature-rs-potential}.

\subsubsection{Calculation of \texorpdfstring{$f_3$}{f3}}

To calculate $f_3$, we evaluate the matrix integral in Eq.~\eqref{eq:app-feature-replica-action}.

Substitution into $G_n$ gives
\begin{equation}
G_n=\frac{N_0\beta^2}{2}\Tr\!\left[\left(\sum_rW^r\right)T_F\left(\sum_sW^s\right)^\top\right]-\beta N_0\sum_r\left[\frac12\Tr(W^rA_F(W^r)^\top)-\widehat m\Tr(W^rS)\right].
\label{eq:app-feature-weight-exponent}
\end{equation}
The remaining difficulty in the matrix integral is the quadratic coupling between different replicas through
$\sum_r W^r$. We remove this coupling by introducing an auxiliary Gaussian matrix. For
$(Z_W)_{ij}\overset{\rm iid}{\sim}\Normal(0,1)$, the Hubbard--Stratonovich identity gives
\begin{equation}
\exp\!\left\{
\frac{N_0\beta^2}{2}
\Tr\!\left[
\left(\sum_rW^r\right)
T_F
\left(\sum_sW^s\right)^\top
\right]
\right\}
=
\E_{Z_W}
\exp\!\left[
\beta\sqrt{N_0}
\sum_r
\Tr(W^rT_F^{1/2}Z_W^\top)
\right].
\label{eq:app-feature-hs-identity}
\end{equation}

Conditioned on $Z_W$, the replicas are now independent. Therefore the matrix integral factorizes as
\begin{equation}
\begin{aligned}
\int\prod_r\dd W^r\,e^{G_n}
&=
\E_{Z_W}
\prod_{r=1}^n
\int\dd W^r
\exp\!\left\{
-\beta N_0
\left[
\begin{aligned}
&\frac12\Tr(W^rA_F(W^r)^\top)
-\widehat m\Tr(W^rS)
\\[-1mm]
&-\frac1{\sqrt{N_0}}
\Tr(Z_WT_F^{1/2}(W^r)^\top)
\end{aligned}
\right]
\right\}
\\
&=
\E_{Z_W}\!\left\{
\left[\int\dd W\,e^{-\beta N_0\mathcal L_{\mathrm{eff}}(W;Z_W)}\right]^n
\right\},
\end{aligned}
\label{eq:app-feature-weight-factorization}
\end{equation}
where
\begin{equation}
\mathcal L_{\mathrm{eff}}(W;Z_W)
:=
\frac12\Tr(WA_FW^\top)
-\widehat m\Tr(WS)
-\frac1{\sqrt{N_0}}
\Tr(Z_WT_F^{1/2}W^\top).
\label{eq:app-feature-effective-objective}
\end{equation}

For positive definite $A_F$, the integral over $W$ is Gaussian. The minimizer of its quadratic exponent satisfies
\begin{equation}
0
=
WA_F-\widehat mS-\frac1{\sqrt{N_0}}Z_WT_F^{1/2}
\qquad\Longrightarrow\qquad
W_F(Z_W)
=
\left(
\widehat mS+\frac1{\sqrt{N_0}}Z_WT_F^{1/2}
\right)A_F^{-1}.
\label{eq:app-feature-effective-matrix}
\end{equation}

Completing the square gives the exact integral
\begin{equation}
\int\dd W\,e^{-\beta N_0\mathcal L_{\mathrm{eff}}(W;Z_W)}
=e^{-\beta N_0\mathcal L_{\mathrm{eff}}(W_F(Z_W);Z_W)}
\left(\frac{2\pi}{\beta N_0}\right)^{d^2/2}
(\det A_F)^{-d/2}.
\label{eq:app-feature-matrix-gaussian-integral}
\end{equation}
Its contribution to Eq.~\eqref{eq:app-feature-replica-action}, after taking $n\to0$, is therefore
\begin{equation}
\begin{aligned}
-\lim_{n\to0}\frac1{n\beta dN_0}
\log\int\prod_r\dd W^r\,e^{G_n}
&=-\frac1{\beta dN_0}\E_{Z_W}
\log\int\dd W\,e^{-\beta N_0\mathcal L_{\mathrm{eff}}(W;Z_W)}\\
&=\frac1d\E_{Z_W}\min_W\mathcal L_{\mathrm{eff}}(W;Z_W)
+\frac{\log\det A_F}{2\beta N_0}
+\frac{d}{2\beta N_0}\log\frac{\beta N_0}{2\pi}.
\end{aligned}
\label{eq:app-feature-matrix-limit}
\end{equation}
The last term is independent of the order parameters. If $d^{-1}\log\det A_F$ remains bounded, the determinant term vanishes as $\beta\to\infty$. Substituting the minimizer into $\mathcal L_{\mathrm{eff}}$ gives
\begin{equation}
\begin{aligned}
f_3(\boldsymbol x)=\frac1d\E_{Z_W}\min_W \mathcal L_{\mathrm{eff}}(W;Z_W)
&=
-\frac1{2d}
\E_{Z_W}
\Tr\!\left[
\left(
\widehat mS+\frac{Z_WT_F^{1/2}}{\sqrt{N_0}}
\right)
A_F^{-1}
\left(
\widehat mS+\frac{Z_WT_F^{1/2}}{\sqrt{N_0}}
\right)^\top
\right]
\\
&=
-\frac1{2d}
\left[
\widehat m^2\Tr(S^2A_F^{-1})
+
\frac1{N_0}
\E_{Z_W}
\Tr\!\left(
Z_WT_F^{1/2}A_F^{-1}T_F^{1/2}Z_W^\top
\right)
\right]
\\
&=
-\frac1{2d}
\left[
\widehat m^2\Tr(S^2A_F^{-1})
+
\frac{d}{N_0}\Tr(T_FA_F^{-1})
\right]
\\
&=
-\frac1{2d}
\left[
\widehat m^2\Tr(S^2A_F^{-1})
+
\frac1\kappa\Tr(T_FA_F^{-1})
\right],
\end{aligned}
\label{eq:app-feature-weight-potential}
\end{equation}
where we used
$\E[Z_WT_F^{1/2}A_F^{-1}T_F^{1/2}Z_W^\top]=\Tr(T_FA_F^{-1})I_d$ and $N_0=\kappa d$ in the last two lines.

Together, Eqs.~\eqref{eq:app-feature-output-potential}, \eqref{eq:app-feature-constraint-contribution}, and \eqref{eq:app-feature-weight-potential} give $f_1,f_2,f_3$ in Eq.~\eqref{eq:app-feature-rs-potential}, and hence $f=f_1+f_2+f_3$.

\subsubsection{Stationarity}

We first differentiate $f$ with respect to the five order parameters $(m,q,\bar q,\chi_Q,\chi_{\bar Q})$. At a differentiable scalar minimum, the derivative through $v^\star$ vanishes because the scalar objective is stationary there. Thus
\begin{equation}
\begin{aligned}
0=\partial_m f
&=\widehat m-\gamma\E\!\left[\frac{v^\star-u_0}{d_\chi}
\left(\mu_1G-\frac{\mu_1^2m}{\sqrt V}Z\right)\right],\\
0=2\partial_q f
&=-\widehat\chi_Q-\gamma\mu_1^2\E\!\left[
\frac{(v^\star-u_0)Z}{d_\chi\sqrt V}\right],\\
0=2\partial_{\bar q} f
&=-\widehat\chi_{\bar Q}-\gamma\frac{1+\Delta}{\alpha}\E\!\left[
\frac{(v^\star-u_0)Z}{d_\chi\sqrt V}\right],\\
0=2\partial_{\chi_Q} f
&=\widehat q_Q-\gamma\mu_1^2\E\!\left[
\frac{(v^\star-u_0)^2}{d_\chi^2}\right],\\
0=2\partial_{\chi_{\bar Q}} f
&=\widehat q_{\bar Q}-\gamma\frac{1+\Delta}{\alpha}\E\!\left[
\frac{(v^\star-u_0)^2}{d_\chi^2}\right].
\end{aligned}
\label{eq:app-feature-output-stationarity}
\end{equation}
These are the five conjugate equations in Eq.~\eqref{eq:app-feature-state-equations}. Taking derivatives of $f$ with respect to the conjugates gives the remaining equations:
\begin{equation}
\begin{aligned}
0=\partial_{\widehat m}f
&=m-\frac{\widehat m}{d}\Tr(S^2A_F^{-1}),\\
0=-2\partial_{\widehat\chi_Q}f
&=q-\frac{\widehat m^2}{d}\Tr(S^3A_F^{-2})
-\frac1{\kappa d}\Tr(T_FS A_F^{-2}),\\
0=-2\partial_{\widehat\chi_{\bar Q}}f
&=\bar q-\frac{\widehat m^2}{d}\Tr(S^2A_F^{-2})
-\frac1{\kappa d}\Tr(T_FA_F^{-2}),\\
0=2\partial_{\widehat q_Q}f
&=\chi_Q-\frac1{\kappa d}\Tr(SA_F^{-1}),\\
0=2\partial_{\widehat q_{\bar Q}}f
&=\chi_{\bar Q}-\frac1{\kappa d}\Tr(A_F^{-1}).
\end{aligned}
\label{eq:app-feature-parameter-stationarity}
\end{equation}
Together, the derivatives with respect to the five order parameters and their five conjugates recover Result~\ref{res:app-feature-state-equations} under the fact that as $d,N_0\to\infty$ with $N_0/d\to\kappa_F\in(0,\infty)$, the empirical spectrum of $S_F$ converges almost surely to $\nu_{\kappa_F}(\dd s)$.

\subsection{Derivation of the main results}
\label{app:feature-main-results}

Equation~\eqref{eq:app-feature-effective-matrix} gives the learned attention matrix stated in the main text:
\begin{equation}
\boxed{\widehat W\overset{\rm d}=\left(\widehat mS+\frac1{\sqrt{N_0}}Z(\widehat q_QS+\widehat q_{\bar Q}I_d)^{1/2}\right)\bigl[(\lambda+\widehat\chi_{\bar Q})I_d+\widehat\chi_QS\bigr]^{-1}}.
\label{eq:app-feature-learned-matrix}
\end{equation}
After the fixed point is solved, Stage II is trained on fresh prompts from the training distribution. Its attention output has moments
\begin{equation}
\begin{aligned}
\E[Gu_{\rm tr}]&=\mu_1m,\\
\E[u_{\rm tr}^2]&=\mu_1^2q+\frac{1+\Delta}{\alpha}\bar q
\end{aligned}
\label{eq:app-feature-readout-training-moments}
\end{equation}
Hence, the output of the attention can be represented as 
\begin{equation}
u_{\rm tr}
=
\mu_1mG
+
\sqrt{
\mu_1^2(q-m^2)
+
\frac{1+\Delta}{\alpha}\bar q
}\,Z,
\qquad
G,Z\overset{\rm iid}{\sim}\Normal(0,1).
\label{eq:output-feature-readout-training-moments}
\end{equation}

We next evaluate the same learned attention matrix on the noiseless inference prompts. The only changes are the task distribution and the prompt length $\widetilde L=\widetilde\alpha d$. For a task drawn from the training pool, the relevant overlaps are $m,q,\bar q$, whereas a fresh task averages over an isotropic direction and therefore involves $\E_Z\Tr\widehat W$. This gives
\begin{equation}
\begin{aligned}
C_{\rm M}&=\mu_1m,
&
V_{\rm M}&=\mu_1^2q+\frac{\bar q}{\widetilde\alpha},
\\
C_{\rm G}&=\frac{\mu_1}{d}\E_Z\Tr\widehat W,
&
V_{\rm G}&=\left(\mu_1^2+\frac1{\widetilde\alpha}\right)\bar q.
\end{aligned}
\label{eq:app-feature-test-moments}
\end{equation}
Thus, for either test protocol $\mathcal P\in\{\mathrm M,\mathrm G\}$, the attention output can be represented as
\begin{equation}
u_{\mathcal P}
=
C_{\mathcal P}G
+
\sqrt{V_{\mathcal P}-C_{\mathcal P}^2}\,Z,
\qquad
G,Z\overset{\rm iid}{\sim}\Normal(0,1),
\label{eq:app-feature-test-channel}
\end{equation}
and the corresponding error is
\begin{equation}
\mathcal E_{\mathcal P}^{\mathrm{feature}}
=
\E_{G,Z}
\left[
\left(
\sigma_\star(G)-\widehat\sigma(u_{\mathcal P})
\right)^2
\right].
\label{eq:app-feature-two-errors}
\end{equation}

It remains to determine the Stage-II readout $\widehat\sigma$. The Stage-II training label is
$\sigma_\star(G)+\sqrt\Delta\,\varepsilon$, where $\varepsilon$ is independent of $(G,u_{\rm tr})$. Therefore, for any scalar function $h$,
$
\E\!\left[
\left(
\sigma_\star(G)+\sqrt\Delta\,\varepsilon-h(u_{\rm tr})
\right)^2
\right]
=
\Delta
+
\E\!\left[
\left(
\sigma_\star(G)-h(u_{\rm tr})
\right)^2
\right].
$
The query noise therefore adds only the constant $\Delta$ and does not affect the optimal readout. Conditioning on $u_{\rm tr}$ gives
\begin{equation}
\begin{aligned}
\E\!\left[
(\sigma_\star(G)-h(u_{\rm tr}))^2
\mid u_{\rm tr}
\right]
&=
\Var(\sigma_\star(G)\mid u_{\rm tr})
+
\left(
\E[\sigma_\star(G)\mid u_{\rm tr}]
-h(u_{\rm tr})
\right)^2.
\end{aligned}
\label{eq:app-feature-conditional-projection}
\end{equation}
The second term is minimized by
$
\widehat\sigma(v)
=
\E[\sigma_\star(G)\mid u_{\rm tr}=v].
$

Since $(G,u_{\rm tr})$ is jointly Gaussian, its conditional distribution is
\begin{equation}
G\mid(u_{\rm tr}=v)
\overset{\rm d}=
\frac{\mu_1m}
{\mu_1^2q+(1+\Delta)\bar q/\alpha}\,v
+
\sqrt{
1-
\frac{\mu_1^2m^2}
{\mu_1^2q+(1+\Delta)\bar q/\alpha}
}\,\psi,
\qquad
\psi\sim\Normal(0,1).
\label{eq:conditional-teacher-score}
\end{equation}
Therefore,
\begin{equation}
\boxed{\widehat\sigma(v)=\E_\psi\!\left[
\sigma_\star\!\left(\frac{\mu_1m}
{\mu_1^2q+(1+\Delta)\bar q/\alpha}\,v+\sqrt{
1-\frac{\mu_1^2m^2}
{\mu_1^2q+(1+\Delta)\bar q/\alpha}
}\,\psi\right)\right]},
\label{eq:app-feature-optimal-readout}
\end{equation}
where $\psi\sim\Normal(0,1)$. Equations~\eqref{eq:app-feature-learned-matrix}--\eqref{eq:app-feature-optimal-readout} determine the learned attention matrix, the Stage-II readout, and the two inference errors, and therefore derive Result \ref{res:feature}.

Now we set \(\sigma_0(v)=v\).
We first study the large-pretraining regime ($\gamma\to\infty$ with $\alpha,\widetilde\alpha$ fixed).
Then we take the infinite-pretraining and infinite-task-pool limits with the more general context scaling \(L=\widetilde L=\alpha d^s\).

\paragraph{Large-pretraining limit.}

For \(\sigma_0(v)=v\), the scalar minimizer in
Eq.~\eqref{eq:app-feature-state-equations} is
\begin{equation}
v^\star=\frac{u_0+d_\chi y}{1+d_\chi},
\qquad
y:=\sigma_\star(G)+\sqrt{\Delta}\,\varepsilon.
\label{eq:app-feature-comparison-vstar}
\end{equation}
Hence the conjugate equations simplify to
\begin{equation}
\begin{aligned}
\widehat m=\widehat\chi_Q&=\frac{\gamma\mu_1^2}{1+d_\chi},
\\
\widehat\chi_{\bar Q}&=\frac{\gamma(1+\Delta)}
{\alpha(1+d_\chi)},
\\
\widehat q_Q&=\frac{\gamma\mu_1^2 R}{(1+d_\chi)^2},
\\
\widehat q_{\bar Q}&=\frac{\gamma(1+\Delta)R}
{\alpha(1+d_\chi)^2},
\end{aligned}
\label{eq:app-feature-comparison-conjugates}
\end{equation}
where $R:=1+\Delta-2\mu_1^2m+\mu_1^2q+\frac{1+\Delta}{\alpha}\bar q$.
Define $\tau:=\frac{1+\Delta}{\alpha\mu_1^2}$.
Equation~\eqref{eq:app-feature-comparison-conjugates} implies
\begin{equation}
\widehat\chi_{\bar Q}=\tau\widehat m,
\qquad
\widehat q_{\bar Q}=\tau\widehat q_Q.
\label{eq:app-feature-comparison-proportionality}
\end{equation}
Therefore
\begin{equation}
A=\widehat m\left[S+\left(\tau+\frac{\lambda}{\widehat m}\right)I_d\right],
\qquad
T=\widehat q_Q(S+\tau I_d).
\label{eq:app-feature-comparison-AT}
\end{equation}
Substituting these identities into
Eq.~\eqref{eq:app-feature-learned-matrix} gives
\begin{equation}
\begin{aligned}
\widehat W\overset{\rm d}=S
\left[S+\left(\tau+\frac{\lambda}{\widehat m}\right)I_d\right]^{-1}+\sqrt{\frac{\widehat q_Q}{N_0\widehat m^2}}\,Z(S+\tau I_d)^{1/2}
\left[S+\left(\tau+\frac{\lambda}{\widehat m}\right)I_d\right]^{-1}.
\end{aligned}
\label{eq:app-feature-comparison-W-decomposition}
\end{equation}
Moreover,
Eq.~\eqref{eq:app-feature-comparison-conjugates} yields
\begin{equation}
\frac{\widehat q_Q}{\widehat m^2}=\frac{R}{\gamma\mu_1^2}.
\label{eq:app-feature-comparison-noise-ratio}
\end{equation}
For fixed \(\alpha>0\), the resolvents in
Eq.~\eqref{eq:app-feature-comparison-W-decomposition} are uniformly bounded
because \(\tau>0\).
Using the equations for \(\chi_Q\) and \(\chi_{\bar Q}\),
\begin{equation}
d_\chi=\frac{\mu_1^2}{\kappa d\,\widehat m}
\Tr\!\left[(S+\tau I_d)\left(S+\left(\tau+\lambda/\widehat m\right)I_d
\right)^{-1}\right]=O(\widehat m^{-1}),
\end{equation}
and hence
\begin{equation}
\widehat m=\Theta(\gamma),
\qquad
\frac{\lambda}{\widehat m}=O(\gamma^{-1}).
\label{eq:app-feature-comparison-mhat-scaling}
\end{equation}
Since $N=\gamma dN_0=\gamma\kappa d^2$,
we have,
\begin{equation}
\gamma^{-1}=O\!\left(\frac{d^2}{N}\right).
\label{eq:app-feature-comparison-gamma-scaling}
\end{equation}
The large-\(\gamma\) limit of
Eq.~\eqref{eq:app-feature-comparison-W-decomposition} is therefore
\begin{equation}
W_{\infty}:=S(S+\tau I_d)^{-1}.
\label{eq:app-feature-large-N-solution}
\end{equation}
Write $\widehat W=W_{\infty}+\delta W_{\rm det}+\delta W_{\rm fl}$, where \(\delta W_{\rm det}\) is the correction caused by
\(\lambda/\widehat m\), and \(\delta W_{\rm fl}\) is the fluctuation (the last term of
Eq.~\eqref{eq:app-feature-comparison-W-decomposition}).
Equation~\eqref{eq:app-feature-comparison-mhat-scaling} gives
\begin{equation}
\|\delta W_{\rm det}\|_{\rm op}=O\!\left(\frac{d^2}{N}\right).
\label{eq:app-feature-comparison-deterministic-correction}
\end{equation}
Conditioned on \(S\), the fluctuation \(\delta W_{\rm fl}\) has zero mean. Moreover, using \(\gamma N_0=N/d\),
\begin{equation}
\begin{aligned}
\frac1d\E_Z\|\delta W_{\rm fl}\|_F^2
&=\frac{R}{\gamma\mu_1^2N_0}
\Tr\!\left[(S+\tau I_d)\left(S+\left(\tau+\lambda/\widehat m\right)I_d\right)^{-2}\right]
\\
&=\frac{dR}{\mu_1^2N}\Tr\!\left[(S+\tau I_d)\left(S+\left(\tau+\lambda/\widehat m\right)I_d\right)^{-2}\right].
\end{aligned}
\label{eq:app-feature-finite-N-fluctuation}
\end{equation}
For fixed \(\alpha>0\), we have \(\tau>0\), so the normalized trace in
Eq.~\eqref{eq:app-feature-finite-N-fluctuation} is \(O(1)\).
Furthermore,
\begin{equation}
R
=\E\!\left[
\left(
\sigma_\star(G)+\sqrt\Delta\,\varepsilon-u_0
\right)^2
\right].
\label{eq:app-feature-RF-mse}
\end{equation}
As the large-\(\gamma\) solution converges to the finite population solution
\((m_\infty,q_\infty,\bar q_\infty)\), we have
\begin{equation}
R\longrightarrow R_{\infty}:=1+\Delta-2\mu_1^2m_\infty+\mu_1^2q_\infty+\frac{1+\Delta}{\alpha}\bar q_\infty,
\end{equation}
which is finite. Consequently,
\begin{equation}
\frac1d\E_Z\|\delta W_{\rm fl}\|_F^2=O\!\left(\frac{d^2}{N}\right).
\label{eq:app-feature-finite-N-fluctuation-rate}
\end{equation}
It follows that the order parameters differ from their infinite-pretraining values by \(O(d^2/N)\):
\begin{equation}
\begin{aligned}
m&=m_\infty+O\!\left(\frac{d^2}{N}\right),
\\
q&=q_\infty+O\!\left(\frac{d^2}{N}\right),
\\
\bar q&=\bar q_\infty+O\!\left(\frac{d^2}{N}\right),
\\
\frac1d\E_Z\Tr\widehat W_F&=w_\infty+O\!\left(\frac{d^2}{N}\right),
\end{aligned}
\label{eq:app-feature-finite-N-overlap-correction}
\end{equation}
where
\begin{equation}
\begin{aligned}
m_\infty
&=
\frac1d
\Tr\!\left[
S^2(S+\tau I_d)^{-1}
\right],
\\
q_\infty
&=
\frac1d
\Tr\!\left[
S^3(S+\tau I_d)^{-2}
\right],
\\
\bar q_\infty
&=
\frac1d
\Tr\!\left[
S^2(S+\tau I_d)^{-2}
\right],
\\
w_\infty
&=
\frac1d
\Tr\!\left[
S(S+\tau I_d)^{-1}
\right].
\end{aligned}
\label{eq:app-feature-large-N-overlaps}
\end{equation}
The linear fluctuation terms vanish after averaging over \(Z\), while their
quadratic contribution is controlled by
Eq.~\eqref{eq:app-feature-finite-N-fluctuation}.

The Stage-II training at infinite pretraining has moments
\begin{equation}
C_{\rm tr}^{(\infty)}
=
\mu_1m_\infty,
\qquad
V_{\rm tr}^{(\infty)}
=
\mu_1^2q_\infty
+
\frac{1+\Delta}{\alpha}\bar q_\infty.
\label{eq:app-feature-large-N-training-moments}
\end{equation}
Accordingly, define the population-optimal readout
\begin{equation}
\widehat\sigma_\infty(v)
:=
\E\!\left[
\sigma_\star(G)
\,\middle|\,
C_{\rm tr}^{(\infty)}G
+
\sqrt{
V_{\rm tr}^{(\infty)}
-
(C_{\rm tr}^{(\infty)})^2
}\,Z
=v
\right].
\label{eq:app-feature-large-N-readout}
\end{equation}
For memorization and generalization, respectively, the infinite-pretraining
test moments are
\begin{equation}
\begin{aligned}
C_{\rm M}^{(\infty)}
&=
\mu_1m_\infty,
&
V_{\rm M}^{(\infty)}
&=
\mu_1^2q_\infty
+
\frac{\bar q_\infty}{\widetilde\alpha},
\\
C_{\rm G}^{(\infty)}
&=
\mu_1w_\infty,
&
V_{\rm G}^{(\infty)}
&=
\left(
\mu_1^2+\frac1{\widetilde\alpha}
\right)\bar q_\infty.
\end{aligned}
\label{eq:app-feature-large-N-test-moments}
\end{equation}
Hence we define
\begin{equation}
\boxed{
\begin{aligned}
g_{\mathcal P}(\alpha,\widetilde\alpha,S)
:=
\E_{G,Z}\!\left[
\left(
\sigma_\star(G)
-
\widehat\sigma_\infty
\left(
C_{\mathcal P}^{(\infty)}G
+
\sqrt{
V_{\mathcal P}^{(\infty)}
-
(C_{\mathcal P}^{(\infty)})^2
}\,Z
\right)
\right)^2
\right],
\end{aligned}
}
\qquad
\mathcal P\in\{\mathrm M,\mathrm G\}.
\label{eq:app-feature-large-N-errors}
\end{equation}
For fixed positive \(\alpha,\widetilde\alpha\), these Gaussian channels are
nondegenerate. Their covariances, and hence the corresponding conditional
expectation in
Eq.~\eqref{eq:app-feature-large-N-readout}, vary smoothly under the
\(O(d^2/N)\) perturbations in
Eq.~\eqref{eq:app-feature-finite-N-overlap-correction}.
Therefore
\begin{equation}
\boxed{
\mathcal E_{\mathcal P}^{\mathrm{feature}}
=
g_{\mathcal P}
(\alpha_F,\widetilde\alpha_F,S_F)
+
O\!\left(\frac{d^2}{N}\right),
\qquad
\mathcal P\in\{\mathrm M,\mathrm G\}.
}
\label{eq:app-feature-finite-N-error}
\end{equation}
This gives the feature-learner finite-pretraining correction.

\paragraph{Infinite task-pool limit and context-length scalings.}

We next set \(\Delta=\lambda=0\) and take first
\(\gamma\to\infty\) and then \(\kappa\to\infty\).
\eqref{eq:app-feature-large-N-solution} gives
\begin{equation}
W_{\infty}=S\left(S+\frac{1}{\alpha\mu_1^2}I_d
\right)^{-1}.
\label{eq:app-feature-population-W-general}
\end{equation}
In the \(\kappa\to\infty\) limit, isotropy of the task distribution implies $S\longrightarrow I_d$.
Hence
\begin{equation}
W_{\infty}\longrightarrow r_{\alpha}I_d,
\qquad
r_{\alpha}:=\frac{\alpha\mu_1^2}{1+\alpha\mu_1^2}.
\label{eq:app-feature-population-W}
\end{equation}
The corresponding limiting order parameters therefore satisfy
\begin{equation}
m_\infty=r_{\alpha},
\qquad
q_\infty=\bar q_\infty=r_{\alpha}^2,
\qquad
\frac1d\Tr W_{\infty}=r_{\alpha}.
\label{eq:app-feature-population-overlaps}
\end{equation}
For $\beta>0$, define
\begin{equation}
B_\beta
:=
\sqrt{c_1^2+\frac1\beta},
\qquad
\rho_\beta
:=
\frac{c_1}{B_\beta},
\qquad
r_\beta:=\rho_\beta^2
=
\frac{\beta c_1^2}{1+\beta c_1^2}.
\label{eq:app-feature-rho-definition}
\end{equation}
The Stage-II training channel can be written as
\begin{equation}
u_{\rm tr}=r_\alpha B_\alpha T_\alpha,
\qquad
T_\alpha=\rho_\alpha G+\sqrt{1-\rho_\alpha^2}\,Z,
\label{eq:app-feature-training-standardized-channel}
\end{equation}
where $G,Z\overset{\rm iid}{\sim}\Normal(0,1)$. By Mehler's identity, the optimal Stage-II readout is
\begin{equation}
\widehat\sigma_\alpha(v)=\sum_{k\ge1}
c_k\rho_\alpha^kH_k\left(\frac{v}{r_\alpha B_\alpha}\right).
\label{eq:app-feature-mismatch-readout}
\end{equation}
At inference with context length
$\widetilde L=\widetilde\alpha d$, the learned attention matrix produces
\begin{equation}
u_{\rm te}=r_\alpha B_{\widetilde\alpha}
T_{\widetilde\alpha},
\qquad
T_{\widetilde\alpha}=\rho_{\widetilde\alpha}G+\sqrt{1-\rho_{\widetilde\alpha}^2}\,Z.
\label{eq:app-feature-mismatch-test-channel}
\end{equation}
Define $s_{\alpha,\widetilde\alpha}
:=\frac{B_{\widetilde\alpha}}{B_\alpha}=\frac{\rho_\alpha}{\rho_{\widetilde\alpha}}$.
For the Hermite polynomials,
\begin{equation}
H_n(sx)=\sum_{m=0}^{\lfloor n/2\rfloor}
\sqrt{\frac{n!}{(n-2m)!}}\frac{s^{n-2m}}{m!}\left(\frac{s^2-1}{2}\right)^m
H_{n-2m}(x).
\label{eq:app-hermite-scaling}
\end{equation}
Hence
\begin{equation}
\widehat\sigma_\alpha(u_{\rm te})=\sum_{\ell\ge0}A_\ell(\alpha,\widetilde\alpha)H_\ell(T_{\widetilde\alpha}),
\end{equation}
where
\begin{equation}
A_\ell(\alpha,\widetilde\alpha):=\sum_{m\ge0}c_{\ell+2m}
\rho_\alpha^{\ell+2m}
\sqrt{\frac{(\ell+2m)!}{\ell!}}
\frac{s_{\alpha,\widetilde\alpha}^{\ell}}{m!}
\left(
\frac{s_{\alpha,\widetilde\alpha}^2-1}{2}
\right)^m.
\label{eq:app-feature-mismatch-coefficients}
\end{equation}
Let
\[
f_{\alpha,\widetilde\alpha}(T_{\widetilde\alpha})
:=
\widehat\sigma_\alpha(u_{\rm te})
=
\sum_{\ell\ge0}
A_\ell(\alpha,\widetilde\alpha)
H_\ell(T_{\widetilde\alpha}).
\]
On the other hand, Mehler's identity gives the Bayes-optimal predictor for the inference channel,
\begin{equation}
m_{\widetilde\alpha}(T_{\widetilde\alpha})
:=
\E[\sigma_\star(G)\mid T_{\widetilde\alpha}]
=
\sum_{\ell\ge1}
c_\ell\rho_{\widetilde\alpha}^{\ell}
H_\ell(T_{\widetilde\alpha}).
\end{equation}
Since
$\E[\sigma_\star(G)-m_{\widetilde\alpha}(T_{\widetilde\alpha})
\mid T_{\widetilde\alpha}
]=0$, we have
\begin{equation}
\E\left[\left(\sigma_\star(G)-f_{\alpha,\widetilde\alpha}(T_{\widetilde\alpha})
\right)^2\right]=\E\left[\left(\sigma_\star(G)-m_{\widetilde\alpha}(T_{\widetilde\alpha})\right)^2\right]+\E\left[\left(m_{\widetilde\alpha}(T_{\widetilde\alpha})-f_{\alpha,\widetilde\alpha}(T_{\widetilde\alpha})\right)^2\right].
\label{eq:app-feature-mismatch-pythagorean}
\end{equation}
The first term is the MMSE of the inference channel. Using
$\E[\sigma_\star(G)^2]=1$ and the orthonormality of the Hermite
polynomials,
\begin{equation}
\E\left[
\left(
\sigma_\star(G)
-
m_{\widetilde\alpha}(T_{\widetilde\alpha})
\right)^2
\right]
=
1-
\sum_{k\ge1}
c_k^2r_{\widetilde\alpha}^k.
\end{equation}
The same orthonormality gives
\begin{equation}
\E\left[
\left(
m_{\widetilde\alpha}(T_{\widetilde\alpha})
-
f_{\alpha,\widetilde\alpha}(T_{\widetilde\alpha})
\right)^2
\right]
=
\sum_{\ell\ge0}
\left[
A_\ell(\alpha,\widetilde\alpha)
-
c_\ell\rho_{\widetilde\alpha}^{\ell}
\right]^2,
\end{equation}
where $c_0=0$. Therefore,
\begin{equation}
\boxed{
\begin{aligned}
\mathcal E_{\mathcal P}^{\mathrm{feature}}
={}&
1-\sum_{k\ge1}c_k^2r_{\widetilde\alpha}^k
+
\sum_{\ell\ge0}
\left[
A_\ell(\alpha,\widetilde\alpha)
-
c_\ell\rho_{\widetilde\alpha}^{\ell}
\right]^2
+o(1),
\end{aligned}}
\qquad
\mathcal P\in\{\mathrm M,\mathrm G\}.
\label{eq:app-feature-population-error-mismatch}
\end{equation}

Although the replica calculation above was formulated with a fixed 
\(\alpha=L/d\),
\eqref{eq:app-feature-population-error-mismatch}
naturally allows us to examine \(d\)-dependent context scalings by substituting $\alpha(d)=\frac{L(d)}{d}$ before taking \(d\to\infty\). Consider
\begin{equation}
L=\widetilde L=\alpha d^s,
\qquad
\alpha>0,\quad s>0.
\label{eq:app-feature-general-context-scaling}
\end{equation}
Then $\alpha(d)=\alpha d^{s-1}$, and hence, using \(\mu_1=c_1\),
\begin{equation}
r_{\alpha(d)}=\frac{\alpha c_1^2d^{s-1}}{1+\alpha c_1^2d^{s-1}}
\longrightarrow
\begin{cases}
0,
& s<1,\\[4pt]
\displaystyle
r_\alpha:=
\frac{\alpha c_1^2}{1+\alpha c_1^2},
& s=1,\\[10pt]
1,
& s>1.
\end{cases}
\label{eq:app-feature-correlation-context-scaling}
\end{equation}
Substituting this limit into
Eq.~\eqref{eq:app-feature-population-error-mismatch} gives
\begin{equation}
\boxed{
\mathcal E_{\mathcal P}^{\mathrm{feature}}
\longrightarrow
\begin{cases}
1,
& s<1,\\[4pt]
\displaystyle
1-\sum_{k=1}^{\infty}c_k^2r_\alpha^k,
& s=1,\\[10pt]
0,
& s>1,
\end{cases}
\qquad
\mathcal P\in\{\mathrm M,\mathrm G\}.
}
\label{eq:app-feature-general-context-limit}
\end{equation}

\section{Additional experiments}
\subsection{Consistency of the Replica method with numerical experiments}
\label{app:consistency}
We first validate our analytical results from the replica method against numerical experiments.
We use the same learned attention matrices and the same fixed-point solutions as in Section~\ref{sec:experiments}. All runs use the teacher in Eq.~\eqref{eq:common-teacher}, $\Delta=0.01$, and $\lambda=10^{-3}$. In Figure \ref{fig:replica-order-parameters} we compare the order parameters obtained from numerics against those obtained from replica under the following settings:
\begin{center}
\begin{tabular}{@{}lccccc@{}}
\toprule
Panels & $d$ & $N$ & $N_0$ & $\alpha$ & $\widetilde\alpha$\\
\midrule
A--B & $32$ & $288$ to $1200000$ & $24$ & $32$ & $32$\\
C--D & $24$ & $1200000$ & $24,48,96,144,192$ & $32$ & $32$\\
\bottomrule
\end{tabular}
\end{center}

\begin{figure}[t]
\centering
\includegraphics[width=\textwidth]{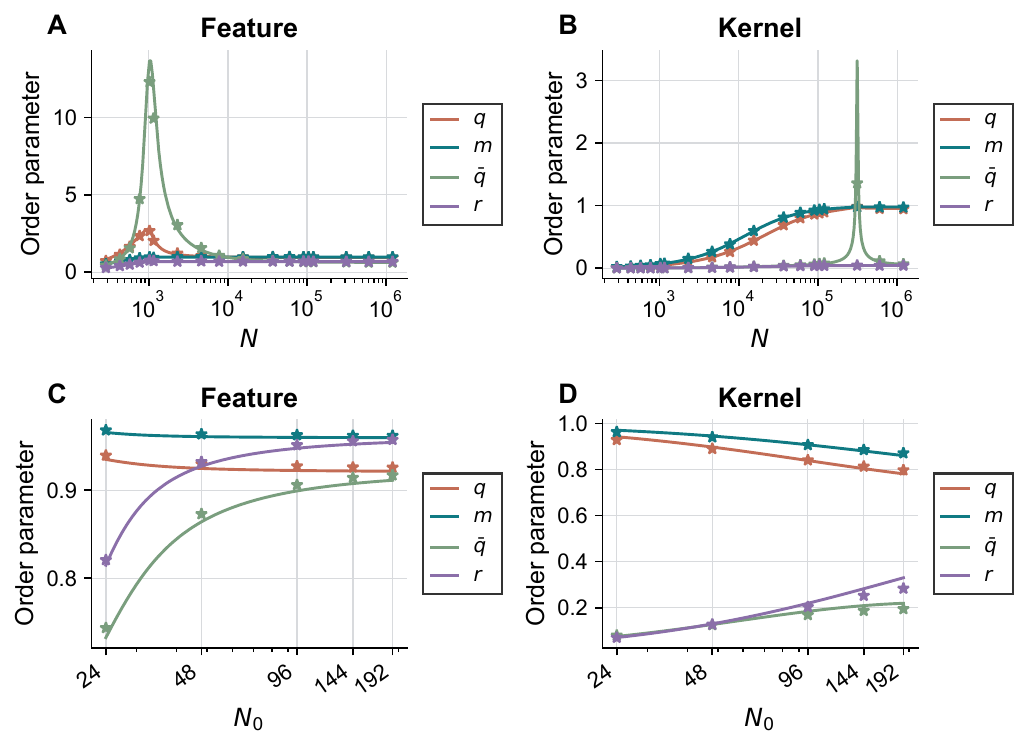}
\caption{Consistency of the replica method with numerical experiments. Left: feature learner; right: kernel learner. The upper row varies $N$ at $d=32$, and the lower row varies $N_0$ at $d=24$, $N=1200000$. Solid lines are replica predictions, and stars show the order parameters measured from the learned attention matrices in one run. Red, teal, green, and purple denote $q,m,\bar q,r$, respectively, for the learner in that column.}
\label{fig:replica-order-parameters}
\end{figure}

The plotted $m,q,\bar q$ are the corresponding learner's order parameters in Eqs.~\eqref{eq:kernel-errors-main}, \eqref{eq:app-kernel-qbar-definition}, and \eqref{eq:app-feature-order-parameters}, measured from the saved attention matrices. The additional quantity $r$ is $d^{-1}\Tr\widehat W$ for the feature learner and $D^{-1}\Tr\widehat W_K$ for the kernel learner. The feature learner and linear attention share the same learned attention, and Stage II leaves these quantities unchanged. Figure \ref{fig:replica-order-parameters} shows that the replica predictions capture the main trends of the measured order parameters, with visible deviations near the kernel interpolation peak.

\subsection{Varying task diversity}
\label{app:vary_N0}
The following experiment varies the task-diversity $N_0$ with $N=1200000$ and $\alpha=\widetilde\alpha=32$, and $N_0\in\{24,48,96,144,192\}$. Increasing $N_0$ means a more diverse task pool and fewer training prompts per task.

\begin{figure}[t]
\centering
\includegraphics[width=\textwidth]{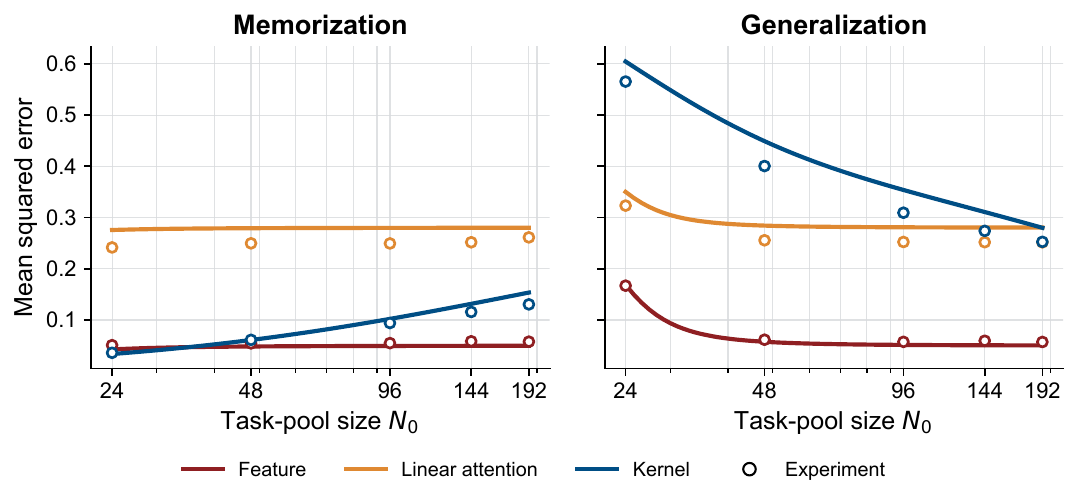}
\caption{Memorization (left) and generalization (right) errors as a function of $N_0$. We fix $d=24$, $D=323$, $N=1200000$, $L=\widetilde L=768$, $\Delta=0.01$, and $\lambda=10^{-3}$. Red, orange, and blue denote the feature learner, linear attention, and the kernel learner, respectively. Solid lines show replica predictions, and circles show numerical results. The three learners use the same training and inference prompts at each $N_0$, with $32768$ inference prompts per error estimate.}
\label{fig:kappa-merged}
\end{figure}

Figure \ref{fig:kappa-merged} shows that increasing task diversity narrows the gap between memorization and generalization for both learners. For the kernel learner, this improves performance on unseen tasks (generalization) at the cost of accuracy on the training tasks (memorization). For the feature learner, the main effect is an improvement in generalization.

The replica analysis also predicts this crossover.
For the common teacher in Eq.~\eqref{eq:common-teacher}, set
$\Delta=\lambda=0$, $\sigma_0(v)=v$, and $L=\widetilde L=\alpha d$. We first take $N\to\infty$,
then $d,N_0\to\infty$ at fixed $\kappa=N_0/d$, and finally
$\alpha\to\infty$. In Appendix \ref{app:task-diversity-comparison}, we show that the memorization errors satisfy
\begin{equation}
\begin{aligned}
\mathcal E^F_M(\infty)
&=
\frac{5\min(1,\kappa)}{3\alpha}
+o(\alpha^{-1}),
&\mathcal E^K_M(\infty)
&=
\frac{\kappa}{\alpha}
+o(\alpha^{-1}).
\end{aligned}
\label{eq:task-diversity-comparison}
\end{equation}
For $\kappa\geq1$, which covers the scan in Fig.~\ref{fig:kappa-merged}, the leading
terms therefore give
\begin{equation}
\mathcal E^K_M(\infty)<\mathcal E^F_M(\infty)
\quad\Longleftrightarrow\quad
\kappa<\frac53.
\label{eq:kernel-memory-threshold}
\end{equation}
For $d=24$, this gives $N_0\simeq40$, consistent with the crossover between $N_0=24$ and $48$ in Figure \ref{fig:kappa-merged}.

\subsection{Derivation of \texorpdfstring{Eq.~\eqref{eq:task-diversity-comparison}}{the task-diversity comparison}}
\label{app:task-diversity-comparison}

We set $\Delta=\lambda=0$, $\sigma_0(v)=v$, and $L=\widetilde L=\alpha d$.
As in Eq.~\eqref{eq:task-diversity-comparison}, we first take $N\to\infty$, then $d,N_0\to\infty$ with $N_0/d\to\kappa\in(0,\infty)$, and finally $\alpha\to\infty$.
The common teacher in Eq.~\eqref{eq:common-teacher} is
$\sigma_\star(G)=\frac{\sqrt3}{2}H_1(G)+\frac12H_2(G)$, so $\mu_1^2=3/4$.

\paragraph{Feature learner.}
With $\tau_\alpha=1/(\alpha\mu_1^2)=4/(3\alpha)$, Eq.~\eqref{eq:app-feature-large-N-overlaps} gives
\begin{equation}
\begin{aligned}
q_\infty+\tau_\alpha\bar q_\infty
&=\int\frac{s^3+\tau_\alpha s^2}{(s+\tau_\alpha)^2}\,\nu_{\kappa}(\dd s)
=\int\frac{s^2}{s+\tau_\alpha}\,\nu_{\kappa}(\dd s)
=m_\infty,\\
V_{\mathrm M}^{(\infty)}
&=\mu_1^2q_\infty+\frac{\bar q_\infty}{\alpha}
=\mu_1^2m_\infty,
\qquad
\frac{(C_{\mathrm M}^{(\infty)})^2}{V_{\mathrm M}^{(\infty)}}
=\frac{\mu_1^2m_\infty^2}{\mu_1^2m_\infty}
=m_\infty.
\end{aligned}
\label{eq:task-diversity-feature-moments}
\end{equation}
Since the Stage-II training and memorization distributions coincide, substituting this identity into Eq.~\eqref{eq:app-feature-large-N-errors} yields
\begin{equation}
\mathcal E_{\mathrm M}^{\mathrm{feature}}(\infty)
=\frac34(1-m_\infty)+\frac14(1-m_\infty^2)
=\frac54(1-m_\infty)-\frac14(1-m_\infty)^2.
\label{eq:task-diversity-feature-error}
\end{equation}
The Marchenko--Pastur spectrum satisfies
$\int s\,\nu_{\kappa}(\dd s)=1$ and $\nu_{\kappa}((0,\infty))=\min(1,\kappa)$. Therefore
\begin{equation}
\begin{aligned}
1-m_\infty
&=\int\left(s-\frac{s^2}{s+\tau_\alpha}\right)\nu_{\kappa}(\dd s)
=\tau_\alpha\int\frac{s}{s+\tau_\alpha}\,\nu_{\kappa}(\dd s),\\
\lim_{\alpha\to\infty}\frac{1-m_\infty}{\tau_\alpha}
&=\int\mathbf1_{\{s>0\}}\,\nu_{\kappa}(\dd s)
=\min(1,\kappa).
\end{aligned}
\label{eq:task-diversity-feature-spectrum-limit}
\end{equation}
As $\alpha\to\infty$, $\tau_\alpha\to0$, and the ratio $s/(s+\tau_\alpha)$ tends to 1 for every $s>0$ and remains 0 at $s=0$.
Since this ratio always lies between 0 and 1, we can take its limit inside the integral, also when $\kappa=1$.
Combining Eqs.~\eqref{eq:task-diversity-feature-error} and \eqref{eq:task-diversity-feature-spectrum-limit}, we obtain
\begin{equation}
\mathcal E_{\mathrm M}^{\mathrm{feature}}(\infty)
=\frac54\tau_\alpha\min(1,\kappa)+o(\tau_\alpha)
=\frac{5\min(1,\kappa)}{3\alpha}+o(\alpha^{-1}).
\label{eq:task-diversity-feature-leading-error}
\end{equation}

\paragraph{Kernel learner.}
Since \(\widetilde\alpha_K=\alpha_K\) and \(\Delta=0\), the two trace
terms in Eq.~\eqref{eq:app-kernel-large-N-memorization} combine to give
\begin{equation}
\mathcal E_{\mathrm M}^{\mathrm{kernel}}(\infty)=\frac{1}{\alpha_KD}
\Tr\!\left[S_K\left(S_K+\frac1{\alpha_K}I_D\right)^{-1}\right].
\label{eq:task-diversity-kernel-trace}
\end{equation}
Let
\begin{equation}
G_K:=\frac1D\Theta_K^\top\Theta_K,
\qquad
\Theta_K=
\begin{bmatrix}
\theta(a^1)&\cdots&\theta(a^{N_0})
\end{bmatrix}.
\end{equation}
The nonzero eigenvalues of \(G_K\) coincide with those of
\((N_0/D)S_K\). Since
$\frac{N_0}{L}\longrightarrow\frac{\kappa}{\alpha}$,
Eq.~\eqref{eq:task-diversity-kernel-trace} can therefore be written as
\begin{equation}
\mathcal E_{\mathrm M}^{\mathrm{kernel}}(\infty)=\frac{N_0}{L}\,
\frac1{N_0}\Tr\!\left[G_K\left(G_K+\frac{N_0}{L}I_{N_0}
\right)^{-1}\right].
\label{eq:task-diversity-kernel-Gram-error}
\end{equation}
For the common teacher,
\(c_1^2=3/4\) and \(c_2^2=1/4\). By the spectral limit
derived in Eq.~\eqref{eq:app-kernel-Gram-limit-law}, when
\(N_0/d\to\kappa\), the empirical spectral distribution of \(G_K\) converges to the law
\begin{equation}
g=\frac14+\frac34 s,
\qquad
s\sim\nu_{1/\kappa}.
\label{eq:task-diversity-kernel-Gram-law}
\end{equation}
Hence, after taking \(d,N_0\to\infty\) at fixed \(\kappa\),
\begin{equation}
\begin{aligned}
\mathcal E_{\mathrm M}^{\mathrm{kernel}}(\infty)=\frac{\kappa}{\alpha}\int\frac{g}{g+\kappa/\alpha}\,\mu_\kappa(\dd g)=\frac{\kappa}{\alpha}\left[1-\frac{\kappa}{\alpha}\mathcal R_\kappa\!\left(\frac{\kappa}{\alpha}\right)\right],
\end{aligned}
\label{eq:task-diversity-kernel-resolvent-error}
\end{equation}
where \(\mu_\kappa\) denotes the law in
Eq.~\eqref{eq:task-diversity-kernel-Gram-law} and
\begin{equation}
\mathcal R_\kappa(z)
:=
\int\frac{\mu_\kappa(\dd g)}{g+z}.
\end{equation}
Since \(g\geq1/4\) almost surely,
\begin{equation}
0\leq
\mathcal R_\kappa(z)
\leq4,
\qquad z\geq0.
\end{equation}
Therefore, for every fixed \(\kappa\),
\begin{equation}
\boxed{
\mathcal E_{\mathrm M}^{\mathrm{kernel}}(\infty)
=
\frac{\kappa}{\alpha}
+
O(\alpha^{-2})
=
\frac{\kappa}{\alpha}
+
o(\alpha^{-1}).
}
\label{eq:task-diversity-kernel-leading-error}
\end{equation}

\section{A formal statement of Theorem \ref{theo:comparison}}
\label{app:rigorous}

In this section, we give a formal version of Theorem \ref{theo:comparison}.

\paragraph{Population setting.}
For each fixed input dimension \(d\) and context length \(L\), a prompt is generated by first sampling
\begin{equation}
a,x_0,x_1,\ldots,x_L\overset{\mathrm{i.i.d.}}{\sim}
\operatorname{Unif}\!\left(\mathbb S^{d-1}(\sqrt d)\right),
\end{equation}
and then setting
\begin{equation}
y_\ell=\sigma_\star\!\left(\frac{a^\top x_\ell}{\sqrt d}\right),
\qquad
\ell=0,1,\ldots,L.
\label{eq:rigorous-population-prompt}
\end{equation}
Thus we study the problem with \(\Delta=\lambda=0\). The task vector is sampled independently for every prompt,
and the training objectives are defined as population expectations. This is the population limit, corresponding to $\gamma\to\infty$ and $\kappa\to\infty$ in Theorem \ref{theo:comparison}.

Let $G_d:=\frac{a^\top x}{\sqrt d}$, where
$a,x\overset{\mathrm{i.i.d.}}{\sim}
\operatorname{Unif}\!\left(\mathbb S^{d-1}(\sqrt d)\right)$, and let \(P_{k,d}\) denote the degree-\(k\) Gegenbauer polynomial ($k\geq0$) normalized
so that $\mathbb E[P_{j,d}(G_d)P_{k,d}(G_d)]=\delta_{jk}$.
We write
\begin{equation}
v_d:=\mathbb E[\sigma_\star(G_d)^2],
\qquad
c_{k,d}:=
\mathbb E[\sigma_\star(G_d)P_{k,d}(G_d)].
\label{eq:rigorous-finite-d-coefficients}
\end{equation}
Recall that the Gaussian Hermite coefficients are
\begin{equation}
c_k:=\mathbb E[\sigma_\star(G)H_k(G)],
\qquad
G\sim\mathcal N(0,1).
\end{equation}

We will work under the following two regularity assumptions.
\begin{assumption}
\label{ass:rigorous-teacher}
The teacher \(\sigma_\star:\mathbb R\to\mathbb R\) is a measurable function independent of \(d\) and \(L\), satisfying
\begin{equation}
\mathbb E[\sigma_\star(G)]=0,\qquad
\mathbb E[\sigma_\star(G)^2]=1,\qquad
c_1:=\mathbb E[G\sigma_\star(G)]\neq0,
\end{equation}
and for some \(C,p<\infty\):
\begin{equation}
|\sigma_\star(x)|
\le C(1+|x|^p)
\end{equation}
\end{assumption}
\begin{assumption}
\label{ass:rigorous-limit}
Assume that \(d,L,\widetilde L\to\infty\) such that, for some fixed
\(c,C,\varepsilon>0\),
\begin{equation}
c\leq\frac{L}{\widetilde L}\leq C,
\qquad
\min\{L,\widetilde L\}\geq d^\varepsilon .
\label{eq:proof-feature-context-assumption}
\end{equation}
\end{assumption}

\paragraph{Kernel learner.}
Fix a finite set of degrees
\(\mathcal K\subset\mathbb N_{\geq1}\), independent of \(d\), and let
\begin{equation}
h(x)=\bigoplus_{k\in\mathcal K}h_k(x)
\in\mathbb R^D,
\qquad
D=\sum_{k\in\mathcal K}D_{k,d},
\end{equation}
where \(h_k\) is an orthonormal basis of the degree-\(k\) spherical harmonics and $D_{k,d}=\frac{2k+d-2}{k}\binom{k+d-3}{k-1}$.  For a prompt define
\begin{equation}
H^{\mathrm K,L}:=\frac1L
h(x_0)\left(\sum_{\ell=1}^L y_\ell h(x_\ell)\right)^\top .
\end{equation}
The population-trained attention matrix is
\begin{equation}
W_{d,L}^{\mathrm K}:=\arg\min_{W\in\mathbb R^{D\times D}}\mathbb E\left[\left(\operatorname{Tr}\!\left(W(H^{\mathrm K,L})^\top\right)-y_0\right)^2\right],
\label{eq:rigorous-kernel-population-training}
\end{equation}
where the expectation is over the prompt in
Eq.~\eqref{eq:rigorous-population-prompt}.  We define its population error by
\begin{equation}
\mathcal E_{d,L,\tilde L}^{\mathrm{kernel}}:=\mathbb E
\left[\left(y_0-\operatorname{Tr}\!\left(
W_{d,L}^{\mathrm K}(H^{\mathrm K,\tilde L})^\top\right)\right)^2\right].
\end{equation}

\paragraph{Feature learner.}
For the feature learner, define
\begin{equation}
H^{\mathrm F,L}:=\frac1Lx_0\left(
\sum_{\ell=1}^L y_\ell x_\ell
\right)^\top,
\qquad
U_W^L:=\operatorname{Tr}\!\left(W(H^{\mathrm F,L})^\top\right).
\end{equation}
We use the identity readout in Stage I, \(\sigma_0(u)=u\), and define the populational attention matrix by
\begin{equation}
W_{d,L}^{\mathrm F}:=\arg\min_{W\in\mathbb R^{d\times d}}\mathbb E\left[(U^L_W-y_0)^2\right].
\label{eq:rigorous-feature-population-stage1}
\end{equation}
In Stage II, the attention matrix is frozen and the scalar readout is optimized over all square-integrable measurable functions:
\begin{equation}
\widehat\sigma_{d,L}\in\arg\min_{\phi}
\mathbb E\left[\left(\phi(U^L_{W_{d,L}^{\mathrm F}})-y_0\right)^2\right].
\label{eq:rigorous-feature-population-stage2}
\end{equation}
We choose a version of the solution given by
\begin{equation}
\widehat\sigma_{d,L}(u):=\mathbb E\!\left[y_0\,\middle|\,U_{W_{d,L}^{\mathrm F}}^L=u\right]
\end{equation}
with \(\widehat\sigma_{d,L}(u)=0\) on the set where the conditional expectation is not specified.
The resulting population test error is
\begin{equation}
\mathcal E_{d,L,\tilde L}^{\mathrm{feature}}
:=\mathbb E\left[\left(y_0-\widehat\sigma_{d,L}
\!\left(U^{\tilde L}_{W_{d,L}^{\mathrm F}}\right)
\right)^2\right].
\label{eq:rigorous-feature-population-error}
\end{equation}
We now state the rigorous version of Theorem \ref{theo:comparison}.

\begin{theorem}
\label{thm:population_scaling}
Suppose Assumptions~\ref{ass:rigorous-teacher} and \ref{ass:rigorous-limit} hold.
Define $\alpha:=\frac{L}{d}$,
$\widetilde\alpha:=\frac{\widetilde L}{d}$, together with
$\rho_\alpha:=\frac{c_1}{\sqrt{c_1^2+1/\alpha}}$,
$r_\alpha:=\rho_\alpha^2$ and
$s_{\alpha,\widetilde\alpha}:=\frac{\rho_\alpha}{\rho_{\widetilde\alpha}}=\sqrt{
\frac{c_1^2+1/\widetilde\alpha}{c_1^2+1/\alpha}}$.
For the kernel learner, define
$\eta_{k,L}:=\frac{c_k^2L}{D_{k,d}+c_k^2L}$.
Then
\begin{equation}
\mathcal E_{d,L,\widetilde L}^{\mathrm{kernel}}
=\nu_{\rm out}+\sum_{k\in\mathcal K}
\left[c_k^2(1-\eta_{k,L})^2+\frac{D_{k,d}}{\widetilde L}\eta_{k,L}^2\right]+o_{d,L,\widetilde L}(1).
\label{eq:rigorous-kernel-population-formula}
\end{equation}
For the feature learner, for \(\rho\in[-1,1]\) define
\begin{equation}
\Phi_\rho(x):=\mathbb E_Z\left[\sigma_\star\left(\rho x+\sqrt{1-\rho^2}\,Z\right)\right],
\qquad
Z\sim\mathcal N(0,1).
\label{eq:proof-feature-gaussian-readout}
\end{equation}
For every \(\ell\geq0\), define
\begin{equation}
A_\ell(\alpha,\widetilde\alpha):=\mathbb E_T
\left[\Phi_{\rho_\alpha}\left(s_{\alpha,\widetilde\alpha}T
\right)H_\ell(T)\right],
\qquad
T\sim\mathcal N(0,1).
\label{eq:rigorous-feature-mismatch-coefficients}
\end{equation}
Then, with \(c_0:=0\),
\begin{equation}
\mathcal E_{d,L,\widetilde L}^{\mathrm{feature}}=1-\sum_{k\geq1}
c_k^2r_{\widetilde\alpha}^k+\sum_{\ell\geq0}\left[A_\ell(\alpha,\widetilde\alpha)-c_\ell\rho_{\widetilde\alpha}^\ell\right]^2+o_{d,L,\widetilde L}(1).
\label{eq:rigorous-feature-population-formula}
\end{equation}
\end{theorem}

\begin{remark}
\label{rem:explicit-feature-mismatch-coefficients}
If the teacher has a Hermite expansion
$\sigma_\star(x)=\sum_{k=1}^{\infty}c_kH_k(x)$,
then the coefficients in Eq.~\eqref{eq:rigorous-feature-mismatch-coefficients} can be evaluated explicitly as
\begin{equation}
A_\ell(\alpha,\widetilde\alpha)
=
\sum_{\substack{m\geq0}}
c_{\ell+2m}
\rho_\alpha^{\ell+2m}
\sqrt{\frac{(\ell+2m)!}{\ell!}}
\frac{s_{\alpha,\widetilde\alpha}^{\ell}}{m!}
\left(
\frac{s_{\alpha,\widetilde\alpha}^2-1}{2}
\right)^m,
\end{equation}
provided that the infinite sum converges absolutely.
In particular, when \(L=\widetilde L\), we have
\(s_{\alpha,\widetilde\alpha}=1\) and therefore
$A_\ell(\alpha,\alpha)=c_\ell\rho_\alpha^\ell$.
Consequently, the mismatch term in Eq.~\eqref{eq:rigorous-feature-population-formula} vanishes and
\begin{equation}
\mathcal E_{d,L,L}^{\mathrm{feature}}=1-\sum_{k\geq1}c_k^2r_\alpha^k+o_{d,L}(1).
\label{eq:rigorous-feature-equal-context-formula}
\end{equation}
\end{remark}

\section{Proof of Theorem \ref{thm:population_scaling}}
\label{app:population-proof}
\subsection{Gegenbauer polynomials}
We first recall some basic properties of the Gegenbauer polynomials. Define
\begin{equation}
Q_{k,d}(q):=
\frac{P_{k,d}(\sqrt d\,q)}{\sqrt{D_{k,d}}},
\qquad -1\leq q\leq1
\label{eq:proof-zonal-polynomial}
\end{equation}
for $k\geq0$, where we choose $D_{0,d}:=1$, $P_{0,d}(t)\equiv1$ and \(P_{1,d}(t)=t\).

\begin{lemma}
\label{lem:proof-spherical}
Let \(e,b\in\mathbb R^d\) be deterministic unit vectors, let
\(q=e^\top b\), and let
\(x\sim\operatorname{Unif}(\mathbb S^{d-1}(\sqrt d))\).
Set
\begin{equation}
g=e^\top x,
\qquad
s=b^\top x.
\end{equation}
Then the following statements hold.

\medskip
\noindent
\textbf{(i) Spherical projection formula.}
For every \(k\geq0\),
\begin{equation}
\mathbb E[P_{k,d}(g)\mid s]
=
Q_{k,d}(q)P_{k,d}(s),
\qquad
|Q_{k,d}(q)|\leq1.
\label{eq:proof-spherical-projection}
\end{equation}
\medskip
\noindent
\textbf{(ii) Gegenbauer limits.}
For every fixed \(k\),
\begin{equation}
\lim_{d\to\infty}
\sup_{|q|\leq1}
|Q_{k,d}(q)-q^k|
=0.
\label{eq:proof-zonal-limit}
\end{equation}
Moreover, for every fixed \(M\), there exists a constant
\(C_M<\infty\) such that, for all sufficiently large \(d\),
\begin{equation}
\max_{0\leq k\leq M}|P_{k,d}(t)|
\leq
C_M(1+|t|^M),
\qquad t\in\mathbb R,
\label{eq:proof-Gegenbauer-polynomial-growth}
\end{equation}
and for every fixed \(A<\infty\),
\begin{equation}
\sup_{d\ {\rm sufficiently\ large}}
\max_{0\leq k\leq M}
\sup_{|t|\leq A}
\left(
|P_{k,d}(t)|+|P_{k,d}'(t)|
\right)
<\infty.
\label{eq:proof-Gegenbauer-compact-C1}
\end{equation}

\medskip
\noindent
\textbf{(iii) Convergence of the teacher coefficients.}
Suppose Assumption~\ref{ass:rigorous-teacher} holds. Then
\begin{equation}
v_d
=
\mathbb E[\sigma_\star(G_d)^2]
\longrightarrow
\mathbb E[\sigma_\star(G)^2]
=1,
\label{eq:proof-vd-convergence}
\end{equation}
and, for every fixed \(k\geq0\),
\begin{equation}
c_{k,d}
=
\mathbb E[\sigma_\star(G_d)P_{k,d}(G_d)]
\longrightarrow
\mathbb E[\sigma_\star(G)H_k(G)]
=
c_k.
\label{eq:proof-ckd-convergence}
\end{equation}
Furthermore,
\begin{equation}
\lim_{d\to\infty}\sum_{k\geq0}
|c_{k,d}^2-c_k^2|=0,
\ \lim_{M\to\infty}
\limsup_{d\to\infty}
\sum_{k>M}c_{k,d}^2
=0.
\label{eq:proof-l1-spectrum-convergence}
\end{equation}
\end{lemma}

\begin{proof}
The identity \eqref{eq:proof-spherical-projection} is the standard formula for spherical harmonics, equivalently a consequence of the spherical addition theorem \citep{frye2012spherical}.
Since \(P_{k,d}(g)\) is normalized to have unit second moment and
\(P_{k,d}(s)\) has the same distribution,
\begin{equation}
Q_{k,d}(q)^2=\mathbb E\!\left[\mathbb E[P_{k,d}(g)\mid s]^2\right]\leq\mathbb E[P_{k,d}(g)^2]=1.
\end{equation}
This proves (i).

For every fixed \(k\), the normalized Gegenbauer polynomial converges,
coefficientwise, to the normalized Hermite polynomial under $d\to\infty$, while $Q_{k,d}$ converges uniformly on \([-1,1]\) to \(q^k\); see, e.g., \cite{mei2021learning}.
This proves \eqref{eq:proof-zonal-limit}.
Since the degree is fixed, coefficientwise convergence immediately implies
\eqref{eq:proof-Gegenbauer-polynomial-growth} and
\eqref{eq:proof-Gegenbauer-compact-C1}.

We next prove the statements involving the teacher.
The random variable \(G_d\) has density
\begin{equation}
p_d(t)=\frac{\Gamma(d/2)}
{\sqrt{\pi d}\,\Gamma((d-1)/2)}
\left(1-\frac{t^2}{d}\right)^{(d-3)/2}
\mathbf 1_{\{|t|\leq\sqrt d\}}.
\label{eq:proof-spherical-density}
\end{equation}
For every fixed \(t\),
\begin{equation}
p_d(t)\longrightarrow
\phi(t):=\frac1{\sqrt{2\pi}}e^{-t^2/2}.
\label{eq:proof-spherical-density-pointwise}
\end{equation}
Moreover, there exists a universal constant \(C<\infty\) such that,
for all sufficiently large \(d\),
\begin{equation}
p_d(t)\leq C\phi(t),
\qquad
t\in\mathbb R.
\label{eq:proof-spherical-density-domination}
\end{equation}
Indeed, the prefactor in \eqref{eq:proof-spherical-density} remains
uniformly bounded, and for \(|t|\leq\sqrt d\),
\begin{equation}
\left(1-\frac{t^2}{d}\right)^{(d-3)/2}\leq\exp\left(-\frac{d-3}{2d}t^2\right)\leq e^{3/2}e^{-t^2/2}.
\end{equation}

Assumption \ref{ass:rigorous-teacher}
together with \eqref{eq:proof-spherical-density-domination}, implies
that for every fixed $r,m<\infty$,
\begin{equation}
\sup_{d\ {\rm sufficiently\ large}}
\mathbb E\!\left[
|\sigma_\star(G_d)|^r(1+|G_d|^m)
\right]<\infty.
\label{eq:proof-teacher-uniform-moments}
\end{equation}
Since
$|\sigma_\star(t)|^2p_d(t)
\leq C(1+|t|^{2p})\phi(t)$,
dominated convergence gives
\begin{equation}
v_d\to\mathbb E[\sigma_\star(G)^2]=1,
\end{equation}
which proves \eqref{eq:proof-vd-convergence}.

Fix \(k\).
By \cite{mei2021learning},
\(P_{k,d}(t)\to H_k(t)\) pointwise, and for all sufficiently large \(d\),
\begin{equation}
|P_{k,d}(t)|
\leq
C_k(1+|t|^k).
\end{equation}
Together with
\eqref{eq:proof-spherical-density-domination}, this yields
\begin{equation}
|\sigma_\star(t)P_{k,d}(t)p_d(t)|\leq C_k|\sigma_\star(t)|(1+|t|^k)\phi(t).
\end{equation}
The right-hand side is integrable by Hölder's inequality and
Assumption~\ref{ass:rigorous-teacher}.
Dominated convergence therefore gives
\begin{equation}
c_{k,d}=\int\sigma_\star(t)P_{k,d}(t)p_d(t)\,\mathrm dt
\longrightarrow\int\sigma_\star(t)H_k(t)\phi(t)\,\mathrm dt=c_k,
\end{equation}
proving \eqref{eq:proof-ckd-convergence}.

Finally, Parseval's identity gives
\begin{equation}
\sum_{k\geq0}c_{k,d}^2=v_d,
\qquad
\sum_{k\geq0}c_k^2=1.
\label{eq:proof-parseval-two-measures}
\end{equation}
Thus, $\lim_{M\to\infty}\limsup_{d\to\infty}\sum_{k>M}c_{k,d}^2=0$.
Moreover,
\begin{equation}
\sum_{k\geq0}|c_{k,d}^2-c_k^2|\leq\sum_{k=0}^M|c_{k,d}^2-c_k^2|+\sum_{k>M}c_{k,d}^2+\sum_{k>M}c_k^2.
\end{equation}
First take \(d\to\infty\) at fixed \(M\), using
\eqref{eq:proof-ckd-convergence}, and then let \(M\to\infty\).
This proves \eqref{eq:proof-l1-spectrum-convergence}.
\end{proof}

From now on Assumption~\ref{ass:rigorous-teacher} is in force.
All limits are taken along an arbitrary deterministic sequence
$(d,L,\widetilde L)$ such that
$d,L,\widetilde L\to\infty$ and
$\frac{L}{\widetilde L}=O(1)$.
We work with sufficiently large $d$, so that $d\geq3$, $v_d>0$, and
$\mu_d:=\mathbb E[\sigma_\star(G_d)G_d]=c_{1,d}\neq0$.
These properties follow from \(v_d\to1\) and
\(c_{1,d}\to c_1\neq0\).
Constants denoted by \(C\) may depend on the teacher and its uniform moment bound, but not on \(d\) or \(L\).
For a random variable \(Z\), write \(\|Z\|_{L^r}:=(\mathbb E[|Z|^r])^{1/r}\).

\subsection{Kernel learner}

We first derive the population-trained kernel learner at finite
\(d,L,\widetilde L\), and then take the asymptotic limit in
Theorem~\ref{thm:population_scaling}.
For a task \(a\), define
\begin{equation}
m_d(a):=\mathbb E[yh(x)\mid a]=\bigoplus_{k\in\mathcal K}
\frac{c_{k,d}}{\sqrt{D_{k,d}}}h_k(a),
\label{eq:proof-kernel-conditional-mean}
\end{equation}
where the second equality follows from the spherical-harmonic expansion
of the teacher. Define also
\begin{equation}
M_d:=\mathbb E_a[m_d(a)m_d(a)^\top]=\bigoplus_{k\in\mathcal K}
\frac{c_{k,d}^2}{D_{k,d}}I_{D_{k,d}}.
\label{eq:proof-kernel-M}
\end{equation}
For an arbitrary context length \(n\geq1\), let
\begin{equation}
V_{d,n}^{\mathrm K}:=\frac1n\sum_{\ell=1}^n y_\ell h(x_\ell),
\qquad
A_{d,n}:=\frac{n-1}{n}M_d+\frac{v_d}{n}I_D.
\label{eq:proof-kernel-V-A}
\end{equation}
Then
$H^{\mathrm K,n}=h(x_0)(V_{d,n}^{\mathrm K})^\top$ and
$\operatorname{Tr}\!\left(W(H^{\mathrm K,n})^\top\right)=h(x_0)^\top W V_{d,n}^{\mathrm K}$.
We first compute the second moment of the context statistic.
Conditional on \(a\), the demonstrations are independent, and
\(\mathbb E[yh(x)\mid a]=m_d(a)\). Moreover, since for every fixed \(x\)
the random variable \(a^\top x/\sqrt d\) has the law of \(G_d\),
\begin{equation}
\mathbb E_{a,x}\!\left[y^2h(x)h(x)^\top\right]
=v_d\,\mathbb E_x[h(x)h(x)^\top]
=v_d I_D.
\label{eq:proof-kernel-single-example-second-moment}
\end{equation}
Therefore
\begin{equation}
\mathbb E\!\left[
V_{d,n}^{\mathrm K}(V_{d,n}^{\mathrm K})^\top
\right]=\frac{n-1}{n}M_d+\frac{v_d}{n}I_D=A_{d,n}.
\label{eq:proof-kernel-context-second-moment}
\end{equation}
Similarly, conditional independence of the query and the demonstrations
given \(a\) gives
\begin{equation}
\mathbb E\!\left[
y_0h(x_0)(V_{d,n}^{\mathrm K})^\top
\right]
=
\mathbb E_a[m_d(a)m_d(a)^\top]
=
M_d.
\label{eq:proof-kernel-cross-moment}
\end{equation}

We now train the attention matrix using contexts of length \(L\).
Using Eqs.~\eqref{eq:proof-kernel-context-second-moment} and
\eqref{eq:proof-kernel-cross-moment}, the population risk in
Eq.~\eqref{eq:rigorous-kernel-population-training} becomes
\begin{equation}
R_{\mathrm K,L}(W)
=
v_d
-2\operatorname{Tr}(WM_d)
+\operatorname{Tr}(WA_{d,L}W^\top).
\label{eq:proof-kernel-training-risk}
\end{equation}
Since
$A_{d,L}\succeq \frac{v_d}{L}I_D\succ0$,
the risk is strictly convex and its unique minimizer is
\begin{equation}
W_{d,L}^{\mathrm K}=M_dA_{d,L}^{-1}.
\label{eq:proof-kernel-exact-W}
\end{equation}
We next evaluate this learned matrix on inference contexts of length
\(\widetilde L\). Applying the same moment calculation with
\(n=\widetilde L\) gives the exact test error
\begin{equation}
\begin{aligned}
\mathcal E_{d,L,\widetilde L}^{\mathrm{kernel}}=v_d-2\operatorname{Tr}\!\left(W_{d,L}^{\mathrm K}M_d\right)+\operatorname{Tr}\!\left(
W_{d,L}^{\mathrm K}
A_{d,\widetilde L}
(W_{d,L}^{\mathrm K})^\top
\right).
\end{aligned}
\label{eq:proof-kernel-exact-test-error-trace}
\end{equation}
Both \(M_d\) and \(A_{d,n}\) are block diagonal in the harmonic
decomposition. 
On the degree-\(k\) block, \(M_d\) has eigenvalue
\(c_{k,d}^2/D_{k,d}\), while \(A_{d,L}\) has eigenvalue
$\frac{v_dD_{k,d}+(L-1)c_{k,d}^2}
{LD_{k,d}}$.
Hence \(W_{d,L}^{\mathrm K}\) acts on this block as multiplication by
$w_{k,d,L}:=\frac{Lc_{k,d}^2}
{v_dD_{k,d}+(L-1)c_{k,d}^2}$.
Substituting these eigenvalues into
Eq.~\eqref{eq:proof-kernel-exact-test-error-trace} yields
\begin{equation}
\mathcal E_{d,L,\widetilde L}^{\mathrm{kernel}}=v_d-\sum_{k\in\mathcal K}c_{k,d}^2+\sum_{k\in\mathcal K}\left[c_{k,d}^2(1-w_{k,d,L})^2+\frac{v_dD_{k,d}-c_{k,d}^2}{\widetilde L}
w_{k,d,L}^2
\right].
\label{eq:proof-kernel-exact-test-error}
\end{equation}
This identity is exact for every finite \(d,L,\widetilde L\).

It remains to replace the finite-dimensional Gegenbauer quantities in
Eq.~\eqref{eq:proof-kernel-exact-test-error} by their asymptotic limits.
Fix \(k\in\mathcal K\), and write
\begin{equation}
a_d:=c_{k,d}^2,
\qquad
a:=c_k^2,
\qquad
n:=D_{k,d},
\qquad
b_d:=v_d-\frac{a_d}{n}.
\label{eq:proof-kernel-asymptotic-notation}
\end{equation}
Since \(k\geq1\), we have \(n\to\infty\), and
Lemma~\ref{lem:proof-spherical} gives
$a_d\longrightarrow a$ and
$b_d\longrightarrow1$.
For all sufficiently large \(d\), \(b_d>0\). Define
\begin{equation}
t:=\frac{L}{n},
\qquad
x_d:=\frac{ta_d}{b_d},
\qquad
x:=ta,
\qquad
r:=\frac{L}{\widetilde L}.
\label{eq:proof-kernel-x-variables}
\end{equation}
Then
$w_{k,d,L}=\frac{x_d}{1+x_d}$,
and the degree-\(k\) contribution in
Eq.~\eqref{eq:proof-kernel-exact-test-error} can be written as
\begin{equation}
a_d\frac{1}{(1+x_d)^2}
+
r\,a_d\frac{x_d}{(1+x_d)^2}.
\label{eq:proof-kernel-exact-block-profile}
\end{equation}
We compare this with
\begin{equation}
a\frac{1}{(1+x)^2}
+
r\,a\frac{x}{(1+x)^2}.
\label{eq:proof-kernel-limit-block-profile}
\end{equation}
Suppose first that \(a>0\). Then
$x_d:=\lambda_dx$ with
$\lambda_d:=\frac{a_d}{ab_d}\longrightarrow1$.
Let
$p(x):=\frac1{(1+x)^2}$ and $q(x):=\frac{x}{(1+x)^2}$
for $x\geq0$. Since
$\sup_{x\geq0}|xp'(x)|<\infty$ and
$\sup_{x\geq0}|xq'(x)|<\infty$,
we have
\begin{equation}
\sup_{x\geq0}|p(\lambda_dx)-p(x)|
+
\sup_{x\geq0}|q(\lambda_dx)-q(x)|
\longrightarrow0.
\label{eq:proof-kernel-profile-uniformity}
\end{equation}
Together with \(a_d\to a\), this shows that the difference between
Eqs.~\eqref{eq:proof-kernel-exact-block-profile} and
\eqref{eq:proof-kernel-limit-block-profile} is \(o(1)\), uniformly over
the value of \(t=L/n\), provided \(r=L/\widetilde L=O(1)\).

If \(a=0\), then \(a_d\to0\), and using
$0\leq p(x)\leq1$,
$0\leq q(x)\leq\frac14$,
we instead obtain directly
\begin{equation}
0\leq
a_d p(x_d)+r\,a_dq(x_d)
\leq
a_d\left(1+\frac r4\right)
\longrightarrow0,
\label{eq:proof-kernel-zero-coefficient}
\end{equation}
again because \(r=O(1)\).

Since \(\mathcal K\) is fixed and finite, summing over
\(k\in\mathcal K\) gives
\begin{equation}
\begin{aligned}
\mathcal E_{d,L,\widetilde L}^{\mathrm{kernel}}
=v_d-\sum_{k\in\mathcal K}c_{k,d}^2+\sum_{k\in\mathcal K}
\left[c_k^2\frac1{(1+Lc_k^2/D_{k,d})^2}+\frac{L}{\widetilde L}c_k^2\frac{Lc_k^2/D_{k,d}}
{(1+Lc_k^2/D_{k,d})^2}\right]+o_{d,L,\widetilde L}(1).
\end{aligned}
\label{eq:proof-kernel-asymptotic-profile}
\end{equation}
By Lemma~\ref{lem:proof-spherical},
$v_d-\sum_{k\in\mathcal K}c_{k,d}^2
\longrightarrow
1-\sum_{k\in\mathcal K}c_k^2=\nu_{\rm out}$.
Finally, defining
$\eta_{k,L}:=\frac{c_k^2L}{D_{k,d}+c_k^2L}$,
we have
$c_k^2\frac1{(1+Lc_k^2/D_{k,d})^2}=c_k^2(1-\eta_{k,L})^2$
and
$\frac{L}{\widetilde L}
c_k^2
\frac{Lc_k^2/D_{k,d}}
{(1+Lc_k^2/D_{k,d})^2}=\frac{D_{k,d}}{\widetilde L}\eta_{k,L}^2$.
Therefore
\begin{equation}
\mathcal E_{d,L,\widetilde L}^{\mathrm{kernel}}=\nu_{\rm out}+\sum_{k\in\mathcal K}
\left[c_k^2(1-\eta_{k,L})^2+\frac{D_{k,d}}{\widetilde L}\eta_{k,L}^2
\right]+o_{d,L,\widetilde L}(1).
\label{eq:proof-kernel-final}
\end{equation}
This proves the kernel-learner part of
Theorem~\ref{thm:population_scaling}.

\subsection{Feature learner Stage-I}
We next derive the population-trained attention matrix of the feature learner. For an arbitrary context length \(n\geq1\), define
\begin{equation}
V_{d,n}
:=
\frac1n\sum_{\ell=1}^n y_\ell x_\ell,
\qquad
T_{d,n}
:=
x_0^\top V_{d,n}.
\label{eq:proof-feature-V-T}
\end{equation}
For Stage-I training we take \(n=L\), so that
$U_W^L=x_0^\top W V_{d,L}$.
Let
$e:=\frac{a}{\sqrt d}$,
so that \(e\) is a unit vector and
\(a^\top x/\sqrt d=e^\top x\). By rotational invariance,
\begin{equation}
\mathbb E[yx\mid a]=\mu_d e,
\qquad
\mu_d:=\mathbb E[\sigma_\star(G_d)G_d]=c_{1,d},
\label{eq:proof-feature-conditional-first-moment}
\end{equation}
and $\mathbb E_a[ee^\top]=\frac1dI_d$.
Moreover, averaging jointly over \(a\) and \(x\), rotational invariance
gives
\begin{equation}
\mathbb E[y^2xx^\top]
=
v_d I_d.
\label{eq:proof-feature-single-example-second-moment}
\end{equation}
Since the query and the demonstrations are conditionally independent
given \(a\),
\begin{equation}
\begin{aligned}
\mathbb E[y_0x_0V_{d,L}^\top]=
\mathbb E_a\!\left[
\mathbb E[y_0x_0\mid a]\,
\mathbb E[V_{d,L}\mid a]^\top
\right]=\mu_d^2\mathbb E_a[ee^\top]=\frac{\mu_d^2}{d}I_d.
\end{aligned}
\label{eq:proof-feature-cross-moment}
\end{equation}
Similarly, using conditional independence of the demonstrations,
\begin{equation}
\begin{aligned}
\mathbb E[V_{d,L}V_{d,L}^\top]=
\frac{L-1}{L}
\mathbb E_a\!\left[
\mathbb E[yx\mid a]
\mathbb E[yx\mid a]^\top
\right]+\frac1L\mathbb E[y^2xx^\top]
=
\left(\frac{L-1}{L}\frac{\mu_d^2}{d}+\frac{v_d}{L}\right)I_d.
\end{aligned}
\label{eq:proof-feature-context-second-moment}
\end{equation}
Since \(x_0\) is independent of \(V_{d,L}\) and
\(\mathbb E[x_0x_0^\top]=I_d\), the Stage-I population risk in
Eq.~\eqref{eq:rigorous-feature-population-stage1} is therefore
\begin{equation}
\begin{aligned}
R_{\mathrm F,L}(W)
:=\mathbb E[(U_W^L-y_0)^2]
=v_d-\frac{2\mu_d^2}{d}\operatorname{Tr}(W)+\left(\frac{L-1}{L}\frac{\mu_d^2}{d}+\frac{v_d}{L}\right)\|W\|_F^2.
\end{aligned}
\label{eq:proof-feature-population-risk}
\end{equation}
For all sufficiently large \(d\), we have \(v_d>0\) and \(\mu_d\neq0\). Hence the risk is strictly convex and has the unique
minimizer
\begin{equation}
W_{d,L}^{\mathrm F}=w_{d,L}I_d,
\qquad
w_{d,L}:=\frac{L\mu_d^2}
{dv_d+(L-1)\mu_d^2}.
\label{eq:proof-feature-exact-W}
\end{equation}
In particular, \(w_{d,L}>0\).
Consequently, when this learned attention matrix is evaluated on a
context of arbitrary length \(n\),
\begin{equation}
U_{W_{d,L}^{\mathrm F}}^n
=
w_{d,L}T_{d,n}.
\label{eq:proof-feature-attention-scalar-reduction}
\end{equation}
Thus Stage-II training at context length \(L\) is
equivalent to scalar regression of \(y_0\) on \(T_{d,L}\).
More precisely, for any chosen version of the Stage-II readout
\(\widehat\sigma_{d,L}\), define the rescaled scalar function
\begin{equation}
f_{d,L}(t)
:=
\widehat\sigma_{d,L}(w_{d,L}t).
\label{eq:proof-feature-rescaled-readout}
\end{equation}
Then its Stage-II training risk is
$\mathbb E\!\left[
\left(y_0-f_{d,L}(T_{d,L})\right)^2\right]$,
while evaluation on an inference prompt of length \(\widetilde L\) takes the form
\begin{equation}
\mathcal E_{d,L,\widetilde L}^{\mathrm{feature}}=\mathbb E\!\left[
\left(y_0-f_{d,L}(T_{d,\widetilde L})
\right)^2\right].
\label{eq:proof-feature-scalar-test-error}
\end{equation}
A canonical version of the population-optimal function \(f_{d,L}\),
together with its asymptotic behavior under the change from \(L\) to \(\widetilde L\), will be constructed below.

\subsection{Concentration of the context}
We next study the concentration of the context statistic \(T_{d,n}\) for an arbitrary context length \(n\geq1\). Define
\begin{equation}
B_{d,n}:=\left(\mu_d^2+\frac{dv_d}{n}\right)^{1/2},
\qquad
\rho_{d,n}:=\frac{\mu_d}{B_{d,n}}\in[-1,1],
\qquad
X_{d,n}:=\frac{T_{d,n}}{B_{d,n}}.
\label{eq:proof-B-rho-X}
\end{equation}
Recall that
$T_{d,n}=x_0^\top V_{d,n}$,
$V_{d,n}=\frac1n\sum_{\ell=1}^n y_\ell x_\ell$ and \(e:=a/\sqrt d\). Write
\begin{equation}
R_{d,n}:=\|V_{d,n}\|,
\qquad
b_{d,n}
:=
\begin{cases}
V_{d,n}/R_{d,n},&R_{d,n}>0,\\
e,&R_{d,n}=0,
\end{cases}
\label{eq:proof-context-radius-direction}
\end{equation}
together with
\begin{equation}
q_{d,n}:=e^\top b_{d,n},
\qquad
S_{d,n}:=b_{d,n}^\top x_0.
\label{eq:proof-context-q-S}
\end{equation}
Then
\begin{equation}
T_{d,n}=R_{d,n}S_{d,n},
\qquad
X_{d,n}
=
\frac{R_{d,n}}{B_{d,n}}S_{d,n}.
\label{eq:proof-context-X-representation}
\end{equation}

\begin{lemma}
\label{lem:proof-context-concentration}
Suppose Assumption~\ref{ass:rigorous-teacher} holds.
For every fixed integer \(r\geq2\), there exists a constant \(C_r<\infty\), depending only on \(r\) and the teacher, such that, for all sufficiently large \(d\) and every \(n\geq1\),
\begin{equation}
\mathbb E\left[\left|\frac{R_{d,n}}{B_{d,n}}-1\right|^r\right]+
\mathbb E\left[|q_{d,n}-\rho_{d,n}|^r
\right]\leq
C_r\left(n^{-r/2}+d^{-r/2}\right).
\label{eq:proof-context-high-moment-bound}
\end{equation}
\end{lemma}

\begin{proof}
Fix an integer \(r\geq2\).
All estimates below are uniform in the task direction \(a\).
We first condition on \(a\), and remove the conditioning at the end.

Recall that
$e:=\frac{a}{\sqrt d}$,
$V_{d,n}=\frac1n\sum_{\ell=1}^n y_\ell x_\ell$,
$\mathbb E[yx\mid a]=\mu_de$,
where
$\mu_d:=\mathbb E[\sigma_\star(G_d)G_d]$.
Define
\begin{equation}
\eta_\ell:=y_\ell x_\ell-\mu_de,
\qquad
\bar\eta_n:=\frac1n
\sum_{\ell=1}^n\eta_\ell.
\label{eq:proof-context-eta}
\end{equation}
Then, conditionally on \(a\), the vectors
\(\eta_1,\ldots,\eta_n\) are independent and centered, and
\begin{equation}
V_{d,n}=\mu_de+\bar\eta_n.
\label{eq:proof-context-V-decomposition}
\end{equation}
We first record the moment bounds that will be used repeatedly.

\paragraph{Step 1: uniform moment bounds.}
Since \(\|x\|=\sqrt d\),
$\|\eta_\ell\|\leq
\sqrt d\,|y_\ell|+|\mu_d|$. Assumption~\ref{ass:rigorous-teacher} and the uniform moment bound for \(G_d\) imply that, for every fixed
\(m<\infty\),
$\sup_d\mathbb E[|\sigma_\star(G_d)|^m]<\infty$.
Moreover, \(\mu_d\to c_1\), so \(\sup_d|\mu_d|<\infty\).
It follows that, for every fixed \(m<\infty\),
\begin{equation}
\mathbb E[\|\eta_\ell\|^m\mid a]\leq C_m d^{m/2}.
\label{eq:proof-context-vector-moment}
\end{equation}
We also require moments of one-dimensional projections of
\(\eta_\ell\).
Let \(u\in\mathbb R^d\) be a deterministic unit vector.
Then
$|u^\top\eta_\ell|\leq|y_\ell|\,|u^\top x_\ell|+|\mu_d|$.
By Hölder's inequality,
\begin{equation}
\mathbb E[|y_\ell|^m|u^\top x_\ell|^m\mid a]\leq\left(\mathbb E[|y_\ell|^{2m}\mid a]\right)^{1/2}\left(\mathbb E[|u^\top x_\ell|^{2m}]\right)^{1/2}.
\end{equation}
The first factor is uniformly bounded by the polynomial-growth of \(\sigma_\star\) (Assumption \ref{ass:rigorous-teacher}), while the second factor is uniformly bounded because \(u^\top x_\ell\) has the same marginal distribution as \(G_d\). Hence
\begin{equation}
\sup_{\|u\|=1}\mathbb E[|u^\top\eta_\ell|^m\mid a]\leq
C_m.
\label{eq:proof-context-directional-moment}
\end{equation}
By rotational invariance around the direction \(e\), the conditional covariance of $\eta$ is
\begin{equation}
\Gamma_d(a):=\mathbb E[\eta_1\eta_1^\top\mid a]=(t_d-\mu_d^2)ee^\top+\beta_d(I_d-ee^\top).
\label{eq:proof-context-covariance}
\end{equation}
where
$t_d=\mathbb E[\sigma_\star(G_d)^2G_d^2]$ and
$\beta_d=\frac{dv_d-t_d}{d-1}$. The same uniform moment bounds imply
$\sup_d t_d<\infty$,
$\sup_d v_d<\infty$,
and therefore
\begin{equation}
\|\Gamma_d\|_{\mathrm{op}}\leq C,
\qquad
\operatorname{Tr}\Gamma_d=dv_d-\mu_d^2
\leq Cd.
\label{eq:proof-context-covariance-bounds}
\end{equation}

Set
$U_n:=e^\top\bar\eta_n=\frac1n\sum_{\ell=1}^ne^\top\eta_\ell$.
The random variables \(e^\top\eta_\ell\) are conditionally independent, centered, and by
Eq.~\eqref{eq:proof-context-directional-moment} have uniformly bounded \(r\)-th moments.
We use the Rosenthal inequality: for independent centered
random variables \(\xi_1,\ldots,\xi_n\) and \(r\geq2\),
\begin{equation}
\left\|\sum_{\ell=1}^n\xi_\ell\right\|_{L^r}\leq
C_r\left[\left(\sum_{\ell=1}^n\mathbb E[\xi_\ell^2]\right)^{1/2}+\left(\sum_{\ell=1}^n\mathbb E[|\xi_\ell|^r]\right)^{1/r}\right].
\label{eq:proof-rosenthal-scalar}
\end{equation}
Applying this conditionally on \(a\) with
\(\xi_\ell=e^\top\eta_\ell\) gives
\begin{equation}
\|U_n\|_{L^r}\leq\frac{C_r}{n}\left(\sqrt n+n^{1/r}\right)\leq\frac{C_r}{\sqrt n},
\end{equation}
where we used \(r\geq2\) and \(n\geq1\). Thus
\begin{equation}
\|e^\top\bar\eta_n\|_{L^r}\leq
C_r n^{-1/2}.
\label{eq:proof-context-parallel-high-moment}
\end{equation}

We will also need a bound for the partial sum $S_k^\eta:=\sum_{\ell=1}^k\eta_\ell$.
The Rosenthal inequality gives
\begin{equation}
\left\|\sum_{\ell=1}^k\eta_\ell\right\|_{L^r}
\leq C_r\left[\left(\sum_{\ell=1}^k\mathbb E\|\eta_\ell\|^2\right)^{1/2}+\left(\sum_{\ell=1}^k\mathbb E\|\eta_\ell\|^r\right)^{1/r}\right].
\label{eq:proof-rosenthal-hilbert}
\end{equation}
By Eqs.~\eqref{eq:proof-context-vector-moment} and
\eqref{eq:proof-context-covariance-bounds},
$\mathbb E\|\eta_\ell\|^2=\operatorname{Tr}\Gamma_d
\leq Cd$ and $\mathbb E\|\eta_\ell\|^r\leq
C_rd^{r/2}$. Hence
\begin{equation}
\|S_k^\eta\|_{L^r}\leq C_r\left(\sqrt{kd}+k^{1/r}\sqrt d
\right)\leq C_r\sqrt{kd},
\end{equation}
because \(k^{1/r}\leq\sqrt k\) for \(r\geq2\).
Therefore
\begin{equation}
\left\|\sum_{\ell=1}^k\eta_\ell\right\|_{L^r}\leq C_r\sqrt{kd}.
\label{eq:proof-context-partial-vector-sum}
\end{equation}

\paragraph{Step 2: concentration of $\|\bar\eta_n\|^2-\mathbb E[\|\bar\eta_n\|^2\mid a]$.}
Expanding the square,
\begin{equation}
\|\bar\eta_n\|^2-\mathbb E[\|\bar\eta_n\|^2\mid a]=D_n+C_n, 
\end{equation}
where $D_n:=\frac1{n^2}\sum_{\ell=1}^n
\left(\|\eta_\ell\|^2-\mathbb E[\|\eta_\ell\|^2\mid a]\right)$ and $C_n:=\frac2{n^2}\sum_{1\leq i<j\leq n}
\eta_i^\top\eta_j$.

We first deal with $D_n$.
Define
$\xi_\ell:=\|\eta_\ell\|^2-\mathbb E[\|\eta_\ell\|^2
\mid a]$.
Then the \(\xi_\ell\) are conditionally independent and centered.
By Eq.~\eqref{eq:proof-context-vector-moment},
$\mathbb E[|\xi_\ell|^r\mid a]\leq
C_r d^r$ and
$\mathbb E[\xi_\ell^2\mid a]\leq
C d^2$.
Applying the Rosenthal inequality \eqref{eq:proof-rosenthal-scalar},
\begin{equation}
\|D_n\|_{L^r}\leq\frac{C_r}{n^2}\left[(nd^2)^{1/2}+(nd^r)^{1/r}\right]=C_r\left(\frac{d}{n^{3/2}}+\frac{d}{n^{2-1/r}}\right).
\label{eq:proof-context-diagonal-bound}
\end{equation}

Now we deal with $C_n$.
Set $M_j:=\left(\sum_{i=1}^{j-1}\eta_i\right)^\top\eta_j$ for $j=2,\ldots,n$.
Then
$C_n=\frac2{n^2}\sum_{j=2}^nM_j$.
Let
$\mathcal F_j:=\sigma(a,\eta_1,\ldots,\eta_j)$.
Because \(\eta_j\) is independent of
\(\mathcal F_{j-1}\) conditional on \(a\), and
\(\mathbb E[\eta_j\mid a]=0\),
$\mathbb E[M_j\mid\mathcal F_{j-1}]=0$.
Thus \((M_j)_{j=2}^n\) is a martingale-difference sequence.

For \(r\geq2\), the martingale Rosenthal inequality gives
\begin{equation}
\left\|
\sum_{j=2}^nM_j
\right\|_{L^r}
\leq
C_r\left[\left\|\left(\sum_{j=2}^n
\mathbb E[M_j^2\mid\mathcal F_{j-1}]\right)^{1/2}
\right\|_{L^r}+\left(\sum_{j=2}^n\mathbb E|M_j|^r\right)^{1/r}\right].
\label{eq:proof-martingale-rosenthal}
\end{equation}
We now bound the two terms on the right-hand side.
Let
$S_{j-1}^\eta:=\sum_{i=1}^{j-1}\eta_i$.
Using Eq. \eqref{eq:proof-context-covariance},
\begin{equation}
\mathbb E[M_j^2\mid\mathcal F_{j-1}]=(S_{j-1}^\eta)^\top\Gamma_dS_{j-1}^\eta\leq\|\Gamma_d\|_{\mathrm{op}}\|S_{j-1}^\eta\|^2\leq
C\|S_{j-1}^\eta\|^2.
\end{equation}
Therefore
\begin{equation}
\begin{aligned}
\left\|\left(\sum_{j=2}^n\mathbb E[M_j^2\mid\mathcal F_{j-1}]
\right)^{1/2}\right\|_{L^r}^2&=
\left\|\sum_{j=2}^n
\mathbb E[M_j^2\mid\mathcal F_{j-1}]\right\|_{L^{r/2}}
\leq
C\left\|\sum_{j=2}^n\|S_{j-1}^\eta\|^2\right\|_{L^{r/2}}
\\&\leq
C\sum_{j=2}^n\left\|\|S_{j-1}^\eta\|^2\right\|_{L^{r/2}}=C\sum_{j=2}^n\|S_{j-1}^\eta\|_{L^r}^2.
\end{aligned}
\end{equation}
Using
Eq.~\eqref{eq:proof-context-partial-vector-sum},
$\|S_{j-1}^\eta\|_{L^r}^2\leq
C_r(j-1)d$.
Hence
\begin{equation}
\left\|\left(\sum_{j=2}^n\mathbb E[M_j^2\mid\mathcal F_{j-1}]\right)^{1/2}\right\|_{L^r}^2\leq
C_rd\sum_{j=2}^n(j-1)\leq
C_rdn^2,
\end{equation}
and therefore
\begin{equation}
\left\|\left(\sum_{j=2}^n\mathbb E[M_j^2\mid\mathcal F_{j-1}]\right)^{1/2}\right\|_{L^r}\leq
C_rn\sqrt d.
\label{eq:proof-context-martingale-variance}
\end{equation}

For the second term in Eq.~\eqref{eq:proof-martingale-rosenthal}, condition on \(\mathcal F_{j-1}\).
If \(S_{j-1}^\eta\neq0\), then
Eq.~\eqref{eq:proof-context-directional-moment} gives
\begin{equation}
\mathbb E[|M_j|^r\mid\mathcal F_{j-1}]=\|S_{j-1}^\eta\|^r
\mathbb E\left[\left|\frac{(S_{j-1}^\eta)^\top}
{\|S_{j-1}^\eta\|}\eta_j\right|^r\middle|\mathcal F_{j-1}\right]\leq
C_r\|S_{j-1}^\eta\|^r.
\end{equation}
The same inequality is trivial when
\(S_{j-1}^\eta=0\).
Taking expectation,
\begin{equation}
\|M_j\|_{L^r}\leq
C_r\|S_{j-1}^\eta\|_{L^r}\leq
C_r\sqrt{jd}.
\end{equation}
Consequently,
\begin{equation}
\left(\sum_{j=2}^n\mathbb E|M_j|^r\right)^{1/r}\leq
C_r\sqrt d\left(\sum_{j=2}^nj^{r/2}\right)^{1/r}\leq
C_r\sqrt d\,
n^{1/2+1/r}\leq
C_rn\sqrt d,
\end{equation}
because \(1/2+1/r\leq1\) for \(r\geq2\).
Combining this estimate with
Eq.~\eqref{eq:proof-context-martingale-variance} in
Eq.~\eqref{eq:proof-martingale-rosenthal} gives
$\left\|\sum_{j=2}^nM_j\right\|_{L^r}\leq
C_rn\sqrt d$.
Hence
\begin{equation}
\|C_n\|_{L^r}
\leq
C_r\frac{\sqrt d}{n}.
\label{eq:proof-context-offdiagonal-bound}
\end{equation}
Combining Eqs.~\eqref{eq:proof-context-diagonal-bound} and
\eqref{eq:proof-context-offdiagonal-bound}, we obtain
\begin{equation}
\left\|\|\bar\eta_n\|^2-\mathbb E[\|\bar\eta_n\|^2\mid a]
\right\|_{L^r}\leq
C_r\left(\frac{\sqrt d}{n}+\frac{d}{n^{3/2}}+\frac{d}{n^{2-1/r}}\right).
\label{eq:proof-context-noise-norm-high-moment}
\end{equation}

\paragraph{Step 3: concentration of the context radius.}
Using Eq.~\eqref{eq:proof-context-V-decomposition},
\begin{equation}
R_{d,n}^2=\|V_{d,n}\|^2=\mu_d^2+2\mu_de^\top\bar\eta_n+\|\bar\eta_n\|^2.
\label{eq:proof-context-radius-expansion}
\end{equation}
Furthermore,
\begin{equation}
\mathbb E[\|\bar\eta_n\|^2\mid a]=\frac1n\operatorname{Tr}\Gamma_d=
\frac{dv_d-\mu_d^2}{n}.
\end{equation}
Since $B_{d,n}^2=\mu_d^2+\frac{dv_d}{n}$,
we have the identity
\begin{equation}
R_{d,n}^2-B_{d,n}^2=2\mu_de^\top\bar\eta_n+\left(\|\bar\eta_n\|^2-\mathbb E[\|\bar\eta_n\|^2\mid a]\right)-\frac{\mu_d^2}{n}.
\label{eq:proof-context-radius-difference}
\end{equation}
Since \(v_d\to1\) and
\(\mu_d\to c_1\neq0\), there exists \(c>0\) such that, for all
sufficiently large \(d\) and all \(n\geq1\),
\begin{equation}
B_{d,n}^2=\mu_d^2+\frac{dv_d}{n}
\geq
c\left(1+\frac dn\right).
\label{eq:proof-B-lower-bound}
\end{equation}
Using Eqs.~\eqref{eq:proof-context-parallel-high-moment},
\eqref{eq:proof-context-noise-norm-high-moment}, and
\eqref{eq:proof-context-radius-difference}, we get
\begin{equation}
\left\|\frac{R_{d,n}^2}{B_{d,n}^2}-1
\right\|_{L^r}\leq C_r\Bigg[\frac{n^{-1/2}}{1+d/n}+\frac{\sqrt d/n}{1+d/n}+\frac{d/n^{3/2}}{1+d/n}+\frac{d/n^{2-1/r}}{1+d/n}+\frac{n^{-1}}{1+d/n}
\Bigg].
\end{equation}
We bound the five terms separately:
\begin{equation}
\begin{aligned}
&\frac{n^{-1/2}}{1+d/n}\leq
n^{-1/2},\\    
&\frac{\sqrt d/n}{1+d/n}=\frac{\sqrt d}{n+d}\leq
d^{-1/2},\\
&\frac{d/n^{3/2}}{1+d/n}=\frac{d}{\sqrt n(n+d)}\leq
n^{-1/2},\\
&\frac{d/n^{2-1/r}}{1+d/n}=\frac{d}{n+d}\,
n^{-1+1/r}\leq
n^{-1+1/r}\leq
n^{-1/2},\\
&\frac{n^{-1}}{1+d/n}\leq
n^{-1}\leq
n^{-1/2},
\end{aligned}
\end{equation}
where we uses \(r\geq2\).
Therefore
\begin{equation}
\left\|\frac{R_{d,n}^2}{B_{d,n}^2}-1\right\|_{L^r}\leq
C_r\left(n^{-1/2}+d^{-1/2}\right).
\label{eq:proof-context-squared-radius-high-moment}
\end{equation}
Since \(R_{d,n}/B_{d,n}\geq0\) and
$|\sqrt u-1|=\frac{|u-1|}{\sqrt u+1}\leq|u-1|$ for $u\geq0$,
Eq.~\eqref{eq:proof-context-squared-radius-high-moment} implies
\begin{equation}
\left\|\frac{R_{d,n}}{B_{d,n}}-1\right\|_{L^r}\leq
C_r\left(n^{-1/2}+d^{-1/2}\right).
\label{eq:proof-context-radius-final}
\end{equation}

\paragraph{Step 4: concentration of the context direction.}

Set $A_{d,n}:=\frac{R_{d,n}}{B_{d,n}}$.
Consider the event
$\mathcal G:=\{A_{d,n}\geq1/2\}$.
On \(\mathcal G\), we have \(R_{d,n}>0\), and hence
\begin{equation}
q_{d,n}=\frac{e^\top V_{d,n}}{R_{d,n}}=\frac{\mu_d+e^\top\bar\eta_n}
{A_{d,n}B_{d,n}}.
\end{equation}
Since $\rho_{d,n}=\frac{\mu_d}{B_{d,n}}$,
we obtain
\begin{equation}
q_{d,n}-\rho_{d,n}=\frac{e^\top\bar\eta_n}
{A_{d,n}B_{d,n}}+\frac{\mu_d}{B_{d,n}}\left(\frac1{A_{d,n}}-1\right).
\end{equation}
On \(\mathcal G\),
\begin{equation}
A_{d,n}^{-1}\leq2,\qquad\left|A_{d,n}^{-1}-1
\right|=\frac{|A_{d,n}-1|}{A_{d,n}}\leq2|A_{d,n}-1|.
\end{equation}
Moreover, Eq.~\eqref{eq:proof-B-lower-bound} gives
\(B_{d,n}\geq c>0\), while \(\mu_d\) remains uniformly bounded.
Thus
\begin{equation}
|q_{d,n}-\rho_{d,n}|\mathbf1_{\mathcal G}\leq
C\left(|e^\top\bar\eta_n|+|A_{d,n}-1|\right).
\label{eq:proof-context-q-good-event}
\end{equation}
By Eqs.~\eqref{eq:proof-context-parallel-high-moment} and
\eqref{eq:proof-context-radius-final},
\begin{equation}
\left\|(q_{d,n}-\rho_{d,n})\mathbf1_{\mathcal G}\right\|_{L^r}\leq
C_r\left(n^{-1/2}+d^{-1/2}\right).
\label{eq:proof-context-q-good-bound}
\end{equation}
On the complementary event \(\mathcal G^c\), both
\(|q_{d,n}|\leq1\) and \(|\rho_{d,n}|\leq1\).
The latter follows from \(\mu_d^2\leq v_d\) by Cauchy--Schwarz.
Hence
$|q_{d,n}-\rho_{d,n}|\leq2$.
Furthermore,
$\mathcal G^c=\{A_{d,n}<1/2\}\subseteq\{|A_{d,n}-1|\geq1/2\}$.
Therefore, by Markov's inequality,
\begin{equation}
\left\|(q_{d,n}-\rho_{d,n})\mathbf1_{\mathcal G^c}
\right\|_{L^r}\leq2\,\mathbb P(\mathcal G^c)^{1/r}
\leq4\|A_{d,n}-1\|_{L^r}\leq
C_r\left(n^{-1/2}+d^{-1/2}\right),
\end{equation}
where we used Eq.~\eqref{eq:proof-context-radius-final} in the last inequality. Combining this with Eq.~\eqref{eq:proof-context-q-good-bound} gives
\begin{equation}
\|q_{d,n}-\rho_{d,n}\|_{L^r}\leq
C_r\left(n^{-1/2}+d^{-1/2}\right).
\label{eq:proof-context-q-final}
\end{equation}
Together with Eq.~\eqref{eq:proof-context-radius-final}, this proves
\begin{equation}
\left\|\frac{R_{d,n}}{B_{d,n}}-1\right\|_{L^r}+\left\|q_{d,n}-\rho_{d,n}
\right\|_{L^r}\leq
C_r\left(\frac1{\sqrt n}+\frac1{\sqrt d}\right).
\label{eq:proof-relative-concentration}
\end{equation}
Raising both sides to the \(r\)-th power and using
$(x+y)^r\leq2^{r-1}(x^r+y^r)$
gives Eq.~\eqref{eq:proof-context-high-moment-bound}.
\end{proof}

Let
$\mathcal C_{d,n}:=\sigma(a,x_1,\ldots,x_n)$ be the sigma-algebra generated by the task and the demonstrations.
The variables \(R_{d,n},b_{d,n},q_{d,n}\) are
\(\mathcal C_{d,n}\)-measurable, whereas \(x_0\) is independent of
\(\mathcal C_{d,n}\). Consequently, conditionally on
\(\mathcal C_{d,n}\), the random variable \(S_{d,n}\) has the same law
as \(G_d\). In particular,
\begin{equation}
S_{d,n}\ \text{is independent of }\mathcal C_{d,n},
\qquad
\mathbb E[S_{d,n}^2\mid\mathcal C_{d,n}]=1.
\label{eq:proof-context-S-independence}
\end{equation}
Since $X_{d,n}=A_{d,n}S_{d,n}$, we have
\begin{equation}
\mathbb E[(X_{d,n}-S_{d,n})^2]=\mathbb E[(A_{d,n}-1)^2S_{d,n}^2]=\mathbb E[(A_{d,n}-1)^2]\mathbb E[S_{d,n}^2]=\mathbb E[(A_{d,n}-1)^2].
\end{equation}
Taking \(r=2\) in
Eq.~\eqref{eq:proof-context-radius-final} proves
\begin{equation}
\mathbb E\left[(X_{d,n}-S_{d,n})^2\right]
=\mathbb E\left[\left(\frac{R_{d,n}}{B_{d,n}}-1\right)^2\right]\longrightarrow0.
\label{eq:proof-normalized-score}
\end{equation}

Moreover, Lemma \ref{lem:proof-context-concentration} has the following important consequence. Let $\delta_d:=d^{-\beta}$ with $0<\beta<\frac14\min\{\varepsilon,1\}$. When \(n\geq d^\varepsilon\), Lemma
\ref{lem:proof-context-concentration} and Markov's inequality imply that, for every \(M<\infty\),
\begin{equation}
\mathbb P(|A_{d,n}-1|>\delta_d)\leq
C_Md^{-M}
\end{equation}
after choosing a sufficiently large $r$ in Eq.~\eqref{eq:proof-context-high-moment-bound}.

\subsection{Feature learner Stage-II}
\label{app:feature-stage2}

We now analyze the population-optimal scalar readout learned with context length \(L\) and evaluated with a possibly different context length \(\widetilde L\).
Throughout this subsection, Assumptions
\ref{ass:rigorous-teacher} and \ref{ass:rigorous-limit} are in force.

Recall from Eq.~\eqref{eq:proof-feature-attention-scalar-reduction} that $U_{W_{d,L}^{\mathrm F}}^n=w_{d,L}T_{d,n}$ with $w_{d,L}>0$ for every context length \(n\). Recall also the definitions
\begin{equation}
B_{d,n}:=\left(\mu_d^2+\frac{dv_d}{n}\right)^{1/2},
\qquad
\rho_{d,n}:=\frac{\mu_d}{B_{d,n}},
\qquad
X_{d,n}:=\frac{T_{d,n}}{B_{d,n}}.
\label{eq:proof-stage2-normalization}
\end{equation}
Write $A_{d,n}:=\frac{R_{d,n}}{B_{d,n}}$, so that, by Eq.~\eqref{eq:proof-context-X-representation}, $X_{d,n}=A_{d,n}S_{d,n}$. Moreover, \(S_{d,n}\) is independent of the demonstration and has the same law as \(G_d\).

By the definition of \(\widehat\sigma_{d,L}\), we can define
$g_{d,L}(X_{d,L})=\mathbb E[y_0\mid X_{d,L}]$.
Hence, writing
$s_{d,L,\widetilde L}:=\frac{B_{d,\widetilde L}}{B_{d,L}}$,
the test error is
\begin{equation}
\mathcal E_{d,L,\widetilde L}^{\mathrm{feature}}=\mathbb E\left[\left(y_0-g_{d,L}\left(s_{d,L,\widetilde L}X_{d,\widetilde L}\right)\right)^2\right].
\label{eq:proof-feature-stage2-test-error}
\end{equation}
Moreover, since
$|y_0|=\left|\sigma_\star(G_d)\right|\leq\sup_{|t|\leq\sqrt d}|\sigma_\star(t)|$,
the conditional expectation satisfies
\begin{equation}
|g_{d,L}(x)|\leq
C(1+d^{p/2}),\qquad x\in\mathbb R,
\label{eq:proof-stage2-regression-bound}
\end{equation}
where we use Assumption~\ref{ass:rigorous-teacher}.

\subsubsection{Asymptotical parameters}
We first relate the quantities in Eq.~\eqref{eq:proof-stage2-normalization} to the limiting parameters appearing in the theorem.

\begin{lemma}
\label{lem:proof-stage2-parameters}
Let $\alpha=\frac Ld$ and $\widetilde\alpha=\frac{\widetilde L}{d}$.
Then, along every sequence satisfying Assumption \ref{ass:rigorous-limit},
\begin{equation}
\rho_{d,L}-\rho_\alpha\longrightarrow0,
\qquad
\rho_{d,\widetilde L}-\rho_{\widetilde\alpha}\longrightarrow0,
\qquad
s_{d,L,\widetilde L}
-
s_{\alpha,\widetilde\alpha}
\longrightarrow0,
\label{eq:proof-stage2-parameter-convergence}
\end{equation}
where $\rho_\beta=\frac{c_1}{\sqrt{c_1^2+1/\beta}}$ and $s_{\alpha,\widetilde\alpha}=\sqrt{\frac{c_1^2+1/\widetilde\alpha}{c_1^2+1/\alpha}}$.
\end{lemma}

\begin{proof}
Set $b_d:=\frac{\mu_d^2}{v_d}$.
Lemma~\ref{lem:proof-spherical} gives
\(b_d\to c_1^2>0\), and \(\mu_d\) has the same sign as \(c_1\)
for all sufficiently large \(d\).  Hence, we can write $\rho_{d,n}^2=\frac{\alpha_n b_d}{1+\alpha_n b_d}$, where
$\alpha_n:=\frac nd$.
For \(b\) bounded away from zero,
\begin{equation}
\sup_{\alpha\geq0}\left|\frac{\partial}{\partial b}
\frac{\alpha b}{1+\alpha b}\right|=\sup_{\alpha\geq0}\frac{\alpha}{(1+\alpha b)^2}=\frac1{4b}.
\end{equation}
Thus the convergence of \(\rho_{d,n}\) to \(\rho_{\alpha_n}\) is uniform over \(\alpha_n>0\).

Similarly, $s_{d,L,\widetilde L}^2=\frac{b_d+1/\widetilde\alpha}
{b_d+1/\alpha}$.
Writing \(r:=L/\widetilde L=\alpha/\widetilde\alpha\),
we have
$s_{d,L,\widetilde L}^2=\frac{\alpha b_d+r}{\alpha b_d+1}$. For \(F_r(b):=(\alpha b+r)/(\alpha b+1)\),
\begin{equation}
\left|\frac{\partial F_r}{\partial b}\right|=\frac{\alpha|1-r|}{(\alpha b+1)^2}
\leq\frac{|1-r|}{4b}.
\end{equation}
Assumption~\ref{ass:rigorous-limit} keeps \(r\) in a fixed compact subset of \((0,\infty)\), proving the last convergence.
\end{proof}

\subsubsection{A spherical-to-Gaussian approximation}

The following lemma is the main technical input for Stage II, showing that $A_dS_d$ can be approximated by a Gaussian variable $Z$.
It will be used both for the training and the inference distributions. The lemma is a weighted and randomly rescaled version of the classical Gaussian limit for high-dimensional spheres, combined with the Gegenbauer-to-Hermite limit \citep{mei2021learning}. We give a direct proof because we additionally require dimension-dependent test functions of uniform polynomial growth and the random rescaling \(A_d\).

\begin{lemma}
\label{lem:proof-stage2-weighted-gaussian}
Let \(n=n(d)\geq d^\varepsilon\), and write
$A_d:=A_{d,n}$, $S_d:=S_{d,n}$. Let \(f_d:\mathbb R\to\mathbb R\) be deterministic functions satisfying, for some fixed \(m,C<\infty\),
\begin{equation}
|f_d(x)|\leq C(1+|x|^m),
\qquad d\geq1,\ x\in\mathbb R.
\label{eq:proof-stage2-polynomial-family}
\end{equation}
For \(k\geq0\), define
$J_{k,d}^{f_d}(a):=\mathbb E_{S\sim G_d}
\left[P_{k,d}(S)f_d(aS)\right]$ and $J_k^{f_d}:=\mathbb E_{Z\sim\mathcal N(0,1)}\left[H_k(Z)f_d(Z)\right]$.
Then, for every fixed \(k\geq0\),
\begin{equation}
\mathbb E\left[
\left|
J_{k,d}^{f_d}(A_d)-J_k^{f_d}
\right|
\right]
\longrightarrow0.
\label{eq:proof-stage2-weighted-gaussian-L1}
\end{equation}
\end{lemma}

\begin{proof}
Fix \(k\geq0\). Choose $0<\beta<\frac14\min\{\varepsilon,1\}$, $\delta_d:=d^{-\beta}$, and let
$\mathcal G_d:=\{|A_d-1|\leq\delta_d\}$.
By Lemma~\ref{lem:proof-context-concentration}, for every \(M<\infty\), after choosing a sufficiently high moment in
Eq.~\eqref{eq:proof-context-high-moment-bound},
\begin{equation}
\mathbb P(\mathcal G_d^c)
\leq C_Md^{-M}.
\label{eq:proof-stage2-good-A}
\end{equation}

We first work on the good event $\mathcal G_d$. For a deterministic
\(a>0\), since \(S\sim G_d\) has density \(p_d\), a change of variables \(x=aS\) gives
\begin{equation}
J_{k,d}^{f_d}(a)=\int_{\mathbb R}f_d(x)P_{k,d}(x/a)\frac1a p_d(x/a)\,\mathrm dx.
\label{eq:proof-stage2-change-variable}
\end{equation}
We claim that
\begin{equation}
\sup_{|a-1|\leq\delta_d}\int_{\mathbb R}
(1+|x|^m)\left|P_{k,d}(x/a)\frac1a p_d(x/a)-H_k(x)\phi(x)\right|\,\mathrm dx\longrightarrow0.
\label{eq:proof-stage2-weighted-density}
\end{equation}
Indeed, choose \(0<\tau<\min\{1/4,\beta/2\}\).
Uniformly for \(|a-1|\leq\delta_d\) and
\(|x|\leq d^\tau\), the spherical density satisfies
\begin{equation}
\frac1a p_d(x/a)=\phi(x)(1+o(1)).
\end{equation}
For fixed \(k\), coefficientwise convergence of the Gegenbauer
polynomials, together with \(a\to1\), gives $P_{k,d}(x/a)-H_k(x)\to0$
uniformly on compact sets.  The polynomial bound Eq.~\eqref{eq:proof-Gegenbauer-polynomial-growth}, the domination \(p_d(t)\leq C\phi(t)\), and Gaussian tails then proves Eq.~\eqref{eq:proof-stage2-weighted-density}.

Using the polynomial bound on \(f_d\), we therefore have
\begin{equation}
\begin{aligned}
\sup_{|a-1|\leq\delta_d}\left|J_{k,d}^{f_d}(a)-J_k^{f_d}\right|\leq
C\sup_{|a-1|\leq\delta_d}\int_{\mathbb R}(1+|x|^m)\left|P_{k,d}(x/a)\frac1a p_d(x/a)-H_k(x)\phi(x)
\right|\,\mathrm dx\longrightarrow0.
\end{aligned}
\label{eq:proof-stage2-J-uniform-good}
\end{equation}
Consequently,
\begin{equation}
\mathbb E\left[
\left|
J_{k,d}^{f_d}(A_d)-J_k^{f_d}
\right|
\mathbf1_{\mathcal G_d}
\right]
\longrightarrow0.
\label{eq:proof-stage2-J-good}
\end{equation}

It remains to control the bad event. The polynomial growth of \(P_{k,d}\) and \(f_d\), together with the uniform moment bounds for \(G_d\), gives
\begin{equation}
\left|J_{k,d}^{f_d}(a)\right|\leq
C_k(1+|a|^m),\qquad a\geq0,
\label{eq:proof-stage2-J-growth}
\end{equation}
and similarly $|J_k^{f_d}|\leq C_k$.
Moreover, Lemma~\ref{lem:proof-context-concentration} implies that
\(A_d\) has uniformly bounded moments of every fixed order.
Hence, for any fixed \(r>1\), Hölder's inequality yields
\begin{equation}
\begin{aligned}
\mathbb E\left[\left|J_{k,d}^{f_d}(A_d)-J_k^{f_d}\right|\mathbf1_{\mathcal G_d^c}\right]\leq
C_k\left(\mathbb E[(1+A_d^m)^r]\right)^{1/r}\mathbb P(\mathcal G_d^c)^{1-1/r}+C_k\mathbb P(\mathcal G_d^c)
\longrightarrow0,
\end{aligned}
\label{eq:proof-stage2-J-bad}
\end{equation}
using Eq.~\eqref{eq:proof-stage2-good-A}.
Combining Eqs.~\eqref{eq:proof-stage2-J-good}
and \eqref{eq:proof-stage2-J-bad}
proves Eq.~\eqref{eq:proof-stage2-weighted-gaussian-L1}.
\end{proof}

\subsubsection{The Stage-II regression on the training distribution}
\label{app:stage-II-regression}
For the rest of this subsection, write
\[
X_d:=X_{d,L},
\qquad
A_d:=A_{d,L},
\qquad
S_d:=S_{d,L},
\qquad
q_d:=q_{d,L},
\qquad
\rho_d:=\rho_{d,L}.
\]
We first identify \(g_{d,L}\) under its training distribution using a matched lower and upper bound.
Let $\mathcal H_d:=\sigma(\mathcal C_{d,L},S_d)$.
Since \(X_d=A_dS_d\), \(\sigma(X_d)\subseteq\mathcal H_d\). Define the richer regression function $\bar g_d:=\mathbb E[Y_d\mid\mathcal H_d]$, where $Y_d:=\sigma_*(G_d)$. Conditionally on the demonstration context, \((G_d,S_d)\) is a two-dimensional spherical projection with correlation \(q_d\).  Lemma~\ref{lem:proof-spherical} therefore gives
\begin{equation}
\bar g_d=\sum_{k\geq0}c_{k,d}Q_{k,d}(q_d)P_{k,d}(S_d)
\qquad\text{in }L^2.
\label{eq:proof-stage2-rich-regression}
\end{equation}
Consequently,
\begin{equation}
R_d^{\rm rich}:=\mathbb E[(Y_d-\bar g_d)^2]=v_d-\sum_{k\geq0}c_{k,d}^2
\mathbb E[Q_{k,d}(q_d)^2].
\label{eq:proof-stage2-rich-risk}
\end{equation}
For each fixed \(k\), Lemma~\ref{lem:proof-context-concentration} gives \(q_d-\rho_d\to0\) in probability, while Lemma~\ref{lem:proof-spherical} gives
$\sup_{|q|\leq1}|Q_{k,d}(q)-q^k|\to0$. Since \(|Q_{k,d}|\leq1\),
\begin{equation}
\mathbb E[Q_{k,d}(q_d)^2]-\rho_d^{2k}
\longrightarrow0.
\end{equation}
Using the uniform tail bound Eq.~\eqref{eq:proof-l1-spectrum-convergence}, we obtain
\begin{equation}
R_d^{\rm rich}=1-\sum_{k\geq1}c_k^2\rho_d^{2k}+o(1).
\label{eq:proof-stage2-rich-limit}
\end{equation}
We next introduce the Gaussian channel
\begin{equation}
h_d(x):=\Phi_{\rho_d}(x)=\mathbb E_Z\left[\sigma_\star\left(\rho_dx+\sqrt{1-\rho_d^2}\,Z\right)\right].
\label{eq:proof-stage2-gaussian-candidate}
\end{equation}
Assumption~\ref{ass:rigorous-teacher} implies the uniform polynomial bound
\begin{equation}
|h_d(x)|\leq C(1+|x|^p).
\label{eq:proof-stage2-h-growth}
\end{equation}
We claim that
\begin{equation}
\mathbb E[(Y_d-h_d(X_d))^2]=1-\sum_{k\geq1}
c_k^2\rho_d^{2k}+o(1).
\label{eq:proof-stage2-candidate-risk}
\end{equation}
Indeed, Lemma~\ref{lem:proof-stage2-weighted-gaussian} gives
\begin{equation}
\mathbb E[h_d(X_d)^2]=\mathbb E[h_d(Z)^2]+o(1)=\sum_{k\geq1}
c_k^2\rho_d^{2k}+o(1),
\label{eq:proof-stage2-candidate-second}
\end{equation}
where the last equality follows from Mehler's identity.
For the cross term, condition on the demonstration context and use the spherical projection formula in Lemma \ref{lem:proof-spherical}:
\begin{equation}
\begin{aligned}
\mathbb E[Y_dh_d(X_d)]=\mathbb E\left[\sum_{k\geq0}
c_{k,d}Q_{k,d}(q_d)J_{k,d}(A_d)\right],
\end{aligned}
\label{eq:proof-stage2-candidate-cross}
\end{equation}
where $J_{k,d}(a):=\mathbb E_{S\sim G_d}\left[P_{k,d}(S)h_d(aS)\right]$.
For \(M\geq0\), define the tail
\begin{equation}
R_{M,d}^{\rm tr}:=\mathbb E\left[\sum_{k>M}c_{k,d}Q_{k,d}(q_d)J_{k,d}(A_d)\right].
\end{equation}
Since \(|Q_{k,d}(q)|\leq1\), Cauchy--Schwarz gives
\begin{align}
|R_{M,d}^{\rm tr}|\leq\left(\sum_{k>M}c_{k,d}^2
\right)^{1/2}\mathbb E\left[\left(\sum_{k>M}J_{k,d}(A_d)^2
\right)^{1/2}\right]\leq\left(\sum_{k>M}c_{k,d}^2\right)^{1/2}
\left(\mathbb E\sum_{k>M}J_{k,d}(A_d)^2\right)^{1/2}.
\label{eq:proof-stage2-training-tail-CS}
\end{align}
For every fixed \(a\), Parseval's inequality for the orthonormal
Gegenbauer basis gives
\begin{equation}
\sum_{k\geq0}J_{k,d}(a)^2\leq\mathbb E_{S\sim G_d}
\left[h_d(aS)^2\right].
\label{eq:proof-stage2-training-tail-parseval}
\end{equation}
Since \(A_d\) is measurable with respect to the demonstration context
and \(S_d\) is independent of that context with the same law as \(G_d\),
we obtain
\begin{equation}
\mathbb E\sum_{k>M}J_{k,d}(A_d)^2\leq\mathbb E[h_d(A_dS_d)^2]=\mathbb E[h_d(X_d)^2].
\end{equation}
By Eq.~\eqref{eq:proof-stage2-candidate-second},
the right-hand side is uniformly bounded for all sufficiently large
\(d\). Consequently,
\begin{equation}
\lim_{M\to\infty}\limsup_{d\to\infty}|R_{M,d}^{\rm tr}|=0,
\label{eq:proof-stage2-training-tail-vanishes}
\end{equation}
where we used
Eq.~\eqref{eq:proof-l1-spectrum-convergence}.

We may therefore first truncate the sum at a fixed \(M\), pass to the
limit \(d\to\infty\) term by term, and then send \(M\to\infty\).
For every fixed \(k\),
Lemma~\ref{lem:proof-stage2-weighted-gaussian} gives
$J_{k,d}(A_d)-\mathbb E[H_k(Z)h_d(Z)]\longrightarrow0$ in $L^1$
while $Q_{k,d}(q_d)-\rho_d^k\longrightarrow0$ in $L^1$
and \(c_{k,d}\to c_k\).
Moreover, by Mehler's identity, $\mathbb E[H_k(Z)h_d(Z)]=c_k\rho_d^k$.
Hence, for every fixed \(M\),
\begin{equation}
\mathbb E\left[\sum_{k=0}^{M}c_{k,d}Q_{k,d}(q_d)J_{k,d}(A_d)\right]=
\sum_{k=0}^{M}c_k^2\rho_d^{2k}+o(1).
\label{eq:finite-M-convergence}
\end{equation}
Finally, $\sup_d\sum_{k>M}c_k^2\rho_d^{2k}\leq\sum_{k>M}c_k^2
\longrightarrow0$.
Combining Eq. \eqref{eq:finite-M-convergence} with
Eq.~\eqref{eq:proof-stage2-training-tail-vanishes} gives
\begin{equation}
\mathbb E[Y_dh_d(X_d)]=\sum_{k\geq0}c_k^2\rho_d^{2k}
+o(1).
\label{eq:proof-stage2-candidate-cross-limit}
\end{equation}
Combining
Eqs.~\eqref{eq:proof-stage2-candidate-second} and
\eqref{eq:proof-stage2-candidate-cross-limit} proves
Eq.~\eqref{eq:proof-stage2-candidate-risk}.

Let $R_d^{\rm opt}:=\mathbb E[(Y_d-g_{d,L}(X_d))^2]$.
Because \(\sigma(X_d)\subseteq\mathcal H_d\), while \(h_d(X_d)\) is a predictor measurable with respect to \(X_d\),
\begin{equation}
R_d^{\rm rich}\leq
R_d^{\rm opt}\leq\mathbb E[(Y_d-h_d(X_d))^2].
\end{equation}
Equations~\eqref{eq:proof-stage2-rich-limit} and
\eqref{eq:proof-stage2-candidate-risk} therefore imply
\begin{equation}
R_d^{\rm opt}=1-\sum_{k\geq1}c_k^2\rho_d^{2k}+o(1).
\label{eq:proof-stage2-optimal-risk}
\end{equation}
Finally, since \(g_{d,L}(X_d)\) is the orthogonal projection of \(Y_d\)
onto \(L^2(\sigma(X_d))\),
\begin{equation}
\mathbb E[(Y_d-h_d(X_d))^2]=R_d^{\rm opt}+\mathbb E[(g_{d,L}(X_d)-h_d(X_d))^2].
\end{equation}
Hence
\begin{equation}
\mathbb E[(g_{d,L}(X_{d,L})-\Phi_{\rho_{d,L}}(X_{d,L}))^2]
\longrightarrow0.
\label{eq:proof-stage2-training-regression-convergence}
\end{equation}

\subsubsection{Transfer from training to inference}

The preceding convergence holds under the law of \(X_{d,L}\). We next transfer it to the statistics encountered at inference.
Let $P_d:=\operatorname{Law}(X_{d,L})$ and $Q_d:=\operatorname{Law}\left(s_{d,L,\widetilde L}
X_{d,\widetilde L}
\right)$. Assumption~\ref{ass:rigorous-limit} and
Lemma~\ref{lem:proof-stage2-parameters} imply
\begin{equation}
0<s_-\leq
s_{d,L,\widetilde L}
\leq s_+<\infty.
\label{eq:proof-stage2-s-bounded}
\end{equation}

\begin{lemma}
\label{lem:proof-stage2-transfer}
Suppose \(f_d:\mathbb R\to\mathbb R\) satisfies
$\mathbb E_{P_d}[f_d^2]\to0$, and, for every fixed \(m<\infty\), $\sup_d\mathbb E_{P_d}[|f_d|^m]<\infty$.
Suppose moreover that on the supports of \(P_d\) and \(Q_d\), $|f_d(x)|\leq d^{C_f}$ for some \(C_f<\infty\).
Then
\begin{equation}
\mathbb E_{Q_d}[f_d^2]\longrightarrow0.
\label{eq:proof-stage2-transfer-conclusion}
\end{equation}
\end{lemma}

\begin{proof}
Choose $0<\beta<\frac14\min\{\varepsilon,1\}$ and $
0<\tau<\min\{1/4,\beta/2\}$.
Let $\mathcal G_{d,n}=\{|A_{d,n}-1|\leq d^{-\beta}\}$.
Lemma~\ref{lem:proof-context-concentration} gives
\begin{equation}
\mathbb P(\mathcal G_{d,n}^c)\leq
C_Md^{-M}
\label{eq:bound-Gc}
\end{equation}
for every \(M<\infty\), uniformly for
\(n\in\{L,\widetilde L\}\).

Conditioned on the good event $\mathcal G_{d,n}$, the argument used in Lemma~\ref{lem:proof-stage2-weighted-gaussian} gives, uniformly on
\(|x|\leq d^\tau\),
\begin{equation}
p_d^{\rm g}(x)=\phi(x)(1+o(1)),
\qquad
q_d^{\rm g}(x)=\phi_{s_d}(x)(1+o(1)),
\end{equation}
where \(s_d=s_{d,L,\widetilde L}\).
Because \(s_d\in[s_-,s_+]\), one may choose \(r>1\), sufficiently close to one, such that
\begin{equation}
\sup_d\int_{|x|\leq d^\tau}
\left(\frac{q_d^{\rm g}(x)}{p_d^{\rm g}(x)}\right)^rp_d^{\rm g}(x)\,\mathrm dx
<\infty.
\label{eq:proof-stage2-likelihood-ratio}
\end{equation}
Let \(r'=r/(r-1)\), and choose \(m>2r'\). By Hölder's inequality,
\begin{equation}
\mathbb E_{P_d}|f_d|^{2r'}
\leq
\left(
\mathbb E_{P_d}|f_d|^2
\right)^{\frac{m-2r'}{m-2}}
\left(
\mathbb E_{P_d}|f_d|^m
\right)^{\frac{2r'-2}{m-2}}.
\label{eq:proof-stage2-transfer-interpolation}
\end{equation}
The first factor tends to zero by assumption, whereas the second factor
is uniformly bounded. Hence
\begin{equation}
\|f_d\|_{L^{2r'}(P_d)}
\longrightarrow0.
\label{eq:proof-stage2-transfer-higher-norm}
\end{equation}
Since \(\mathbb P(\mathcal G_{d,L})\to1\), the same convergence holds under the training law conditioned on the good event:
\begin{equation}
\begin{aligned}
\mathbb E_{P_d^{\rm g}}|f_d|^{2r'}=\frac{\mathbb E\!\left[|f_d(X_{d,L})|^{2r'}\mathbf 1_{\mathcal G_{d,L}}\right]}{\mathbb P(\mathcal G_{d,L})}\leq
\frac{\mathbb E_{P_d}|f_d|^{2r'}}{\mathbb P(\mathcal G_{d,L})}\longrightarrow0.
\end{aligned}
\label{eq:proof-stage2-transfer-good-higher-norm}
\end{equation}
Let $C_d:=\{|x|\leq d^\tau\}$. On this region,
\begin{equation}
\mathbb E_{Q_d^{\rm g}}\left[f_d^2\mathbf 1_{C_d}\right]=\int_{C_d}
f_d(x)^2\frac{q_d^{\rm g}(x)}{p_d^{\rm g}(x)}
p_d^{\rm g}(x)\,\mathrm dx.
\end{equation}
Applying Hölder's inequality:
\begin{align}
\mathbb E_{Q_d^{\rm g}}\left[f_d^2\mathbf 1_{C_d}\right]\leq\left[\int_{C_d}\left(\frac{q_d^{\rm g}(x)}{p_d^{\rm g}(x)}\right)^rp_d^{\rm g}(x)\,\mathrm dx\right]^{1/r}\times\left[\int_{C_d}|f_d(x)|^{2r'}p_d^{\rm g}(x)\,\mathrm dx\right]^{1/r'}.
\label{eq:proof-stage2-transfer-central-holder}
\end{align}
The first factor is uniformly bounded by Eq.~\eqref{eq:proof-stage2-likelihood-ratio}, while the second factor is bounded by
$\left[\mathbb E_{P_d^{\rm g}}|f_d|^{2r'}
\right]^{1/r'}\longrightarrow0$
according to Eq.~\eqref{eq:proof-stage2-transfer-good-higher-norm}. Therefore
\begin{equation}
\mathbb E_{Q_d^{\rm g}}
\left[
f_d^2\mathbf 1_{\{|x|\leq d^\tau\}}
\right]
\longrightarrow0.
\label{eq:proof-stage2-transfer-central}
\end{equation}

It remains to control the complement. On \(\mathcal G_{d,\widetilde L}\), we have $|A_{d,\widetilde L}-1|\leq d^{-\beta}$, and \(s_d\) is bounded above and below by positive constants. Hence,
for all sufficiently large \(d\),
$|s_dX_{d,\widetilde L}|>d^\tau\Longrightarrow|S_{d,\widetilde L}|>c_0d^\tau$
for some fixed \(c_0>0\). Since \(S_{d,\widetilde L}\) has the same distribution as \(G_d\), and \(p_d(t)\leq C\phi(t)\),
\begin{equation}
Q_d^{\rm g}(|x|>d^\tau)\leq
C\exp(-c_{tail}d^{2\tau})
\label{eq:proof-stage2-transfer-tail-probability}
\end{equation}
for some \(c_{tail}>0\). Using \(|f_d(x)|\leq d^{C_f}\) on the support of \(Q_d\), we therefore obtain
\begin{equation}
\mathbb E_{Q_d^{\rm g}}\left[f_d^2\mathbf 1_{\{|x|>d^\tau\}}\right]\leq
d^{2C_f}Q_d^{\rm g}(|x|>d^\tau)\longrightarrow0.
\label{eq:proof-stage2-transfer-tail}
\end{equation}
Finally, on the bad event,
\begin{equation}
\mathbb E
\left[
f_d(s_dX_{d,\widetilde L})^2
\mathbf 1_{\mathcal G_{d,\widetilde L}^c}
\right]
\leq
d^{2C_f}
\mathbb P(\mathcal G_{d,\widetilde L}^c).
\end{equation}
The right-hand side tends to zero after choosing \(M>2C_f\) and using \eqref{eq:bound-Gc}. Combining this estimate with Eqs.~\eqref{eq:proof-stage2-transfer-central} and \eqref{eq:proof-stage2-transfer-tail} proves
\begin{equation}
\mathbb E_{Q_d}[f_d^2]\longrightarrow0.
\end{equation}
\end{proof}

We apply Lemma~\ref{lem:proof-stage2-transfer} to
\begin{equation}
f_d(x):=g_{d,L}(x)-\Phi_{\rho_{d,L}}(x).
\end{equation}
Equation~\eqref{eq:proof-stage2-training-regression-convergence}
gives the required \(L^2(P_d)\) convergence.
Jensen's inequality gives
\begin{equation}
\mathbb E_{P_d}|g_{d,L}(X_{d,L})|^m\leq\mathbb E|Y_d|^m\leq C_m,
\end{equation}
while Eq.~\eqref{eq:proof-stage2-h-growth} and the moment bounds for \(X_{d,L}\) give the corresponding estimate for
\(\Phi_{\rho_{d,L}}(X_{d,L})\).
Finally, Eq.~\eqref{eq:proof-stage2-regression-bound} and the polynomial growth of \(\Phi_\rho\) give the required polynomial pointwise bound. Hence, all the conditions in Lemma~\ref{lem:proof-stage2-transfer} are satisfied and we have
\begin{equation}
\mathbb E\left[\left(g_{d,L}(s_{d,L,\widetilde L}X_{d,\widetilde L})-
\Phi_{\rho_{d,L}}(s_{d,L,\widetilde L}X_{d,\widetilde L})\right)^2\right]\longrightarrow0.
\label{eq:proof-stage2-inference-transfer}
\end{equation}

\subsubsection{Reduction of the inference error to a Gaussian channel}
Define $F_d(t):=\Phi_{\rho_{d,L}}\left(s_{d,L,\widetilde L}t\right)$. Equation~\eqref{eq:proof-stage2-inference-transfer}, together with Cauchy--Schwarz and the uniform moment bounds, gives
\begin{equation}
\begin{aligned}
&\left|\mathcal E_{d,L,\widetilde L}^{\mathrm{feature}}-\mathbb E\!\left[\left(Y_d-F_d(X_{d,\widetilde L})\right)^2\right]\right|
\\
&\qquad\le\left\|g_{d,L}(s_{d,L,\widetilde L}X_{d,\widetilde L})-F_d(X_{d,\widetilde L})\right\|_{L^2}\times\left(2\|Y_d\|_{L^2}+\left\|g_{d,L}(s_{d,L,\widetilde L}X_{d,\widetilde L})
\right\|_{L^2}+\left\|F_d(X_{d,\widetilde L})\right\|_{L^2}\right)=o(1).
\end{aligned}
\label{eq:proof-stage2-error-Fd}
\end{equation}
Write
$A_d^{\rm te}:=A_{d,\widetilde L}$,
$S_d^{\rm te}:=S_{d,\widetilde L}$,
$q_d^{\rm te}:=q_{d,\widetilde L}$ and $\rho_d^{\rm te}:=\rho_{d,\widetilde L}$. Lemma~\ref{lem:proof-stage2-weighted-gaussian} applied with \(f_d=F_d^2\) gives
\begin{equation}
\mathbb E[F_d(X_{d,\widetilde L})^2]=\mathbb E[F_d(Z)^2]+o(1).
\label{eq:proof-stage2-test-second}
\end{equation}
For the cross term, the spherical projection formula in Lemma \ref{lem:proof-spherical} gives
\begin{equation}
\mathbb E[Y_dF_d(X_{d,\widetilde L})]=\mathbb E\left[\sum_{k\geq0}
c_{k,d}Q_{k,d}(q_d^{\rm te})J_{k,d}^{\rm te}(A_d^{\rm te})\right],
\label{eq:proof-stage2-test-cross-expansion}
\end{equation}
where $J_{k,d}^{\rm te}(a):=\mathbb E_{S\sim G_d}[P_{k,d}(S)F_d(aS)]$.
For every fixed \(k\),
Lemma~\ref{lem:proof-stage2-weighted-gaussian} gives
\begin{equation}
J_{k,d}^{\rm te}(A_d^{\rm te})-\mathbb E[H_k(Z)F_d(Z)]
\longrightarrow0,
\end{equation}
while $Q_{k,d}(q_d^{\rm te})-(\rho_d^{\rm te})^k
\longrightarrow0$ by Lemma \ref{lem:proof-spherical}.
After truncating the Gegenbauer expansion (as in Appendix \ref{app:stage-II-regression}) and using Eq.~\eqref{eq:proof-l1-spectrum-convergence}, we obtain
\begin{equation}
\mathbb E[Y_dF_d(X_{d,\widetilde L})]=\sum_{k\geq0}c_k(\rho_d^{\rm te})^k\mathbb E[H_k(Z)F_d(Z)]+o(1).
\label{eq:proof-stage2-test-cross-limit}
\end{equation}
Let
$T_d^{\rm te}:=\rho_d^{\rm te}G+\sqrt{1-(\rho_d^{\rm te})^2}\,Z$ with
$G,Z\overset{\rm iid}{\sim}\mathcal N(0,1)$.
By Mehler's identity,
\begin{equation}
\mathbb E[\sigma_\star(G)F_d(T_d^{\rm te})]=\sum_{k\geq0}c_k(\rho_d^{\rm te})^k\mathbb E[H_k(Z)F_d(Z)].
\label{eq:proof-stage2-test-cross}
\end{equation}
Since \(T_d^{\rm te}\sim\mathcal N(0,1)\),
Eqs.~\eqref{eq:proof-stage2-error-Fd}, \eqref{eq:proof-stage2-test-second} and \eqref{eq:proof-stage2-test-cross-limit}, \eqref{eq:proof-stage2-test-cross} therefore imply
\begin{equation}
\mathcal E_{d,L,\widetilde L}^{\mathrm{feature}}=\mathbb E_{G,Z}\left[\left(\sigma_\star(G)-\Phi_{\rho_{d,L}}\left(s_{d,L,\widetilde L}T_d^{\rm te}\right)\right)^2\right]+o(1).
\label{eq:proof-stage2-finite-gaussian-channel}
\end{equation}

\subsubsection{Continuity of the Gaussian-channel error}
\label{app:feature-stage2-continuity}

It remains to justify replacing the finite-dimensional parameters
$\rho_{d,L},\rho_{d,\widetilde L},s_{d,L,\widetilde L}$ by their limiting counterparts
$\rho_\alpha,\rho_{\widetilde\alpha},s_{\alpha,\widetilde\alpha}$. Note that \(\sigma_\star\) might not be continuous. We therefore establish continuity of the Gaussian channel.

We first record a simple approximation lemma. 
\begin{lemma}
\label{lem:proof-gaussian-perturbation}
Let \(f:\mathbb R\to\mathbb R\) be measurable and suppose that, for some \(C,p<\infty\),
\begin{equation}
|f(x)|\leq C(1+|x|^p),
\qquad x\in\mathbb R.
\label{eq:proof-gaussian-perturbation-growth}
\end{equation}
Let \(X_n,X\) be centered Gaussian random variables such that $X_n\longrightarrow X$ in $L^2$ and suppose that
\(\operatorname{Var}(X)>0\).
Then
\begin{equation}
\mathbb E\left[|f(X_n)-f(X)|^2\right]
\longrightarrow0.
\label{eq:proof-gaussian-perturbation-conclusion}
\end{equation}
\end{lemma}

\begin{proof}
Write $v_n:=\operatorname{Var}(X_n)$ and $v:=\operatorname{Var}(X)>0$.
Then \(v_n\to v\).  Hence, for all sufficiently large \(n\),
$v_n\in[v/2,2v]$. Let \(\phi_w\) denote the density of
\(\mathcal N(0,w)\).

Fix \(\eta>0\).  We may choose \(R<\infty\) such that
\begin{equation}
\sup_{w\in[v/2,2v]}\int_{|x|>R}(1+|x|^{2p})\phi_w(x)\,\mathrm dx<\eta.
\label{eq:proof-gaussian-perturbation-tail}
\end{equation}
Since continuous compactly supported functions are dense in
\(L^2(\mathcal N(0,v))\), we may choose a continous function $g$ with support contained in \((-R,R)\), such that
\begin{equation}
\int_{-R}^R|f(x)-g(x)|^2\phi_v(x)\,\mathrm dx<\eta.
\label{eq:proof-gaussian-perturbation-approx}
\end{equation}
On the compact interval \([-R,R]\), the density ratios
\(\phi_w/\phi_v\) are uniformly bounded for
\(w\in[v/2,2v]\).  Therefore
\begin{equation}
\sup_{n\ {\rm sufficiently\ large}}\int_{-R}^R
|f(x)-g(x)|^2\phi_{v_n}(x)\,\mathrm dx\leq C_R\eta.
\label{eq:proof-gaussian-perturbation-local-transfer}
\end{equation}
Combining
Eqs.~\eqref{eq:proof-gaussian-perturbation-tail} and
\eqref{eq:proof-gaussian-perturbation-local-transfer}, and using the
polynomial growth of \(f\), gives
\begin{equation}
\limsup_{n\to\infty}
\mathbb E|f(X_n)-g(X_n)|^2
\leq C_R'\eta.
\label{eq:proof-gaussian-perturbation-fg}
\end{equation}
The same argument gives
\begin{equation}
\mathbb E|f(X)-g(X)|^2
\leq C_R'\eta.
\label{eq:proof-gaussian-perturbation-fg-limit}
\end{equation}

On the other hand, \(X_n\to X\) in probability  Since \(g\) is bounded and continuous, $g(X_n)\longrightarrow g(X)$ in probability, and boundedness of \(g\) implies
\begin{equation}
\mathbb E|g(X_n)-g(X)|^2\longrightarrow0.
\label{eq:proof-gaussian-perturbation-g}
\end{equation}
Consequently,
\begin{align}
\limsup_{n\to\infty}\mathbb E|f(X_n)-f(X)|^2
&\leq
3\limsup_{n\to\infty}\mathbb E|f(X_n)-g(X_n)|^2+
3\limsup_{n\to\infty}\mathbb E|g(X_n)-g(X)|^2+3\mathbb E|g(X)-f(X)|^2
\nonumber\\
&\leq
C_R''\eta.
\end{align}
Since \(\eta>0\) is arbitrary, this proves
Eq.~\eqref{eq:proof-gaussian-perturbation-conclusion}.
\end{proof}

We now apply the preceding lemma to the Gaussian channel that appears
in the Stage-II test error.  Recall
$\Phi_\rho(x):=\mathbb E_W\left[
\sigma_\star\left(\rho x+\sqrt{1-\rho^2}\,W\right)\right]$, where $W\sim\mathcal N(0,1)$.
For $(\rho,u,s)\in[-1,1]^2\times(0,\infty)$,
define
\begin{equation}
T_u:=uG+\sqrt{1-u^2}\,Z,
\qquad
G,Z\overset{\rm iid}{\sim}\mathcal N(0,1),
\label{eq:proof-continuity-Tu}
\end{equation}
and define the corresponding error
\begin{equation}
\mathscr E(\rho,u,s):=\mathbb E_{G,Z}\left[
\left(\sigma_\star(G)-\Phi_\rho(sT_u)\right)^2\right].
\label{eq:proof-continuity-channel-error}
\end{equation}

\begin{lemma}
\label{lem:proof-continuity-channel-error}
The map
$(\rho,u,s)\longmapsto\mathscr E(\rho,u,s)$
is continuous on $[-1,1]^2\times(0,\infty)$.
In particular, it is uniformly continuous on $[-1,1]^2\times[s_-,s_+],0<s_-<s_+<\infty$.
\end{lemma}

\begin{proof}
Let $(\rho_n,u_n,s_n)\longrightarrow
(\rho,u,s)$ with $s>0$. Introduce an additional independent
\(W\sim\mathcal N(0,1)\), and define
\begin{equation}
X_n:=\rho_ns_n\left(u_nG+\sqrt{1-u_n^2}\,Z\right)+\sqrt{1-\rho_n^2}\,W,
\label{eq:proof-continuity-Xn}
\end{equation}
and
\begin{equation}
X:=\rho s\left(uG+\sqrt{1-u^2}\,Z\right)+\sqrt{1-\rho^2}\,W.
\label{eq:proof-continuity-X}
\end{equation}
All coefficients in
Eq.~\eqref{eq:proof-continuity-Xn}
converge to the corresponding coefficients in
Eq.~\eqref{eq:proof-continuity-X}.  Hence
\begin{equation}
X_n\longrightarrow X
\qquad\text{in }L^2.
\label{eq:proof-continuity-X-L2}
\end{equation}
Moreover, $\operatorname{Var}(X)=1-\rho^2+\rho^2s^2>0$ because \(s>0\).

Applying Lemma~\ref{lem:proof-gaussian-perturbation}
with \(f=\sigma_\star\) gives
\begin{equation}
\mathbb E\left[|\sigma_\star(X_n)-\sigma_\star(X)|^2\right]
\longrightarrow0.
\label{eq:proof-continuity-sigma-X}
\end{equation}
Conditioning on \(G,Z\), we have
\begin{equation}
\mathbb E[\sigma_\star(X_n)\mid G,Z]=\Phi_{\rho_n}(s_nT_{u_n}),
\qquad
\mathbb E[\sigma_\star(X)\mid G,Z]=\Phi_\rho(sT_u).
\label{eq:proof-continuity-conditional}
\end{equation}
Therefore, Jensen's inequality and Eq.~\eqref{eq:proof-continuity-sigma-X} imply
\begin{equation}
\left\|\Phi_{\rho_n}(s_nT_{u_n})-\Phi_\rho(sT_u)
\right\|_{L^2}\longrightarrow0.
\label{eq:proof-continuity-Phi-L2}
\end{equation}
The polynomial growth of \(\sigma_\star\) (Assumption \ref{ass:rigorous-teacher}), together with the boundedness of the variances of \(X_n\), also gives
\begin{equation}
\sup_n\left\|\Phi_{\rho_n}(s_nT_{u_n})\right\|_{L^2}<\infty,
\label{eq:proof-continuity-uniform-Phi}
\end{equation}
where we again use Jensen's inequality.
Hence
\begin{align}
&\left|\mathscr E(\rho_n,u_n,s_n)-\mathscr E(\rho,u,s)\right|\nonumber\\
&\qquad\leq
\left\|\Phi_{\rho_n}(s_nT_{u_n})-\Phi_\rho(sT_u)
\right\|_{L^2}\Big(2\|\sigma_\star(G)\|_{L^2}+\|\Phi_{\rho_n}(s_nT_{u_n})\|_{L^2}+\|\Phi_\rho(sT_u)\|_{L^2}
\Big)\longrightarrow0.
\label{eq:proof-continuity-error}
\end{align}
This proves continuity.

Uniform continuity on
\([-1,1]^2\times[s_-,s_+]\)
then follows from compactness.
\end{proof}

We finally specialize this continuitt to our problem.  Recall that
$s_{d,L,\widetilde L}=\frac{B_{d,\widetilde L}}{B_{d,L}}$ and $s_{\alpha,\widetilde\alpha}=\sqrt{\frac{c_1^2+1/\widetilde\alpha}{c_1^2+1/\alpha}}$. Writing $r:=\frac{L}{\widetilde L}=\frac{\alpha}{\widetilde\alpha}$, we have
\begin{equation}
s_{d,L,\widetilde L}^2=\frac{\mu_d^2+r\,dv_d/L}{\mu_d^2+dv_d/L},
\qquad
s_{\alpha,\widetilde\alpha}^2=\frac{c_1^2+r/\alpha}{c_1^2+1/\alpha}.
\label{eq:proof-continuity-s-bounds}
\end{equation}
Each ratio in Eq.~\eqref{eq:proof-continuity-s-bounds}
is a convex combination of \(1\) and \(r\).  Therefore,
Assumption~\ref{ass:rigorous-limit} implies that there exist constants \(0<s_-<s_+<\infty\) such that, for all sufficiently large \(d\),
$s_- \leq s_{d,L,\widetilde L},
s_{\alpha,\widetilde\alpha}\leq s_+$.

By Lemma~\ref{lem:proof-stage2-parameters},
\begin{equation}
\left|\rho_{d,L}-\rho_\alpha\right|+\left|\rho_{d,\widetilde L}-\rho_{\widetilde\alpha}\right|+
\left|s_{d,L,\widetilde L}-s_{\alpha,\widetilde\alpha}\right|
\longrightarrow0.
\label{eq:proof-continuity-parameter-distance}
\end{equation}
Both parameter triples therefore lie in the compact set
$[-1,1]^2\times[s_-,s_+]$ for all sufficiently large \(d\).  Uniform continuity from Lemma~\ref{lem:proof-continuity-channel-error} and
Eq.~\eqref{eq:proof-continuity-parameter-distance} yield
\begin{equation}
\begin{aligned}
\mathbb E_{G,Z}\left[\left(\sigma_\star(G)-\Phi_{\rho_{d,L}}\left(s_{d,L,\widetilde L}T_d^{\rm te}\right)\right)^2\right]=
\mathbb E_{G,Z}
\left[\left(\sigma_\star(G)-\Phi_{\rho_\alpha}\left(s_{\alpha,\widetilde\alpha}
T_{\widetilde\alpha}\right)\right)^2\right]+o(1),
\end{aligned}
\label{eq:proof-stage2-parameter-replacement}
\end{equation}
where
$T_d^{\rm te}=\rho_{d,\widetilde L}G+\sqrt{1-\rho_{d,\widetilde L}^2}\,Z$ and  $T_{\widetilde\alpha}=\rho_{\widetilde\alpha}G+\sqrt{1-\rho_{\widetilde\alpha}^2}\,Z$.

\subsubsection{Final proof of Theorem \ref{thm:population_scaling}}
Combining
Eqs.~\eqref{eq:proof-stage2-finite-gaussian-channel} and
\eqref{eq:proof-stage2-parameter-replacement} gives
\begin{equation}
\mathcal E_{d,L,\widetilde L}^{\mathrm{feature}}=\mathbb E_{G,Z}
\left[\left(\sigma_\star(G)-\Phi_{\rho_\alpha}
\left(s_{\alpha,\widetilde\alpha}T_{\widetilde\alpha}\right)\right)^2
\right]+o_{d,L,\widetilde L}(1).
\label{eq:proof-feature-final-channel}
\end{equation}
Define $F_{\alpha,\widetilde\alpha}(t):=\Phi_{\rho_\alpha}
\left(s_{\alpha,\widetilde\alpha}t\right)$.
Under assumption \ref{ass:rigorous-teacher}, \(F_{\alpha,\widetilde\alpha}\in L^2(\mathcal N(0,1))\).
Let
\begin{equation}
A_\ell(\alpha,\widetilde\alpha):=\mathbb E_T[F_{\alpha,\widetilde\alpha}(T)H_\ell(T)],
\qquad
T\sim\mathcal N(0,1).
\end{equation}
Then $F_{\alpha,\widetilde\alpha}(T)=\sum_{\ell\geq0}
A_\ell(\alpha,\widetilde\alpha)H_\ell(T)$ in $L^2$.

On the other hand, the Bayes-optimal predictor is
\begin{equation}
m_{\widetilde\alpha}(T_{\widetilde\alpha}):=\mathbb E[
\sigma_\star(G)\mid
T_{\widetilde\alpha}]=\sum_{\ell\geq0}c_\ell\rho_{\widetilde\alpha}^{\ell}H_\ell(T_{\widetilde\alpha}),
\end{equation}
where \(c_0=0\). The orthogonal-projection identity gives
\begin{equation}
\begin{aligned}
\mathbb E\left[\left(\sigma_\star(G)-F_{\alpha,\widetilde\alpha}(T_{\widetilde\alpha})
\right)^2\right]=\mathbb E\left[\left(\sigma_\star(G)-m_{\widetilde\alpha}(T_{\widetilde\alpha})\right)^2
\right]+\mathbb E\left[\left(m_{\widetilde\alpha}(T_{\widetilde\alpha})-F_{\alpha,\widetilde\alpha}(T_{\widetilde\alpha})\right)^2\right].
\end{aligned}
\end{equation}
By Mehler's identity and Parseval's identity,
\begin{equation}
\mathbb E\left[\left(\sigma_\star(G)-m_{\widetilde\alpha}(T_{\widetilde\alpha})\right)^2\right]=1-\sum_{k\geq1}
c_k^2r_{\widetilde\alpha}^k,
\end{equation}
and
\begin{equation}
\mathbb E\left[\left(m_{\widetilde\alpha}(T_{\widetilde\alpha})-F_{\alpha,\widetilde\alpha}(T_{\widetilde\alpha})\right)^2\right]=
\sum_{\ell\geq0}\left[A_\ell(\alpha,\widetilde\alpha)-c_\ell\rho_{\widetilde\alpha}^{\ell}\right]^2.
\end{equation}
Substituting these identities into
Eq.~\eqref{eq:proof-feature-final-channel} yields
\begin{equation}
\mathcal E_{d,L,\widetilde L}^{\mathrm{feature}}=1-
\sum_{k\geq1}c_k^2r_{\widetilde\alpha}^k+\sum_{\ell\geq0}
\left[A_\ell(\alpha,\widetilde\alpha)-c_\ell\rho_{\widetilde\alpha}^{\ell}\right]^2+o_{d,L,\widetilde L}(1).
\label{eq:proof-feature-final}
\end{equation}
This proves the feature-learner part of
Theorem~\ref{thm:population_scaling}.

\section{Declaration of LLM usage} \label{app:LLMusage}
We used LLMs in the following stages of the preparation of this work: editing (e.g., grammar, spelling, word choice), drafting sections of the paper, assisting in the writing of proofs, facilitating or running experiments, visualizing results for submission. All LLM-generated material used in this manuscript was manually verified by the authors to ensure that it performs as intended, in particular with regard to the derivation, proof and the numerical experiments.
\end{document}